\documentclass[a4paper,10pt]{amsart}
\usepackage{listings,qtree}
\usepackage{amsfonts, enumitem, imakeidx}%
\usepackage{amssymb,amsmath,amstext,appendix, amsthm, mathtools}
\usepackage[english]{babel}
\usepackage{graphicx}
\usepackage{pgf}
\usepackage{tikz}
\usepackage{tikz-cd}
\usetikzlibrary{calc, babel, arrows}
\usepackage{pgflibraryarrows}
\usepackage{pgflibrarysnakes}
\usepackage{enumitem}
\usepackage{color}
\usepackage{float}
\usepackage{xcolor}
\usepackage{hyperref}
\hypersetup{backref,
    colorlinks,%
    citecolor=blue,%
    filecolor=black,%
    linkcolor=black,%
    urlcolor=cyan,
    }

\usepackage{mathrsfs}

\newcommand\restr[2]{{
  \left.\kern-\nulldelimiterspace 
  #1 
  \vphantom{\big|} 
  \right|_{#2} 
  }}

\makeatletter
\def\paragraph{\@startsection{paragraph}{4}%
  \z@\z@{-\fontdimen2\font}%
  {\normalfont\bfseries}}
\makeatother

\newcommand\A{\mathbb A}
\newcommand\C{\mathbb C}

\newcommand\F{\mathbb F}
\newcommand\G{\mathbb G}

\newcommand\K{\mathbb K}

\newcommand\N{\mathbb N}

\newcommand\R{\mathbb R}

\newcommand\Z{\mathbb Z}

\newcommand\CC{{\mathcal C}}
\newcommand\CD{{\mathcal D}}

\newcommand\CG{{\mathcal G}}
\newcommand\CH{{\mathcal H}}

\newcommand\CM{{\mathcal M}}

\newcommand\CU{{\mathcal U}}

\newcommand\CW{{\mathcal W}}

\newcommand\scrA{{\mathscr A}}
\newcommand\scrB{{\mathscr B}}
\newcommand\scrC{{\mathscr C}}
\newcommand\scrD{{\mathscr D}}

\newcommand\scrF{{\mathscr F}}
\newcommand\scrG{{\mathscr G}}
\newcommand\scrH{{\mathscr H}}

\newcommand\scrL{{\mathscr L}}
\newcommand\scrM{{\mathscr M}}
\newcommand\scrN{{\mathscr N}}

\newcommand\scrP{{\mathscr P}}

\newcommand\scrR{{\mathscr R}}

\newcommand\scrT{{\mathscr T}}

\newcommand\scrV{{\mathscr V}}

\newcommand\scrZ{{\mathscr Z}}

\newtheorem{theorem}{Theorem}[section]
\newtheorem{lemma}[theorem]{Lemma}
\newtheorem{proposition}[theorem]{Proposition}

\newtheorem{corollary}[theorem]{Corollary}

\newtheorem{main}[theorem]{Theorem (Main)}
\newtheorem{prop-main}[theorem]{Proposition (Main)}
\newtheorem{corol-main}[theorem]{Corollary (Main)}
\newtheorem{definition}{Definition}

\newtheorem{claim}{Claim}[theorem]
\newtheorem{claimNested}{Claim}[claim]

\theoremstyle{definition}

\theoremstyle{remark}
\newtheorem{remark}[theorem]{Remark}

\newtheorem{open-problem}[theorem]{Open Problem}

\newtheorem{example}[theorem]{Example}

\newenvironment{proof-claim}[1][Proof of the Claim]{\noindent\textbf{#1.} }{\
\rule{0.5em}{0.5em}\medskip}

\numberwithin{equation}{section}

\DeclareMathOperator{\grad}{\rm grad}

\DeclareMathOperator{\Mdeg}{\rm Max-deg}
\DeclareMathOperator{\rank}{\rm rank}

\newcommand\lestricto{\subsetneq}

\newcommand{\crctc}{\raisebox{0.5pt}[1ex][1ex]{$\chi$}}

\normalbaselines

\begin{document}

\title[Erzeugungsgrad]{Erzeugungsgrad, VC-Dimension and Neural Networks with rational activation function}

\author{L. M. Pardo}
\address{Independent Researcher. Santander, Cantabria, Spain.}
\email{luis.m.pardo@gmail.com}
\author{D. Sebasti\'an}
\address{IES Garcilaso de la Vega. C. Eugenio de Lemus, s/n. 39300. Torrelavega, Cantabria, Spain. \vskip 0.1cm
Depto. de Matem\'aticas, Estad\'istica y Computaci\'on. Facultad de Ciencias. Universidad de Cantabria. Avda. Los Castros s/n. 39071. Santander, Cantabria, Spain.}
\email{danielsesan@gmail.com}

\dedicatory{Dedicated to the memory of Joos Heintz}



\maketitle

\begin{abstract}
The notion of Erzeugungsgrad was introduced by Joos Heintz in \cite{Heintz83} to bound the number of non-empty cells occurring after a process of quantifier elimination. We extend this notion and the combinatorial bounds of Theorem 2 in \cite{Heintz83} using the degree for constructible sets defined in \cite{PardoSebastian}. We show that the Erzeugungsgrad is the key ingredient to connect affine Intersection Theory over algebraically closed fields and the VC-Theory of Computational Learning Theory for families of classifiers given by parameterized families of constructible sets. In particular, we prove that the VC-dimension and the Krull dimension are linearly related up to logarithmic factors based on Intersection Theory. Using this relation, we study the density of correct test sequences in evasive varieties. We apply these ideas to analyze parameterized families of neural networks with rational activation function.
\end{abstract}



\section{Introduction}

These pages pay tribute to the scientific legacy of our friend Joos Heintz. The motivation goes back to a discussion held by the authors with Joos Heintz in Santander, during his last visit to the city. The conversation focused on the work \cite{HeintzPardoRedesNeuronales}. Both J. Heintz and the present authors have oriented their interest to the connections between Computational Learning Theory and Computational Algebraic Geometry. During these conversations, the authors of this manuscript launched a conjecture: \\
\emph{There must be strong connections between the Krull dimension of the space of parameters and the current notions of dimension in Computational Learning Theory, at least in the case of constructible families of classifiers.} \\
The answer of Heintz was a challenge: \emph{``I have no idea, prove it if you can''}. The present manuscript answers Heintz's challenge positively. What is even more poignant is that the answer to Heintz's challenge found his basis in an almost forgotten notion in \cite{Heintz83}: the \emph{Erzeugungsgrad} or generation degree. \\
The notion of Erzeugungsgrad extended Heintz's study of affine Intersection Theory in \cite{Heintz83}. There, Heintz introduced a notion of degree of affine algebraic subvarieties $V \subseteq \A^n(K)$, where $K$ is an algebraically closed field of any characteristic and $\A^n(K) = K^n$ denotes the $n$-dimensional affine space over $K$. The most relevant contribution of \cite{Heintz83} is a \emph{Bézout's Inequality} satisfied by his notion of degree of algebraic varieties. Then, Heintz used his notion and his bounds to establish upper complexity bounds for algorithms that eliminate quantifiers over algebraically closed fields. We must say that Heintz's Bézout's Inequality was obtained independently, with quite original proofs, of the almost simultaneous results in different contexts obtained by W. Vogel (\cite{Vogel}) or W. Fulton (\cite{Fulton}). \\
In \cite{Heintz83}, the author generalized his ideas in two main directions. Firstly, since constructible sets naturally arise within the context of quantifier elimination over algebraically closed fields, Heintz introduced a notion of degree of constructible sets (which we call here the Zariski degree or $Z$-degree). Secondly, quantifier elimination leads to Boolean formulae presented as finite unions of constructible sets given by finite intersections that Heintz called \emph{cells}. Then, Heintz tried to bound the size of these Boolean formulae by bounding the number of non--empty cells. Within this last context, the Erzeugungsgrad arose as a useful tool. \\
However, two main difficulties arise after the systematic study of \cite{Heintz83}. As already observed in \cite{Heintz85}, his notion of degree of constructible sets (the $Z$-degree) does not satisfy a Bézout's Inequality. Consequently, many of his statements only hold if the term ``constructible set'' is replaced by ``locally closed subset of an affine space''. This difficulty has been fixed in \cite{PardoSebastian}. These last authors introduced in $2022$ not only one but two notions of degree of constructible sets and proved that both of them satisfy a Bézout's Inequality (cf. \cite{PardoSebastian}). In order to help the reader to follow these technical results, we have introduced Section \ref{notaciones-basicas:sec}, which also helps to establish some of the notations and basic results used in the present manuscript. Thus, the notions and main outcomes of \cite{Heintz83} are summarized in Subsection \ref{gradoHeintz:subsec}. And, then, the two notions of degree of constructible sets and some of their main properties, as proved in \cite{PardoSebastian}, are summarized in Subsection \ref{gradoPardoSebastian:subsec}. \\
Once these difficulties have been fixed, the study of the Erzeugungsgrad begins. This is what we discuss in the present manuscript in Section \ref{erzeugungsgrad:sec}. First of all, we must adapt the notion to constructible sets and, then, we generalize Theorem $2$ of \cite{Heintz83} to control the number of non--empty cells in a constructible subset $C \subseteq \A^n(K)$. The main contribution is Theorem \ref{cotas-Erzeugunsgrad:teor} below. The idea is to establish upper bounds for the number of non--empty cells defined by a finite family of constructible sets within the constructible subset $C \subseteq \A^n(K)$ (see Claim $i)$ in Theorem \ref{cotas-Erzeugunsgrad:teor}). Roughly speaking, this statement establishes a connection between \emph{combinatorial bounds} and Intersection Theory that will be profitable later in these pages. \\
Next, we focus on Computational Learning Theory and  concentrate our interest in the seminal ideas of Vapnik and Chervonenkis (cf. \cite{VC:paper}). As the reader interested in Sections \ref{notaciones-basicas:sec} and \ref{erzeugungsgrad:sec} may not be aware of the notions and main statements of VC-theory, we have summarized some of this material in Subsection \ref{basicNotionsVCTheory:subsec}. Then, we proceed by solving positively the conjecture and Heintz's challenge above in Theorem \ref{FuncionCrecimientoDimensionConstructibles:thm} and Corollary \ref{VCDimensionConstructibles:corol}. Connections between Real Algebraic Geometry invariants and VC-theory are already known  (cf. \cite{Goldberg}, \cite{Lickteig}, \cite{Karpinsky}, \cite{VCAnalytic}, \cite{Montanna-Pardo} and references therein). As constructible sets over algebraically closed fields of characteristic $0$ may be seen as semi--algebraic sets of some real affine space, these classical results could be used to establish some connections. However, these connections are only valid in the characteristic $0$ case and use upper bounds for the number of connected components of semi-algebraic sets (cf. \cite{Milnor}, \cite{Thom}, \cite{Oleinik1}, \cite{Oleinik2} and \cite{Warren}). Our approach, based on the Erzeugungsgrad, holds independently of the characteristic of the field and use notions natural for the theory of constructible sets and not the theory of semi--algebraic sets. \\ 
The remainder of Section \ref{VCTheory:sec} is devoted to the first applications of our main outcomes. In Subsection \ref{evasive:sec} we study evasive varieties, improving upper bounds of Theorem $2$ of \cite{Dvir-Kollar}. In Subsection \ref{CTS-Evasivas:sec} we prove that correct test sequences are densely distributed within evasive varieties of positive dimension, extending to positive dimension the usual treatment in the case of Krull dimension zero. \\
Finally, Section \ref{redes-neuronales:sec} is devoted to neural networks with rational activation functions that do not admit ``Vermeidung von Divisionen'' (\cite{Strassen:vermeidung}).  We introduce the formalism and recall his connections with Computational Algebraic Geometry (through the classic TERA algorithm:  \cite{Cortona}, \cite{Krick-Pardo96}, \cite{Pardo-survey}, \cite{SolvedFast}, \cite{Kronecker-CRAS}, \cite{lower-diophantine:jpaa}, \cite{SLPsElimination} or \cite{ArithmeticNullstellensatz}). Then, we apply our results in Sections \ref{notaciones-basicas:sec}, \ref{erzeugungsgrad:sec} and \ref{VCTheory:sec} to establish upper bounds for the growth function of the class of binary classifiers associated to a neural networks with rational activation function and to study the distribution of correct test sequences for families of rational functions given by parameterized families of neural networks with rational activation function. \\
As the main outcomes are rather technical and use notions not very conventional for many readers of the journal AAECC, we have summarized these main outcomes in the following subsection. We have tried to be precise with notions and statements and we hope this helps the reader to enter into the two frameworks involved in our proposal.

\subsection{Summary of the manuscript}

Given a constructible subset $C \subseteq \A^n(K)$, Heintz's notion of degree was the degree of its Zariski closure $\overline{C}^z \subseteq \A^n(K)$. We denote by $\deg_{z} (C) := \deg(\overline{C}^z)$. Additionally, we consider $\deg_{\rm lci} (C)$ as in Definition \ref{grado-constructibles:def} below, which is based on the decompositions in locally closed irreducible sets described in Lemma \ref{descomposicion-union-irreducibles-distintos:lemma}. Next, we follow Heintz's notion of degree for a finite family of constructible sets as described in Definition \ref{GradoFamiliaConstructibles:def} below:

\begin{definition}[{\bf Degree of a finite family of constructible sets}]
Let $\scrF$ be a finite family of constructible subsets of $\A^n(K)$. Let $\CC$ be a mapping that associates to every $X\in \scrF$ a minimum $LCI$-degree decomposition $\CC(X)$ of $X$ into locally closed irreducible sets according to Definition \ref{grado-constructibles:def}. We then define the finite class of irreducible varieties associated to $\CC$ by the following equality:
$$\scrC(\scrF, \CC):=\{ V\subseteq \A^n (K) \; :\; \exists X\in \scrF, \; \exists W\in \CC(X), \; V=\overline{W}^z\}=\bigcup_{X\in \scrF} \{ \overline{W}^z\; : \; W\in \CC(X)\}.$$
We define the degree of the family $\scrF$ with respect to $\CC$ as the sum of all the degrees of all irreducible varieties $V\in \scrC(\scrF, \CC)$:
    $$\deg\left(\scrF,\CC\right):=\sum_{V\in \scrC(\scrF, \CC)} \deg(V).$$
Finally, we define the degree of the family $\scrF$ as the minimum of these degrees:
\begin{equation*}
\begin{split}
\deg(\scrF):=  \min\{ & \deg(\scrF, \CC)\; : \; {\hbox {\rm $\CC$ associates to each $X\in \scrF$}} \\
& {\hbox {\rm a minimum $LCI$-degree decomposition}}\}.
\end{split}
\end{equation*}
\end{definition}

Following Heintz's thoughts in \cite{Heintz83}, we also consider the $Z$-degree of a finite family of constructible sets (see Definition \ref{GradoZariskiFamiliaConstructibles:def} below):

\begin{definition}[\bf Degree of the Zariski closures of a finite family of constructible sets]
\label{GradoZariskiFamiliaConstructibles:def}
Let $\scrF$ be a finite family of constructible subsets of $\A^n(K)$. We define:
$$\scrD (\scrF) := \{ V \subseteq \A^n(K) \; : \; \exists X \in \scrF, \; {\hbox {\rm $V$ is an irreducible component of $\overline{X}^z$}} \}.$$
The degree of the Zariski closures of the family $\scrF$ is defined as the following quantity:
$$\deg_{z} (\scrF) := \sum_{V \in \scrD (\scrF, C)} \deg(V).$$
\end{definition}

Observe that if we denote by $\overline{\scrF}^z := \{ \overline{X}^z \; : \; X \in \scrF \}$, then $\deg_{z} (\scrF) = \deg(\overline{\scrF}^z)$, which corresponds to the original notion introduced in \cite{Heintz83}. Note also that if $\scrF$ is a class of locally closed subsets of $\A^n(K)$, we have $\deg(\scrF) = \deg_{z} (\scrF)$. Next, we generalize the notion of cell introduced in Section $3$ of \cite{Heintz83} as follows:

\begin{definition}[{\bf $\CH-$definable, $\CH-$cell}]
Let $\CH$ be a finite class of constructible subsets of $\A^n (K)$ and let $C\subseteq \A^n(K)$ be another constructible set.
\begin{itemize}
\item Let $\scrB(\CH, C)$ denote the Boolean algebra of subsets of $C$ generated by $\CH$. A subset $X$ of $C$ is called \emph{$\CH-$definable in $C$} if it belongs to this Boolean algebra.
\item We define an \emph{$\CH-$cell in $C$} as any set $X\subseteq C$ for which there exists $M\subseteq \CH$ such that:
    $$X= C\cap \left( \bigcap_{Y\in M} Y\right) \bigcap \left(\bigcap_{Y\not\in M} (C\setminus Y)\right).$$
    We denote by $\scrZ(\CH, C)$ the class of all non-empty $\CH-$cells in $C$. 
 \item  Finally, we define the class $\text{Gen}(\scrB(\CH, C))$ as the class of all finite families $\CW:=\{W_1,\ldots, W_t\}$ of algebraic varieties such that $\scrB(\CH,C) \subseteq \scrB(\CW, C)$.

\end{itemize}
\end{definition}

Note that $\scrZ(\CH,C)$ defines a finite partition of $C$ and every element in $\scrB(\CH,C)$ is a finite union of $\CH-$cells in $C$. Finally, we introduce the notion of Erzeugungsgrad of a family of constructible sets as follows (see also Definition \ref{erzeugungsgrad:def}):

\begin{definition}[{\bf Erzeugungsgrad}] 
 Let $\CH$ be a finite family of constructible sets of $\A^n(K)$. We define the Erzeugungsgrad of $\CH$  as:
$$\grad(\CH ):= \min \{ \deg(\CW) \; : \; \CW \in \text{Gen}(\scrB(\CH, \A^n(K)))
, \; {\hbox {\rm $\CW$ is a finite family of closed sets}}\}.$$
\end{definition}

Our first main outcome is the following generalization of Theorem $2$ in \cite{Heintz83}. We avoid potential troubles caused by ``$\deg_{z}$'' in \cite{Heintz83} and we extend the bounds for cells defined inside a fixed constructible set $C \subseteq \A^n(K)$. Note that this statement connects combinatorial aspects and Intersection Theory notions through upper bounds which are not necessarily optimal. 

\begin{main}[{\bf Erzeugungsgrad and combinatorial bounds}]\label{cotas-Erzeugunsgrad:teor} Let $C\subseteq \A^n (K)$ be a constructible set and  $\CH$ a finite family of constructible subsets of $\A^n (K)$. Then, we have:
\begin{enumerate}
\item $\sharp(\scrZ(\CH, C))\leq \deg_{\rm lci}(C)(1+ \grad(\CH)))^{\dim(C)}$.
\item $\deg_z(\scrB(\CH, C))\leq \deg_{\rm lci}(C)(1+ \grad(\CH))^{\dim(C)}$.
\item Any $\CH-$definable finite subset of $C$ contains at most $\deg_{\rm lci}(C)(1+ \grad(\CH))^{\dim(C)}$ points.
\item $\sharp(\scrB(\scrH, C))\leq 2^{\deg_{\rm lci}(C)(1+ \grad(\CH)) ^{\dim(C)}}$.
\end{enumerate}
\end{main}

The first application of these combinatorial upper bounds proves that the Erzeugungsgrad of \cite{Heintz83} is the key ingredient to connect affine Intersection Theory over algebraically closed fields and the VC-theory of Computational Learning Theory for families of classifiers given by parameterized families of constructible sets. \\ 
We begin Section \ref{VCTheory:sec} by recalling some of the terminology of the VC-theory in Subsection \ref{basicNotionsVCTheory:subsec}. Among these notions, we introduce the notion of VC-dimension (Definition \ref{VCDimension:def}) that we reproduce here: 

\begin{definition}[{\bf Shattering, VC-dimension}] 
 Let $Y$ be a set and $\scrH\subseteq \{0,1\}^Y$ a family of classifiers. Given a finite subset $X\subseteq Y$, we say that $\scrH$ \emph{shatters} $X$ if for every $F\subseteq X$, there is some classifier $\chi\in \scrH$ such that the restriction to $X$ of $\chi$ equals the restriction to $X$ of $\crctc_F$. Namely, $\scrH$ shatters $X$ if for every $F\subseteq X$ there is some $\chi\in \scrH$ such that:
$$\restr{\chi}{X}= \restr{\crctc_F}{X}.$$
We define the \emph{VC-dimension} of the family of classifiers $\scrH$ as:
$$\dim_{VC}(\scrH):=\max\{\sharp(X)\; : \; {\hbox {\rm $\scrH$ shatters $X$}}\}.$$
\end{definition}

Next, we consider classes of classifiers defined by parameterized families of constructible sets as follows. Let $N,n \in \N$ be two positive integers and $V \subseteq \A^N(K) \times \A^n(K)$ a constructible set. For instance, $V$ may be the graph of a polynomial mapping. We also consider a constructible subset $\Omega \subseteq \A^N(K)$, which is called \emph{the constructible set formed by the parameters of our family of classifiers}. Additionally, we have the two canonical projections restricted to $V$:
\begin{center}
\begin{tikzcd}[column sep=2.5em]
& V \arrow[dl, "\pi_1"'] \arrow[dr, "\pi_2"] & \\
\mathbb{A}^N(K) & & \mathbb{A}^n(K)
\end{tikzcd}
\end{center} \medskip
With these notations, we define the following class of constructible sets:
$$\scrC (V, \Omega) := \{ \pi_2 (\pi_1^{-1} (\{a\})) \; : \; a \in \Omega \}.$$
We also define $\grad(V) = \grad(\{ V \})$ as the Erzeugungsgrad of the class defined by $V$. Finally, we define the class of classifiers given as the characteristic functions defined by the subsets in $\scrC (V, \Omega)$:
$$\scrH (V, \Omega) := \left\{\crctc_U \; : \; U \in \scrC (V, \Omega) \right\}.$$

Our main contribution in Subsection \ref{VCTheory-LocalmenteCerradoParametros:sec} is Theorem \ref{FuncionCrecimientoDimensionConstructibles:thm}, but the most descriptive conclusion is Corollary \ref{VCDimensionConstructibles:corol}, which answer positively the above stated conjecture and Heintz's challenge: Up to some logarithmic quantities, the VC-dimension of a parameterized family of constructible subsets is bounded by the Krull dimension of the space of parameters. Namely, we have:

\begin{corol-main} 
With the same notations and assumptions as above, the following inequality holds:
$$\frac{s}{\log_2(s)+k} - \frac{\log_2 (\deg_{\rm lci} (\Omega))}{\log_2(s)+k} \leq \dim(\Omega),$$
where $s = \dim_{VC} (\scrH(V, \Omega))$ and $k = 1 + \log_2(\grad(V))$.
\end{corol-main}

The rest of the manuscript is devoted to exhibit applications of the previous results. In Subsection \ref{VCAbiertosDistinguidos:subsec}, we apply these results to studying VC-dimension of parameterized classes of polynomials. \\
We present our main application in Subsection \ref{CTS-Evasivas:sec}. We prove that correct test sequences are densely distributed among any irreducible \emph{evasive variety} of \emph{positive dimension} with respect to any well-behaved probability distribution. Roughly speaking, an evasive variety (see Definition \ref{evasivas:def}) is an algebraic variety that does not intersect any complete intersection variety defined by polynomials in some constructible set. The key ingredient in the proof is the connection between the Krull dimension and the VC-dimension in the case of parameterized classes of polynomials. Note that classical studies on the density of correct test sequences have been done for zero-dimensional sets (as in \cite{HeintzSchnorr}, \cite{Krick-Pardo96} and references therein). The novelty here is that we extend these results to positive dimension and to any well-behaved probability distribution.
\begin{main}
Let $\Omega\subseteq P_d^{K}(X_1,\ldots, X_n)$ be a constructible set of polynomials of degree at most $d$. Let $V\subseteq \A^n(K)$ be an irreducible variety of positive dimension that is evasive for hypersurfaces in $\Omega$. Let $\scrB \subseteq 2^V$ be any $\sigma-$algebra that contains the Borel subsets of $V$ with respect to the Zariski topology. Let $\mu: \scrB \longrightarrow [0,1]$ be a probability distribution on $\scrB$ that satisfies the following property: for every Zariski closed subset $A \subseteq V$, if $\dim(A) < \dim(V)$, then $\mu(A) =0$. Let $L \in \N$ be a positive integer,  $\scrB^{\otimes L}$  the $\sigma-$algebra in the product $V^{L}$ induced by $(V, \scrB)$ and $\mu^{\otimes L}$ the probability distribution defined on $\scrB^{\otimes L}$ by $\mu : \scrB \longrightarrow [0,1]$. Let $CTS(\Omega, V, L)$ be the class of all sequences ${\bf Q} \in V^L$ that are a correct test sequences for $\Omega$ with respect to $\{0\}$. Assume that the following inequality holds:
\begin{equation}
64 \left(1+ \frac{1+ \log(\deg_{\rm lci} (\Omega))}{\dim(\Omega)} + \log(L(d+1)) \right) < \frac{L}{\dim(\Omega)} 
\end{equation}
where $\log$ stands for the natural logarithm. Then, the following inequality holds:
$${\rm Prob}_{x \in V^L}\left[ x \in CTS(\Omega, V, L) \right] \geq 1 - \frac{1}{\deg_{\rm lci} (\Omega) e^{\dim(\Omega)}}.$$
\end{main}
Observe that in the previous theorem, we obtain the same level of probability error as in the zero-dimensional case (cf. Corollary 5.6 in \cite{PardoSebastian}). \\
Section \ref{redes-neuronales:sec} is devoted to applying the previous ideas to study parameterized families of neural networks with rational activation functions. Neural networks are the most successful and widely used data structure in Computational Learning Theory. We have tried to fix its syntax and semantics in Subsections \ref{syntax-neuralnetwork:subsec} and \ref{semantics-neuralnetwork:subsec}, due to the lack of an appropiate reference. Despite its enormous success, we have not found a precise mathematical introduction, since most references in the field just contain ``ad hoc'' descriptions. Once the mathematical formalism has been established, we focus on the class of \emph{activation functions} used by a family of neural networks. We specifically focus on algebraic neural networks, as they are most related to our previous study. For instance, if the activation function is the function $\varphi(t) = t^2$, the class of neural networks becomes the class of \emph{evaluation schemes}, as introduced in Section 3 of \cite{Krick-Pardo96}. These polynomial evaluation schemes are central in Elimination Theory, since all elimination polynomials admit an encoding as well-paralelizable neural networks with activation function $\varphi(t) = t^2$. This is the underlying idea behind the classic TERA algorithm. \\
In these pages, we focus on the class of neural networks with rational activation function (i.e. the activation function is $\varphi(t) = p(t)/q(t)$, where $p$ and $q$ are two univariate polynomials in $K[t]$). In some cases, rational activation functions are avoidable (this is the case of ``Vermeidung von Divisionen'' technique in \cite{Strassen:vermeidung}). In these cases, we may avoid the presence of divisions and reduce the problem to the case of quadratic activation functions (i.e. $\varphi(t)=t^2$). Thus, the novelty here is to consider neural networks with \emph{non-avoidable} rational activation functions. These are typically the natural data structures to represent multi-variate rational functions (in $K(X_1, \ldots, X_n)$, for instance). Our main contributions in this context are presented in Subsection \ref{NNrational:subsec}, where we show upper bounds for the growth function of the associated class of binary classifiers in Corollary \ref{FuncionCrecimientoRacional:corol} and  study the density of correct test sequences for parameterized families of neural networks with rational activation functions in Corollaries \ref{cuestoresRacional:corol} and \ref{TestNulidadRacionales:corol}.

\section{Basic notions of affine algebraic geometry} \label{notaciones-basicas:sec}

\subsection{Zariski topology and constructible sets}

Let $K$ be an algebraically closed field. Let $K[X_1, \ldots, X_n]$ be the ring of polynomials in the set of variables $\{ X_1, \ldots, X_n \}$ with coefficients in $K$. We denote by $\A^n (K)$ the affine space of dimension $n$ over $K$. Given a finite set of polynomials $\{ f_1, \ldots, f_s \}$, we denote by $V_{\A} (f_1, \ldots, f_s) \subseteq \A^{n} (K)$ the algebraic variety of the common zeros in $\A^n(K)$ of these polynomials, i.e.
$$V_{\A}(f_1, \ldots, f_s):=\{ x\in \A^n(K) \; : \; f_1(x)=f_2(x)=\ldots = f_s(x)= 0 \}.$$
Given an ideal ${\frak a}\subseteq K[X_1,\ldots, X_n]$, we also denote by $V_{\A} ({\frak a}) \subseteq \A^n(K)$ the set of common zeros of all polynomials in ${\frak a}$. Note that if ${\frak a}$ is the ideal generated by the finite set $\{f_1, \ldots, f_s \}$ (i.e. ${\frak a} = (f_1, \ldots, f_s)$), then we have $V_{\A} ({\frak a}) = V_{\A} (f_1, \ldots, f_s)$. \\
There is a unique topology in $\A^n(K)$ whose closed sets are affine algebraic varieties. This topology is usually known as the \emph{Zariski topology} on $\A^n(K)$. Given $S\subseteq \A^n(K)$, we denote by $\overline{S}^z$ the closure of $S$ with respect to the Zariski topology on $\A^n (K)$. For every subset $X\subseteq \A^n (K)$, the topology induced on $X$ by the Zariski topology of $\A^n(K)$ will be called the Zariski topology of $X$. We then use the terms \emph{Zariski open in $X$}
or \emph{Zariski closed in $X$} to describe open and closed sets in $X$ with respect to its Zariski topology. Moreover, the Zariski topology is quasi-compact. Note that all the results of this manuscript are also valid for the Zariski topology of the $\kappa$-definable sets on $\A^n(K)$, where $\kappa$ is a field and $K$ its algebraic closure.  \\
 In particular, the open sets of this topology can be written as finite unions of complements of hypersurfaces. Let $f \in K[X_1, \ldots, X_n]$ be a polynomial, we define the  \emph{distinguished open set defined by} $f$ as the open Zariski subset given by:
\begin{equation} \label{abierto-distinguido:eqn}
D(f) := \{ x \in \A^n (K) \; : \; f(x) \neq 0 \} = \A^n (K) \setminus V_\A(f).
\end{equation}
Note that the family of distinguished open sets forms a basis of open sets for the Zariski topology and, as a consequence of Hilbert's Basis Theorem, every open set $U \subseteq \A^n (K)$ with respect to the Zariski topology is a finite union of distinguished open sets, i.e.
$$U := D(f_1) \cup \ldots \cup D(f_s),$$
where $f_1, \ldots, f_s \in K[X_1, \ldots, X_n]$. \\
A subset $V \subseteq \A^n(K)$ is said to be \emph{locally closed} if it can be expressed as the intersection of an open subset and a closed subset with respect to the Zariski topology. A locally closed subset $V \subseteq \A^n (K)$ is called \emph{irreducible} if it is an open subset of a closed irreducible variety of $\A^n (K)$. Since the affine Zariski topology is Noetherian, locally closed sets admit a minimal decomposition as finite union of locally closed irreducible sets. Uniqueness up to permutation of these minimal decompositions of locally closed sets allows us to use the term \emph{locally closed irreducible components}. \\
Finally, finite unions of locally closed sets are called \emph{constructible sets}. Note that constructible sets are not necessarily locally closed sets. For instance, in \cite{PardoSebastian}, the authors introduced the following constructible set, known as \emph{La Croix de Berny} in \cite{PardoSebastian}, which is not locally closed:
\begin{equation}\label{CroixDeBerny:eq}
C := \pi(W) = \{ (x,y) \in \A^2(\C) \; : \; xy \neq 0 \} \cup \{ (0,1), (0,-1), (1,0), (-1,0) \},
\end{equation}
where $W:= \{ (x,y,z) \in \A^3(\C)\; : \; zxy + (x^2+y^2-1)=0\}$ is a cubic irreducible hypersurface of $\A^3(\C)$ and $\pi: \A^3(\C)\longrightarrow \A^2(\C)$ is the projection that ``forgets'' the last coordinate. The class of constructible sets is closed under finite unions, finite intersections and complementation. The relevance of constructible sets comes from the following result, which can be found in \cite{Chevalley}
\begin{theorem}[\cite{Chevalley}] \label{teorema-chevalley:teor}
A set $C\subseteq \A^n(K)$ is constructible if and only if there is some $m\in \N$ and some algebraic variety $V\subseteq \A^{n+m}(K)$ such that $C:=\pi(V)$, where $\pi:\A^{n+m}(K)\longrightarrow \A^n(K)$ is the canonical projection that ``forgets'' the last $m$ coordinates.
\end{theorem}

Obviously, graphs of polynomial mappings defined on constructible sets are constructible. Hence, images of constructible sets under polynomial mappings are also constructible. Conversely, it can be easily proven that a subset $C \subseteq \A^{n} (K)$ is constructible if and only if there exist $m \in \N$, an algebraic variety $V \subseteq \A^m (K)$ and a polynomial mapping $\varphi: \A^m(K) \longrightarrow  \A^n (K)$ such that $C = \varphi (V)$. In other words, constructible sets are simply the images of affine algebraic varieties under polynomial mappings. \\
In \cite{Chevalley}, C. Chevalley discussed the irreducible components of a constructible set as the irreducible components of its Zariski closure. This approach does not provide a good definition of degree of a constructible set and led to the flaw in Remark 2 (1) of \cite{Heintz83} and its wrong consequences (cf. also \cite{Heintz85}). For this reason, in \cite{PardoSebastian}, the authors reconsidered the idea of decomposing constructible sets as finite unions of locally closed irreducible sets as stated in the following proposition:

\begin{lemma}[\cite{PardoSebastian}] \label{descomposicion-union-irreducibles-distintos:lemma} Let $C\subseteq \A^n(K)$ be a constructible set. Then, there is a finite set $\CC:=\{ U_1\cap V_1, \ldots, U_s\cap V_s\}$ of locally closed irreducible sets such that the following properties hold:
\begin{equation}\label{descomposicion-union-irreducibles-distintos0:eqn}
C:=(U_1\cap V_1)\cup \cdots \cup (U_s\cap V_s),
\end{equation}
and
\begin{enumerate}
\item $V_i$ is an irreducible algebraic variety in $\A^n (K)$,
\item $U_i$ is the maximum of the Zariski open subsets $O_i\subseteq \A^n (K)$  such that $O_i\cap V_i\subseteq C$,
\item $V_i\not= V_j$.
\end{enumerate}
\end{lemma}

If $C \subseteq \A^n(K)$ is a locally closed subset, a decomposition of $C$ as in Identity (\ref{descomposicion-union-irreducibles-distintos0:eqn}), satisfying conditions $i)$, $ii)$ and $iii)$ of Lemma \ref{descomposicion-union-irreducibles-distintos:lemma}, yields a unique class of irreducible varieties $\{V_1, \ldots, V_s \}$. These irreducible varieties  $\{V_1, \ldots, V_s \}$ are called the \emph{irreducible components of the locally closed set $C$}. Nevertheless, this uniqueness does not hold for any constructible subset as shown in \cite{PardoSebastian}. More precisely, let $C \subseteq \A^2 (\C)$ be the constructible set described in Identity (\ref{CroixDeBerny:eq}). We observe that $C$ admits several decompositions into locally closed irreducible subsets that yield different classes of ``irreducible components'', depending on the chosen decomposition. For instance, we have the following two decompositions of $C$ into locally closed irreducible subsets:
$$C=\{ (x,y)\in \A^2(\C)\; :\; xy \not= 0\} \cup \{(1,0)\}\cup \{(-1,0)\} \cup \{(0,1)\} \cup \{(0,-1)\}$$
and
$$C = \{ (x,y)\in \A^2(\C)\; :\; xy\not=0\} \cup \{ (x,y)\in \A^2(\C) \; :\; x^2+y^2-1=0\}.$$
Thus, the decomposition of constructible sets into locally closed irreducible subsets is not unique. Moreover, if $C \subseteq \A^n(K)$ is a constructible set and we have a locally closed irreducible decomposition of $C$ (as in Lemma \ref{descomposicion-union-irreducibles-distintos:lemma}):
$$C =(U_1\cap V_1)\cup \cdots \cup (U_s\cap V_s), $$
we immediately obtain a decomposition of $\overline{C}^z$ as a finite union of irreducible algebraic varieties as follows:
$$\overline{C}^z = V_1 \cup \cdots \cup V_s.$$ 
Nervertheless, some of these irreducible sets are embedded as components of $\overline{C}^z$. And, however, their locally closed parts $U_i \cap V_i$ are essential to define $C$ (see Example 2.3 and Proposition 2.5 of \cite{PardoSebastian} for further details). \\
Finally, we define the \emph{dimension} of a constructible set $C\subseteq \A^n (K)$ as its Krull dimension as topological space. It is well-known that Krull dimension is sub-additive and does not increase under polynomial mappings (i.e., given a constructible set $C \subseteq \A^n (K)$ and a polynomial mapping  $\varphi: \A^n(K) \longrightarrow  \A^m (K)$, then $\dim(\varphi(C)) \leq \dim(C)$).

\subsection{Degree of locally closed sets according to \cite{Heintz83} and their Bézout's Inequality} \label{gradoHeintz:subsec}

We follow the terminology of \cite{Heintz83}. Let $\A^{rn}(K) = \CM_{r \times n} (K)$ be the space of $r\times n$ matrices with coordinates in $K$. We consider the following Zariski open  set:
$$\CG (n,r) := \{ M \in \CM_{r \times n} (K) \; : \; \rank (M) =r \}. $$
Given $b \in \A^r (K)$ and $M \in \CM_{r \times n} (K)$, we consider the linear affine variety that they define:
$$\G (M,b) := \{ x \in \A^n (K) \; : \; M x^{t} - b = 0 \},$$
where $x^t$ is the transpose of $x=(x_1,\ldots, x_n)$. If $M \in \CG (n,r)$, then $\G(M, b) \neq \emptyset$ is a linear affine variety of dimension $n-r$. We denote by $\G(n,r)$ the ``Grassmannian'' of all linear affine subvarieties $L \subseteq \A^n (K)$ of dimension $n-r$ (i.e. of co-dimension $r$).  Note that we have the following onto mapping:
\begin{equation} \label{topologiaMatriz:eq}
\begin{matrix}
G: & \CG(n,r) \times \A^{r} (K) & \longrightarrow & \G(n,r)\\
& (M, b) & \longmapsto & \G(M,b)\end{matrix}.
\end{equation}
Thus, we can endow $\G(n,r)$ with the final topology induced by this onto mapping. This final topology will be called \emph{the Zariski topology on $\G(n,r)$}. Let $V \subseteq \A^n(K)$ be a locally closed irreducible subset of dimension $r$. We introduce the following class: 
$$\G(V):=\{ A \in \G(n,r)\; :\; \sharp(A\cap V)< \infty\}.$$
The next proposition is a fundamental result that immediately follows from \cite{Heintz83}:

\begin{proposition}\label{interpretacion-geometrica-grado:prop}
With the same notations and assumptions as above, we have:
\begin{enumerate}
\item The class  $\G(V)$ is non-empty and contains an open subset of  $\G(n,r)$.
\item The maximum $\max\{ \sharp\left(A \cap V\right)\; : \; A \in \G(V)\}$ is finite.
\item There is a dense open subset $\CU\subseteq \G(n,r)$ with respect to the above defined topology, such that for all $A \in \CU$ we have:
$$\sharp\left(A\cap V \right)= \max\{\sharp\left(T\cap V\right)\; : \; T\in \G(V)\}.$$
\end{enumerate}
\end{proposition}

We are now in conditions to define the degree of a locally closed set:

\begin{definition}[{\bf Degree  of a locally closed subset}]\label{grado-localmente-cerrados:def}
Let $V\subseteq \A^n(K)$ be a locally closed irreducible subset. We define the degree of $V$ as the following quantity:
$$\deg(V):=\max\{ \sharp\left( A \cap V\right) \; : \; A \in \G(n,r), \; \sharp(A\cap V)< \infty\}.$$
Let $W\subseteq \A^n(K)$ be any locally closed subset and let  $C_1,\ldots, C_s$ be its locally closed irreducible components. The degree  of $W$ is defined as:
$$\deg(W):=\sum_{i=1}^s \deg(C_i).$$
\end{definition}

The key contribution of \cite{Heintz83} in this context is the proof of the famous \emph{Bézout's Inequality} for algebraic varieties (this proof can be easily extended to the case of locally closed sets using similar arguments):

\begin{theorem}[{\bf B\'ezout's Inequality for locally closed sets, \cite{Heintz83}}]\label{Bezout-localmente-cerrados:teor}
Let $V, W\subseteq \A^n(K)$ be two locally closed subsets. Then, we have:
$$\deg(V\cap W)\leq \deg(V)\deg(W).$$
\end{theorem}

Next, we summarize some properties of the degree that will be used throughout the manuscript:

\begin{proposition}\label{propiedades-basicas-grado-lc:prop}
 With the same notations and assumptions as above, we have:
\begin{enumerate}
\item For every locally closed subset $V\subseteq \A^n (K)$, its degree agrees with the degree of its Zariski closure, i.e.
$$\deg(V)=\deg(\overline{V}^z).$$
 \item The degree of a finite set $C\subseteq \A^n (K)$ equals its cardinality:
 $$\deg(C)= \sharp(C).$$
\item The number of irreducible components of a locally closed set is bounded by its degree.
\item The degree of any linear affine variety is $1$.
\item The degree of locally closed sets is invariant by linear or affine isomorphisms.
\item  For every non-constant $f\in K[X_1,\ldots, X_n]$, the degree of the hypersurface $V(f)$ is at most the degree of the polynomial $\deg(f)$ and $\deg(V_\A(f))=\deg(f)$ if and only if $f$ is square--free.
\item If $V\subseteq \A^n(K)$ is a locally closed subset and  $A\subseteq \A^n(K)$ is a linear affine variety, then:
$$\deg(V\cap A)\leq \deg(V).$$
\item If $V\subseteq \A^n (K)$ and $W\subseteq \A^m (K)$ are two locally closed sets, then:
$$\deg(V\times W) = \deg(V) \deg(W).$$
\item If $\ell: \A^n (K)\longrightarrow \A^m(K)$ is a linear mapping and $V\subseteq \A^n(K)$ a locally closed subset, then:
$$\deg(\overline{\ell(V)}^z) \leq \deg(V).$$
\item If $\varphi =(f_1, \ldots, f_m): \A^n(K) \longrightarrow  \A^m (K)$ is a polynomial mapping, i.e. $f_1, \ldots, f_m \in K[X_1, \ldots, X_n]$, and $V$ a locally closed set, then:
$$\deg( \overline{\varphi(V) }^z ) \leq \deg(V) (\max \{ \deg(f_i) \; : \; 1 \leq m     \})^{\dim (V)}.$$
\item If $V_1, \ldots, V_s \subseteq \A^n (K)$ is a family of locally closed sets, then:
$$\deg \left( \bigcap_{i=1}^s V_i \right) \leq \deg(V_1) (\max \{ \deg(V_i) \; : \; 1 \leq i \leq s  \})^{\dim(V_1)}.$$ 
\end{enumerate}
\end{proposition}

\subsection{Degrees of construtible sets according to \cite{PardoSebastian} and their Bézout's Inequalities} \label{gradoPardoSebastian:subsec}

In  \cite{Heintz83}, the author introduces a first notion of degree of construcible sets that we will call $Z-$degree:
\begin{definition} \label{zGrado:def}
Let $C \subseteq \A^n(K)$ be a constructible set. We define the $Z-$degree of $C$ as the degree of its Zariski closure, i.e.
$$\deg_{z} (C) = \deg (\overline{C}^z).$$ 
\end{definition}
The problem with the $Z-$degree is that, although it preserves some of the basic properties of the degree, it fails in an essential property: \emph{the $Z-$degree does not satisfy the Bezout's Inequality}. Let $C$ be La Croix de Berny (see Identity (\ref{CroixDeBerny:eq})) and consider the linear affine variety $L := \{ (x,0) \in \A^2 (\C) \}$. Then, we have:
$$\deg (C \cap L) = \deg_{z} (C \cap L) = 2 \not\leq \deg_{z} (C) \deg_{z}(L) = 1 \cdot 1 = 1.$$
In \cite{PardoSebastian}, the authors introduce two distinct notions for the degree of a constructible set. The first notion involves decompositions of constructible sets into locally closed sets (as in Lemma \ref{descomposicion-union-irreducibles-distintos:lemma}). One possible approach is to define a notion of degree based on one of these decompositions. The problem, however, is that, as we have already mentioned, such decompositions are not unique (see Example 2.3 and Remark 2.18 of \cite{PardoSebastian} for further details). Therefore, the first notion of degree of a constructible set will be the minimal degree of a presentation as union of locally closed irreducible sets:
\begin{definition}[{\bf $LCI-$degree of a constructible set}]\label{grado-constructibles:def} Let $C\subseteq \A^n(K)$ be a constructible subset. Let $\CC:=\{ U_1\cap V_1,\ldots, U_r\cap V_r\}$ be a decomposition of $C$  as finite union of locally closed irreducible sets that satisfies Lemma \ref{descomposicion-union-irreducibles-distintos:lemma}. We define the degree of $C$ relative to this decomposition as:
 $$\deg(C, \CC):=\sum_{i=1}^r \deg(V_i).$$
Finally, we define the $LCI-$degree of $C$ as the minimum of all the degrees of all decompositions of this kind:
 $$\deg_{\rm lci}(C):=\min\{ \deg(C,\CC)\; :\; {\hbox {\rm $\CC$ is a decomposition of $C$ that satisfies Lemma \ref{descomposicion-union-irreducibles-distintos:lemma}}}\}.$$
We say that a decomposition $\CC:=\{ U_1\cap V_1,\ldots, U_r\cap V_r\}$ of a constructible set $C\subseteq \A^n(K)$  is a minimum $LCI$-degree decomposition of $C$ if $\CC$ satisfies  Lemma \ref{descomposicion-union-irreducibles-distintos:lemma} and also:
$$\deg_{\rm lci}(C)=\sum_{i=1}^r \deg(V_i).$$
 \end{definition}
Since constructible sets are projections of algebraic varieties, the following notion of degree was also introduced in \cite{PardoSebastian}:
\begin{definition}[{\bf $\pi-$degree of a constructible set}] \label{grado-como-proyeccion:def} Let $C\subseteq \A^n(K)$ be a constructible set. We consider the class of all locally closed subsets that project onto $C$. Namely,
for every $m\in \N$, $m\geq n$, we define:
$$\Pi_m(C):=\{ V \subseteq \A^m(K) \; :  {\hbox {\rm $V$ is locally closed and  $\pi(V)=C$}}\},$$
where $\pi: \A^m(K) \longrightarrow \A^n(K)$ is the canonical projection that ``forgets'' the last $m-n$ coordinates. Then, we define:
$$\Pi(C):=\bigcup_{m\geq n} \Pi_m(C),$$
and we define the projection degree (also $\pi-$degree) of $C$ as the following minimum:
$$\deg_\pi(C):=\min \{ \deg_z(V)\; : \; V\in \Pi(C)\}.$$
\end{definition}
Observe that the $\pi-$degree of a constructible set $C$ is the minimum of the $Z-$degrees of all locally closed sets that project onto $C$. \\
In \cite{PardoSebastian}, it is shown that the three notions of degree are different:
\begin{example}[{\bf Three different ``degrees''}] \label{tres-grados-distintos:ej}
In \cite{PardoSebastian}, the authors consider the following quadratic hypersurface:
$$W':=\{(x,y,z)\in \A^3(\C) \; : \; xz+ (y^2-1)=0\},$$
and the constructible subset $C:= \pi (W') \subseteq \A^2 (\C)$,  where $\pi:\A^3(\C)\longrightarrow \A^2(\C)$ is the canonical projection that ``forgets'' the last coordinate. Then, in Example 2.17 of \cite{PardoSebastian}, it is shown that the following inequalities hold:
$$1= \deg_z(C) < 2 = \deg_\pi(C) < \deg_{\rm lci}(C)=3,$$
and the three notions of degree of constructible sets do not coincide.
\end{example}
The next result shows that the three notions of degree agree in the case of locally closed sets:

\begin{proposition} \label{grado-igual-localmente-cerrado:prop}
Let $C \subseteq \A^n (K)$ be a constructible set. If $C$ is locally closed, we have:
$$\deg_{z} (C) = \deg_{\rm lci} (C) = \deg_{\pi} (C).$$
\end{proposition}

Now, we summarize some properties of these notions of degree for constructible sets.

\begin{proposition}
With the above notations, we have:
\begin{enumerate}
\item The three notions of degree are sub-adittive. Namely, given two constructible sets $C, D \subseteq \A^n (K)$, we have:
$$\deg_{z} (C \cup D) \leq \deg_{z} (C) + \deg_{z}(D),$$
$$\deg_{\pi} (C \cup D) \leq \deg_{\pi} (C) + \deg_{\pi}(D),$$
$$\deg_{\rm lci} (C \cup D) \leq \deg_{\rm lci} (C) + \deg_{\rm lci}(D).$$

\item If $C \subseteq \A^n (K)$ is a contructible set, then we have the following inequalities:
$$\deg_{z} (C) \leq  \deg_{\pi} (C) \leq \deg_{\rm lci} (C).$$
\end{enumerate}
\end{proposition}

In what concerns images under linear maps, we have:

\begin{proposition} \label{grado-imagen-lineal:prop}
Let $C \subseteq \A^n(K)$ be a constructible set and $\ell: \A^n(K) \longrightarrow \A^m(K)$ be a linear map. We have:
\begin{enumerate}
\item $ \deg(\overline{\ell(C) }^{z})  = \deg_{z} (\ell (C)) \leq \deg_{\rm lci} (C).$
\item $ \deg(\overline{\ell(C) }^{z})  =  \deg_{z} (\ell (C)) \leq \deg_{\pi} (\ell(C)) \leq \deg_{\pi} (C) \leq \deg_{\rm lci} (C). $
\end{enumerate}
\end{proposition}

The three notions of degree have a good behaviour with respect to the Cartesian product:

\begin{proposition}
Let $C, D \subseteq \A^n (K)$ be two constructible sets, we have:
$$\deg_{z} (C \times D) \leq \deg_{z}(C) \deg_{z} (D),$$
$$\deg_{\pi} (C \times D) \leq \deg_{\pi} (C) \deg_{\pi} (D),$$
$$\deg_{\rm lci} (C \times D) \leq \deg_{\rm lci} (C) \deg_{\rm lci} (D). $$
\end{proposition}

In the case of images under polynomial mappings, we have the following results:

\begin{proposition} \label{grado-imagen-constructible:prop}
Let $\varphi:=(\varphi_1,\ldots, \varphi_m): \A^n(K)\longrightarrow \A^m(K)$ be a polynomial mapping.  Assume that for every $i$, $1\leq i\leq m$, $\varphi_i\in K[X_1,\ldots, X_n]$ is a polynomial of degree at most $d$, where $d\geq 1$. Let $C\subseteq \A^n(K)$ be a constructible subset. Then, we have:
$$ \deg(\overline{\varphi(C)}^{z}) =  \deg_{z} (\varphi(C)) \leq \deg_{\rm lci} (C) d^{\dim(C)}, $$
$$\deg(\overline{\varphi(C)}^{z}) = \deg_{z} (\varphi(C)) \leq \deg_{\pi} (\varphi(C)) \leq \deg_{\pi} (C) d^m \leq \deg_{\rm lci} (C) d^m.$$
\end{proposition}

Finally, we have that both $LCI-$degree and $\pi-$degree satisfy a \emph{Bézout's Inequality}:

\begin{theorem}[{\bf B\'ezout's Inequalities for constructible sets}]\label{Desigualdad-Bezout-constructibles:teor}
Let $C, D\subseteq \A^n(K)$ be two constructible sets. We have:
$$\deg_{\rm lci}(C\cap D) \leq \deg_{\rm lci}(C)\deg_{\rm lci}(D),$$
$$\deg_\pi(C\cap D) \leq \deg_\pi(C)\deg_\pi(D).$$
\end{theorem}

\section{Erzeugungsgrad: revisiting Section 3 of \cite{Heintz83}, combined with constructible sets} \label{erzeugungsgrad:sec}

In \cite{Heintz83}, along with the notion of degree of affine algebraic varieties and locally closed sets, J. Heintz introduced the notion of \emph{``Erzeugungsgrad''} (or generation degree). This notion was important for \cite{Heintz83} to bound the number of non-empty cells occurring after a process of quantifier elimination. The notion is also relevant since, independently, it yields information about the ``growth function'' in Computational Learning Theory, especially in the case of binary classifiers, VC-dimension and the Sauer-Shelah-Perles Lemma. In other words, the Erzeugungsgrad is closely related to the notion of \emph{shattering} finite sets. This relation will be discussed in forthcoming sections. \\
Unfortunately, there is a small flaw in the statement of the Erzeugungsgrad in \cite{Heintz83}, which slightly affects its consequences. J. Heintz used the notion of $Z$-degree of constructible sets (see Definition \ref{zGrado:def}) as starting point to define this concept. As we have already observed, the $Z$-degree does not satisfy Bézout's Inequality and, hence, its use may affect both the arguments and the main bounds. Therefore, the aim of this section is twofold. Firstly, we define the Erzeugungsgrad using accurate notions of degree of constructible sets (in particular, we use $\deg_{\rm lci} (C)$ as introduced in the previous section). Next, we study upper bounds for the Erzeugungsgrad of a finite family of constructible sets, providing a more complete study of this notion.

\subsection{Basic definitions and preparatory results}

\begin{definition}[{\bf Degree of a finite family of constructible sets}] \label{GradoFamiliaConstructibles:def}
Let $\scrF$ be a finite family of constructible subsets of $\A^n(K)$. Let $\CC$ be a mapping that associates to every $X\in \scrF$ a minimum $LCI$-degree decomposition $\CC(X)$ of $X$ into locally closed irreducible sets according to Definition \ref{grado-constructibles:def}. We then define the finite class of irreducible varieties associated to $\CC$ by the following equality:
$$\scrC(\scrF, \CC):=\{ V\subseteq \A^n (K) \; :\; \exists X\in \scrF, \; \exists W\in \CC(X), \; V=\overline{W}^z\}=\bigcup_{X\in \scrF} \{ \overline{W}^z\; : \; W\in \CC(X)\}.$$
We define the degree of the family $\scrF$ with respect to $\CC$ as the sum of all the degrees of all irreducible varieties $V\in \scrC(\scrF, \CC)$:
    $$\deg\left(\scrF,\CC\right):=\sum_{V\in \scrC(\scrF, \CC)} \deg(V).$$
Finally, we define the degree of the family $\scrF$ as the minimum of these degrees:
\begin{equation*}
\begin{split}
\deg(\scrF):=  \min\{ & \deg(\scrF, \CC)\; : \; {\hbox {\rm $\CC$ associates to each $X\in \scrF$}} \\
& {\hbox {\rm a minimum $LCI$-degree decomposition}}\}.
\end{split}
\end{equation*}
 We also define $\Mdeg\left(\scrF \right)$ as the maximum of the $LCI$-degrees of any $X\in \scrF$:
    $$\Mdeg\left(\scrF\right):=\max\{\deg_{\rm lci}(X)\; :\; X\in \scrF\}.$$
Additionally, let $C \subseteq \A^n(K)$ be a constructible subset. Given $\scrF$ as above, let us denote by $\restr{\scrF}{C}$ the restrictions of the elements in $\scrF$ to $C$. Namely,
$$\restr{\scrF}{C} := \{ C \cap X \; : \; X \in \scrF \}.$$
We denote by $\deg(\scrF,C)$ the degree of $\restr{\scrF}{C}$.
\end{definition}

\begin{remark}
Let the reader observe that the assignment $\CC(X)$ is based on a minimum $LCI$-degree decomposition and not on the irreducible components of the Zariski closure $\overline{X}^z$ of $X$. The reason is that the irreducible components of $\overline{X}^z$ are among the irreducible components of a minimum $LCI$-degree decomposition. Nevertheless, a minimum $LCI$-degree decomposition may contain relevant locally closed irreducible sets which vanish (as embedded sets) when taking the Zariski closure.
\end{remark}

\begin{example}[{\bf Dependence of the assignement $\CC$}] The quantity $\deg(\scrF, \CC)$ depends on the mapping $\CC$, as shown by the following example. Let us consider again Example \ref{tres-grados-distintos:ej}. We had the following quadratic hypersurface:
\begin{equation}
W':=\{(x,y,z)\in \A^3(\C) \; : \; xz+ (y^2-1)=0\},
\end{equation}
and the construtible subset $C :=\pi(W)\subseteq \A^2(\C)$, where $\pi:\A^3(\C)\longrightarrow \A^2(\C)$ is the canonical projection that ``forgets'' the last coordinate. Let $D$ be the Zariski open set defined as:
$$D:= \A^{2}(\C) \setminus V_{A}(X) = \{ (x,y) \in \A^2(\C) \; :\; x\not=0\}.$$
 We have the following two different decompositions of $C$ that minimize its $LCI-$degree:
$$C = D \cup\{ (x,y)\in \A^2(\C) \; : \; x^2+ y^2-1=0\},$$
$$C= D \cup\{ (x,y)\in \A^2 (\C)\; : \; (x+1)- y^2=0\}.$$
Next, we may define $\scrF:=\{ C, V\}$, where $V:=V_{\A}((X+1)- Y^2)$ is the irreducible algebraic variety given as the zero set in $\A^2(\C)$ of the polynomial $h(X,Y) = (X+1)- Y^2 \in \C[X,Y]$. We consider two distinct mappings, $\CC_1$ and $\CC_2$, on $\scrF$, defined as follows:
$$\CC_1(C):= \{ D, V_\A(X^2+ Y^2-1)\}, \; \CC_1(V)=V, \qquad \CC_2(C):=\{ D, V\}, \; \CC_2(V)=V.$$
These two mappings yield two different classes of irreducible components:
$$\scrC(\scrF, \CC_1)= \{\A^2(\C) , V_\A(X^2+ Y^2-1),V\}, \qquad \scrC(\scrF, \CC_2)=\{\A^2(\C), V\}.$$
They result in two different degrees:
$$\deg(\scrF, \CC_1)= \sum_{W\in \scrC(\scrF, \CC_1)} \deg(W) = \deg(\A^2(\C)) + \deg(V_\A(X^2+ Y^2-1)) + \deg(V)= 1 +2+2 = 5,$$
$$\deg(\scrF, \CC_2) =\sum_{W\in \scrC(\scrF, \CC_2)} \deg(W)= \deg(\A^2(\C)) + \deg(V) = 1 +2 = 3.$$
\end{example}

\begin{remark} \label{DesigualdadesMaxDeg}
 Observe also that the following inequalities hold, but equality among them does not always occur:
$$\Mdeg(\scrF)\leq \deg(\scrF) \leq \sum_{D\in \scrF}\deg_{\rm lci}(D)\leq \sharp(\scrF) \cdot \Mdeg(\scrF).$$
\end{remark}

\begin{definition}[\bf Degree of the Zariski closures of a finite family of constructible sets]
\label{GradoZariskiFamiliaConstructibles:def}
Let $\scrF$ be a finite family of constructible subsets of $\A^n(K)$. We define:
$$\scrD (\scrF) := \{ V \subseteq \A^n(K) \; : \; \exists X \in \scrF, \; {\hbox {\rm $V$ is an irreducible component of $\overline{X}^z$}} \}.$$
The degree of the Zariski closures of the family $\scrF$ is defined as the following quantity:
$$\deg_{z} (\scrF) := \sum_{V \in \scrD (\scrF)} \deg(V).$$
Let $C \subseteq \A^n(K)$ be a constructible set and let $\scrF$ be as above. We define:
$$\deg_z (\scrF,C) := \deg_z (\restr{\scrF}{C}).$$
\end{definition}
Observe that if we denote by $\overline{\scrF}^z := \{ \overline{X}^z \; : \; X \in \scrF \}$ then $\deg_{z} (\scrF) = \deg(\overline{\scrF}^z)$, which corresponds to the original notion introduced in \cite{Heintz83}. 

\begin{proposition}
 Let $\scrF$ and $\scrG$ be two finite families of constructible subsets of $\A^n(K)$ such that $\scrF \subseteq \scrG$. Then, we have:
$$\deg(\scrF) \leq \deg(\scrG), \qquad \qquad \deg_{z} (\scrF) \leq \deg_{z} (\scrG),$$
that is, both degrees are sub--additive functions.
\end{proposition}
\begin{proof}
Let $\CC_1$ be a mapping that associates to every element of $\scrG$ a decomposition into locally closed irreducible components that minimize its $LCI$-degree and verifies the following identity:
$$D = \deg(\scrG) = \deg(\scrG, \CC_1).$$
Next, we consider $\CC_2 := \restr{\CC_1}{\scrF}$, the restriction to $\scrF$ of $\CC_1$. Clearly, if $\CC_1$ associates to every element of $\scrG$ a decomposition of minimal $LCI$-degree, $\CC_2$ does the same for all the elements of $\scrF$. We have the following two classes:
$$\scrC (\scrG, \CC_1) := \{ V \subseteq \A^n(K) \: : \: \exists X \in \scrG, \; \exists W \in \CC_1 (X), \; V = \overline{W}^z \},$$
$$\scrC(\scrF, \CC_2) := \{ V \subseteq \A^n(K) \; : \; \exists X \in \scrF, \; \exists W \in \CC_2 (X), \; V = \overline{W}^z \}.$$
Since $\CC_2 (Y) = \CC_1 (Y)$ for every $Y \in \scrF$, we conclude that $\scrC(\scrF, \CC_2) \subseteq \scrC (\scrG, \CC_1)$ and, therefore:
$$\deg(\scrF, \CC_2) = \sum_{V \in \scrC (\scrF, \CC_2)} \deg(V) \leq \sum_{V \in \scrC(\scrG, \CC_1)} \deg(V) =  \deg(\scrG, \CC_1) = \deg(\scrG) = D.$$
As $\deg(\scrF)$ is the minimum of  $\deg(\scrF, \CC)$ for any mapping $\CC$, we conclude that: $$\deg(\scrF) \leq \deg(\scrG).$$ 
The second inequality is an immediate consequence of the first one since:
$$\deg_{z} (\scrF) = \deg(\overline{\scrF}^z), \qquad \deg_{z} (\scrG) = \deg (\overline{\scrG}^z)$$
and $\overline{\scrF}^z \subseteq \overline{\scrG}^z$.
\end{proof}

\begin{definition}[{\bf $\CH-$definable, $\CH-$cell}]
Let $\CH$ be a finite class of constructible subsets of $\A^n (K)$ and let $C\subseteq \A^n(K)$ be another constructible set.
\begin{itemize}
\item Let $\scrB(\CH, C)$ denote the Boolean algebra of subsets of $C$ generated by $\CH$. A subset $X$ of $C$ is called \emph{$\CH-$definable in $C$} if it belongs to this Boolean algebra, i.e. $X\in \scrB(\CH, C)$. In the case $C=\A^n (K)$, we simply write $\scrB(\CH):=\scrB(\CH, \A^n (K))$.
\item We define an \emph{$\CH-$cell in $C$} as any set $X\subseteq C$ for which there exists $M\subseteq \CH$ such that:
    $$X= C\cap \left( \bigcap_{Y\in M} Y\right) \bigcap \left(\bigcap_{Y\not\in M} (C\setminus Y)\right).$$
    We denote by $\scrZ(\CH, C)$ the class of all non-empty $\CH-$cells in $C$. In the case $C=\A^n (K)$, we simply write $\scrZ(\CH):=\scrZ(\CH, \A^n (K))$.
\item  Finally, we define the class $\text{Gen}(\scrB(\CH, C))$ as the class of all finite families $\CW:=\{W_1,\ldots, W_t\}$ of algebraic varieties such that $\scrB(\CH,C) \subseteq \scrB(\CW, C)$ (i.e. closed generators of the Boolean algebra $\scrB(\CH, C)$). In the case $C=\A^n (K)$, we simply write $\text{Gen}(\scrB(\CH)) := \text{Gen}(\scrB(\CH, \A^n(K))$.
\end{itemize}

\end{definition}

Observe that $\scrZ(\CH,C)$ defines a finite partition of $C$ and every element in $\scrB(\CH,C)$ is a finite union of $\CH-$cells in $C$.

\begin{definition}[{\bf Erzeugungsgrad}] \label{erzeugungsgrad:def}
 Let $\CH$ be a finite family of constructible sets. We define the Erzeugungsgrad of $\CH$  as:
$$\grad(\CH ):= \min \{ \deg(\CW) \; : \; \CW \in \text{Gen}(\scrB(\CH))
, \; {\hbox {\rm $\CW$ is a finite family of closed sets}}\}.$$
\end{definition}

\begin{remark} \label{ErzeugungsgradIrreducible:rmk}
Given a finite family of constructible sets $\CH$, there exists a finite family $\CW$ of irreducible algebraic varieties such that $\CH \subseteq \scrB (\CW)$ and:
$$\grad(\CH) = \sum_{W \in \CW} \deg(W).$$
To see this, let $\widetilde{\CW}$ be a finite family of algebraic varieties such that $\CH \subseteq \scrB (\widetilde{\CW})$ and $\grad(\CH) = \deg(\widetilde{\CW})$. As in Definition \ref{GradoZariskiFamiliaConstructibles:def}, let $\scrD(\widetilde{\CW})$ be the set of irreducible components of the elements of $\widetilde{\CW}$. We obviously have $\CH \subseteq \scrB(\scrD(\widetilde{W}))$, since $\widetilde{\CW} \subseteq \scrB(\scrD(\widetilde{\CW}))$ and $\CH \subseteq \scrB (\widetilde{\CW})$, and:
$$\deg(\widetilde{\CW}) = \sum_{W \in \scrD(\widetilde{\CW})} \deg(W)$$
Thus, $\CW = \scrD(\widetilde{\CW})$ is a finite family of irreducible algebraic varieties such that $\CH \subseteq \scrB (\CW)$ and:
$$\grad(\CH) = \sum_{W \in \CW} \deg(W).$$
\end{remark}

The following lemma shows that if $\scrF$ is a class of locally closed subsets of $\A^n(K)$, then $\deg(\scrF) = \deg_{z} (\scrF)$.

\begin{lemma} \label{GradoWSumaComponentesIrreducibles:lemma}
Let $\CW$ be a finite family of locally closed subsets of $\A^n(K)$. Let $\scrD(\CW)$ be the finite class of irreducible algebraic varieties introduced in Definition \ref{GradoZariskiFamiliaConstructibles:def}. Namely:
$$\scrD(\CW) := \left\{ V \subseteq \A^n(K) \; : \; \exists W \in \CW, \; {\hbox {\rm $V$ is an irreducible component of $\overline{W}^z$}}   \right\}.$$
Then, we have:
$$\deg(\CW) = \sum_{V \in \scrD (\CW)} \deg(V).$$
\end{lemma}

\begin{proof}
Let $\CC$ be any assignment that associates to every $W \in \CW$ a minimum $LCI$-degree decomposition. For each $W \in \CW$, let us denote by $\scrC(W, \CC) := \scrC(\{W \}, \CC)$, according to Definition \ref{GradoFamiliaConstructibles:def}. Furthermore, let $\scrD (W) := \scrD(\{W\})$, according to Definition \ref{GradoZariskiFamiliaConstructibles:def}. Let us observe that:
$$\scrC (\CW, \CC) = \bigcup_{W \in \CW} \scrC (W, \CC) \qquad \text{and} \qquad \scrD(\CW) = \bigcup_{W \in \CW} \scrD (W).$$
Let us prove that for every $W \in \CW$, 
$$\scrD(W) \subseteq \scrC(W, \CC).$$ 
In order to prove this, observe that $W$ admits a decomposition as $W = V \cap U$, where $V \subseteq \A^n(K)$ is a Zariski closed subset and $U \subseteq \A^n(K)$ is a Zariski open subset. Observe also that $V$ can be chosen such that it admits a decomposition into irreducible algebraic varieties of the following form:
\begin{equation} \label{DescomposicionV:eqn}
V = V_{1} \cup \cdots \cup V_m, 
\end{equation}
in such a way that $V_i \cap U \neq \emptyset$ for all $i$, $1 \leq i \leq m$. Then, we conclude that the following is a decomposition of $\overline{W}^z$ into irreducible components:
\begin{equation} \label{IrreduciblesW:eqn}
\overline{W}^z = V_1 \cup \cdots \cup V_m.
\end{equation}
Next, assume that $\CC(W) = \{A_1, \ldots, A_s \}$, where $A_i$ is a locally closed irreducible subset of $\A^n(K)$. Then, we also have:
\begin{equation} \label{IrreduciblesWMap:eqn}
\overline{W}^z = \overline{A_1}^z \cup \ldots \cup \overline{A_s}^z,
\end{equation}
and $\scrC(W, \CC) = \{ \overline{A_1}^z, \ldots, \overline{A_s} \}$. Hence, from Identities (\ref{IrreduciblesW:eqn}) and (\ref{IrreduciblesWMap:eqn}), as Identity (\ref{IrreduciblesW:eqn}) is a decomposition of $\overline{W}^z$ into irreducible components, we conclude:
$$\{V_1, \ldots, V_m \} \subseteq \{ \overline{A_1}^z, \ldots,\overline{A_s}^z \}.$$
Since $\scrD(W) = \{ V_1, \ldots, V_m \}$, we have proved that for each $W \in \CW$,
$$\scrD(W) \subseteq \scrC(W,\CC).$$
Then, $\scrD(\CW) \subseteq \scrC(\CW, \CC)$ and, hence,
$$\sum_{V \in \scrD(\CW)} \deg(V) \leq \sum_{V \in \scrC(\CW, \CC)} \deg(V) = \deg(\CW, \CC).$$
On the other hand, as $W = V \cap U$, taking into account the decomposition of $V$ described in Identity (\ref{DescomposicionV:eqn}), we have an assignment $\CC_1$ which associates to $W \in \CW$ the minimal $LCI$-degree decomposition:
$$\CC_1 (W) = \{ V_1 \cap U, \ldots, V_m \cap U \}.$$
Then, $\scrC(W, \CC_1) \subseteq \scrD (W)$ for each $W \in \CW$ and, hence,
$$\deg(\CW, \CC_1) = \sum_{V \in \scrC (\CW, \CC_1)} \deg(V) \leq \sum_{V \in \scrD(\CW)} \deg(V).$$
This proves that:
$$\deg(\CW) \leq \sum_{V \in \scrD (\CW)} \deg(V),$$
and the lemma follows.
\end{proof}

\begin{lemma} \label{cotasD:lemma}
Let $C \subseteq \A^n(K)$ be a constructible set and $\CW = \{ W_1, \ldots, W_s \}$ a finite family of locally closed subsets of $\A^n(K)$. We define the following set:
$$\scrT (\CW, C) := \left\{ X \subseteq \A^n(K) \; : \; \exists S \subseteq \{ 1, \ldots, s \}, \; X = C \cap \left( \bigcap_{i \in  S} W_i \right) \right \}.$$
Let $\scrD := \scrD (\scrT (\CW, C))$ be the class of irreducible algebraic varieties associated to $\scrT (\CW, C)$ as in Definition \ref{GradoZariskiFamiliaConstructibles:def}. Then, we have:
$$\sum_{A \in \scrD} \deg(A) \leq \deg_{\rm lci} (C) (1 + \deg(\CW))^{\dim(C)}.$$
\end{lemma}

\begin{proof}
First, we observe that it suffices to prove this lemma in the case where $C$ is locally closed. In order to prove this claim, let us consider a decomposition of the constructible set $C$ into locally closed irreducible subsets that minimizes the $LCI$-degree as in Definition \ref{grado-constructibles:def}, i.e.
\begin{equation*} 
C := C_1 \cup \ldots \cup C_m,
\end{equation*}
such that:
\begin{equation*} 
\deg_{\rm lci} (C) = \sum_{i=1}^m \deg(C_i).
\end{equation*}
With these notations, the following holds:
\begin{claim}
$$\scrD \subseteq \bigcup_{j=1}^m \scrD (\scrT (\CW, C_j)).$$
\end{claim}
\begin{proof-claim}
Let $V \in \scrD$ be an irreducible algebraic variety such that there exists $X \in \scrT (\CW,C)$ satisfying that $V$ is an irreducible component of $\overline{X}^z$. Moreover, since $X \in \scrT(\CW, C)$, there exists $S \subseteq \{ 1, \ldots, s \}$ such that:
$$X = C \cap \left( \bigcap_{i \in S} W_i \right).$$
Thus,
$$X = \bigcup_{j=1}^m \left( C_j \cap \left( \bigcap_{i \in S} W_i \right) \right).$$
Then, taking Zariski closures, we have:
$$\overline{X}^z = \bigcup_{j=1}^m \overline{C_j \cap \left( \bigcap_{i \in S} W_i \right)}^z.$$
Now, recall that any decomposition of an algebraic variety $A$ as finite union of irreducible varieties can be refined to obtain a decomposition of $A$ into its irreducible components. Then, the irreducible components of $\overline{X}^z$ are elements of the set:
$$\bigcup_{j=1}^m \scrD \left( \overline{C_j \cap \left( \bigcap_{i \in S} W_i \right)}^z \right), $$
As $V$ is an irreducible component of $\overline{X}^z$, it follows that $V$ must be an irreducible component of $\overline{C_j \cap \left( \bigcap_{i \in S} W_i \right)}^z$ for some $j \in \{ 1, \ldots, m \}$. Therefore, since:
$$C_j \cap \left( \bigcap_{i \in S} W_i \right) \in \scrT (\CW, C_j),$$
we conclude that $V \in \scrD (\scrT (\CW, C_j))$ 
and the claim follows.
\end{proof-claim}

If the lemma holds in the locally closed case, we would have the following inequality for each $j \in \{ 1, \ldots, m \}$:
$$\sum_{A \in \scrD (\scrT (\CW, C_j))} \deg(A) \leq \deg(C_j) \left(1+ \deg(\CW) \right)^{\dim(C_j)}.$$
Since $\dim(C) = \max \{ \dim(C_j) \; : \; 1 \leq j \leq m \}$ and applying the previous claim, we conclude:
\begin{equation*}
\begin{aligned}
\sum_{A \in \scrD} \deg(A) & \leq \sum_{j=1}^m  \left( \sum_{A \in \scrD (\scrT (\CW, C_j))} \deg (A) \right) \leq & \\
& \leq \sum_{j=1}^m \deg(C_j) \left(1 + \deg(\CW) \right)^{\dim(C_j)} \leq & \\
& \leq \sum_{j=1}^m \deg(C_j) \left(1+ \deg(\CW) \right)^{\dim(C)} = & \\
& =\left( \sum_{j=1}^m \deg(C_j) \right) \left(1+ \deg(\CW) \right)^{\dim(C)} = & \\
& = \deg_{\rm lci} (C) \left(1+ \deg(\CW) \right)^{\dim(C)}, &
\end{aligned}
\end{equation*}
which proves the lemma for any constructible set $C$. \\
From now on, we may assume, as claimed, that $C$ is a locally closed subset of $\A^n(K)$. We will prove the lemma for the case where $C$ is locally closed (in which case $\deg(C) = \deg_z (C) = \deg_{\pi}(C) = \deg_{\rm lci}(C)$). \\
Let $d := \dim(C)$. We define the following classes of irreducible components:

\begin{itemize}
\item $\scrD(d) := \{ V \; : \; {\hbox {\rm $V$ is an irreducible component of $\overline{C}^z$}}\}$.
\item For $0 \leq k \leq d-1$, let us define:
\begin{equation} \label{scrDk:eqn}
\scrD(k) := \left\{ V \in \scrD \; : \; \dim(V)=k, \; V \not\in \scrD(d) \right\}.
\end{equation}
\end{itemize}
We thus have a partition of $\scrD$ given by the following disjoint union:
$$\scrD = \bigcup_{k=0}^d \scrD(k).$$
Next, we prove the following claim:

\begin{claim} 
Let $k \in \N$ be a non--negative integer such that $0 \leq k \leq d-1$. Let $A \in \scrD(k)$ be an element of $\scrD(k)$ according to Identity (\ref{scrDk:eqn}). Then, there exist $A^* \in \scrD$ and $V \in \scrD(\CW)$ such that:
\begin{enumerate}
\item $\dim(A^*) \geq \dim(A) +1 = k+1$.
\item $A$ is an irreducible component of $A^* \cap V$.
\end{enumerate}
\end{claim}

\begin{proof-claim}
As $A \in \scrD(k) \subseteq \scrD$, there exists a subset $S \subseteq \{1, \ldots, s \}$ of minimal cardinality such that the following two properties hold:
\begin{enumerate}
\item The locally closed set $B$ given by:
$$B := C \cap \left( \bigcap_{i \in S} W_i  \right)$$
is in $\scrT(W,C)$,
\item $A$ is an irreducible component of $\overline{B}^z$.
\end{enumerate}

As $A \in \scrD(k)$ and $k<d$, we already know that $A$ is not an irreducible component of $\overline{C}^z$. Thus, $S \neq \emptyset$ (otherwise $B=C$) and we may choose $i \in S$. Define $S' := S \setminus \{ i \}$ and:
$$B' := C \cap \left( \bigcap_{j \in S'} W_j \right) \in \scrT(\CW, C).$$
Let $W:= W_i \in \CW$ and observe that $B = B' \cap W$. Since $A$ is an irreducible component of $\overline{B}^z \subseteq \overline{B'}^z$, there always exists an irreducible component $A^*$ of $\overline{B'}^z$ such that $A \subseteq A^*$. Moreover, due to the minimal cardinality of $S$, $A$ cannot be an irreducible component of $\overline{B'}^z$. In particular, for every irreducible component $A^*$ of $\overline{B'}^z$ such that $A \subseteq A^*$, we would have $A \subsetneq A^*$. As both of them are irreducible, we also conclude that for every irreducible component $A^*$ of $\overline{B'}^z$ such that $A \subseteq A^*$, we have $\dim(A^*) \geq \dim(A)+1$. \\
Next, recall that we are assuming $C$ is locally closed and that the elements of the family $\CW$ are locally closed. Therefore, $B'$ is also a locally closed subset of $\A^n (K)$. Thus, there is a decomposition of $B'$ as finite union of non--empty locally closed irreducible sets:
$$B' = \left( A_1^* \cap B_1 \right) \cup \cdots \cup \left( A_t^* \cap B_t \right),$$ 
where:
\begin{itemize}
\item $A_j^*$ is a closed irreducible subset of $\A^n(K)$,
\item $B_j$ is an open subset of $\A^n(K)$ such that $A_j^* \cap B_j \neq \emptyset$, $\forall j \in \{ 1, \ldots, t \}$,
\item $\overline{B'}^z = A_1^* \cup \cdots \cup A_t^*$ is the decomposition of $\overline{B'}^z$ into irreducible components. 
\end{itemize}

Note that $W$ is also locally closed and, therefore, admits a decomposition into locally closed irreducible sets:
$$W = (V_1 \cap U_1) \cup \cdots \cup (V_r \cap U_r )$$
such that:
\begin{itemize}
\item $V_{\ell}$ is an irreducible variety of $\A^n(K)$,
\item $U_{\ell}$ is a Zariski open subset of $\A^n(K)$ satisfying $V_{\ell} \cap U_{\ell} \neq \emptyset$, $\forall \ell \in \{1, \ldots, r \}$,
\item $\overline{W}^z = V_1 \cup \cdots \cup V_r$ is the decomposition of $\overline{W}^z$ into its irreducible components.
\end{itemize}

Observe that, in particular, we have that $\{ V_1, \ldots, V_r \} \subseteq \scrD(\CW)$. \\
Thus, we conclude:
$$B = B' \cap W = \bigcup_{j=1}^t \bigcup_{\ell = 1}^{r} \left( A_j^* \cap V_{\ell} \right) \cap \left( B_j \cap U_{\ell} \right).$$
Hence:
$$\overline{B}^z = \overline{B' \cap W}^z = \bigcup_{j=1}^t \bigcup_{\ell=1}^r \overline{\left( A_j^* \cap V_{\ell} \right) \cap \left( B_j \cap U_{\ell} \right)}^z.$$
As the decomposition of the algebraic variety $\overline{B}^z$ as finite union of irreducible sets admits a refinement to obtain the decomposition of $\overline{B}^z$ into its irreducible components, we conclude that there exist $j \in \{1, \ldots, t \}$ and $\ell \in \{1, \ldots, r \}$ such that $A$ is an irreducible component of $\overline{\left( A_j^* \cap V_{\ell} \right) \cap \left( B_j \cap U_{\ell} \right)}^z$. \\
Next, observe that $A_j^* \cap V_{\ell} \subseteq \A^n(K)$ is an algebraic variety and $B_j \cap U_{\ell} \subseteq \A^n(K)$ is a Zariski open subset. Assume that the following is a decomposition of $A_j^* \cap V_{\ell}$ into its irreducible components:
$$A_j^* \cap V_{\ell} = Z_1 \cup \cdots \cup Z_p.$$
Then, for all $q$, $1 \leq q \leq p$, we have that $Z_q \cap (B_j \cap U_{\ell})$ is either empty or dense in $Z_q$. Up to some reordering of the indices, assume that $u$, with $1<u \leq p$, is such that:
$$Z_q \cap \left( B_j \cap U_{\ell} \right) \neq \emptyset \; \text{ if and only if} \; 1 \leq q \leq u.$$
Then,
$$\overline{\left( A_j^* \cap V_{\ell} \right) \cap \left( B_j \cap U_{\ell} \right)}^z = Z_1 \cup \cdots \cup Z_u,$$
and, hence, $A \in \{ Z_1, \ldots, Z_u \}$. The claim follows by choosing $A^* = A_j^*$ and $V = V_{\ell}$.
\end{proof-claim}

Finally, using the previous claim, we are in a position to prove the lemma. First of all, for every $k<d$, we define the following mapping:
$$\begin{matrix}
\Phi_k: &\scrD(k)& \longrightarrow & \left(\bigcup_{r=k+1}^d \scrD(r)\right)\times \scrD(\CW) \\
& A & \longmapsto & (A^{*}, V)
\end{matrix},$$
given by the following rule: \\ For every $A \in \scrD(k)$, $k<d$, we associate a pair $(A^*, V)$ such that $A$ is an irreducible component of $A^* \cap V$. \smallskip

Let us define the following quantity:
$$\CD(k) := \sum_{r=k}^d \left( \sum_{A \in \scrD(r)} \deg(A) \right).$$
Then, we prove by induction on $m=d-k$ that the following inequality holds:
\begin{equation} \label{DesigualdadInduccionCotasD}
\CD(k) = \CD(d-m) \leq \deg(C) (1+R)^m,
\end{equation}
where $R:= \deg(\CW)$. \\
Note that:
\begin{equation} \label{IgualdadInduccionCotasD}
\CD(k) = \CD(k+1) + \sum_{A \in \scrD(k)} \deg(A).
\end{equation}
The case $m=0$ is obvious since $\scrD(d)$ is the class of all irreducible components of the locally closed set $C$. Hence,
$$\sum_{A \in \scrD(d)} \deg(A) = \deg_z (C) = \deg(C).$$
Assume now that Inequality (\ref{DesigualdadInduccionCotasD}) holds for $m-1$ and let us prove it
 for $m=d-k \geq 1$. Observe that the following inequality holds:
$$\sum_{A \in \scrD(k)} \deg(A) \leq \sum_{(A^*, V) \in \left( \bigcup_{r=k+1}^d \scrD(r) \times \scrD (\CW) \right)} \sum_{A \in \Phi_k^{-1} (A^*, V)} \deg(A).$$

Now, from the previous claim, $A$ is an irreducible component of $A^* \cap V$ for every  $A \in \Phi_k^{-1} (A^*, V)$. Thus, by the Bézout's Inequality of \cite{Heintz83} (see Theorem \ref{Bezout-localmente-cerrados:teor}) and by the definition of $\Phi_k$, we obtain:
$$\sum_{A \in \Phi_k^{-1} (A^*, V) } \deg(A) \leq \deg(A^* \cap V) \leq \deg(A^*) \cdot \deg(V).$$
Therefore, we have:
$$\sum_{A \in \scrD(k)} \deg(A) \leq \sum_{(A^*, V) \in \left( \bigcup_{r=k+1}^d \scrD(r) \times \scrD (\CW) \right)} \deg(A^*) \cdot \deg(V).$$
Hence,
\begin{equation} \label{PasoFinalCotasD:eqn}
\sum_{A \in \scrD(k)} \deg(A) \leq \left(\sum_{A^* \in \left( \bigcup_{r=k+1}^d \scrD(r) \right)} \deg(A^*) \right) \cdot \left( \sum_{V \in \scrD(\CW)} \deg(V) \right).
\end{equation}
Now, applying Lemma \ref{GradoWSumaComponentesIrreducibles:lemma}, we obtain:
$$\sum_{V \in \scrD(\CW)} \deg(V) = \deg(\CW) = R.$$
Observe that:
 $$\sum_{A^* \in \left( \bigcup_{r=k+1}^d \scrD(r) \right)} \deg(A^*) = \sum_{r=k+1}^d \left( \sum_{A^* \in \scrD(r)} \deg(A^*) \right) = \CD (k+1).$$
Thus, Inequality (\ref{PasoFinalCotasD:eqn}) becomes:
$$\sum_{A \in \scrD(k)} \deg(A) \leq \CD(k+1) \cdot R.$$
In conclusion, combining Identity (\ref{IgualdadInduccionCotasD}) with the previous inequality, we obtain:
$$\CD(k) = \CD(k+1) + \sum_{A \in \scrD(k)} \deg(A) \leq \CD(k+1) +  \CD(k+1) \cdot R = \CD(k+1) \cdot (1+R)$$
Applying the induction hypothesis for $m-1=d-(k+1)$, we conclude:
$$\CD(k)  \leq  \deg(C) (1+R)^{m-1} \cdot (1+R) = \deg(C) (1+R)^m.$$
Hence, the lemma follows when $C$ is locally closed. From our discussion at the beginning of this proof, the lemma also follows when $C \subseteq \A^n(K)$ is any constructible set.
\end{proof}

\begin{corollary} \label{CotasDSuma:corol}
With the same notations and assumptions as in the previous lemma, we have:
$$\sum_{A \in \scrD(\CW,C)} \deg(A) \leq \deg_{\rm lci} (C) \left( 1+ \sum_{i=1}^s \deg(W_i) \right)^{\dim(C)}.$$
\end{corollary}

\begin{proof}
Observe simply that:
$$\deg(\CW) \leq \sum_{r=1}^s \deg(W_i).$$
\end{proof}

\subsection{Combinatorial bounds}

\begin{definition}Let $\CW = \{ W_1, \ldots, W_s \}$ be a finite family of locally closed subsets of $\A^n (K)$. We define the quantity $\widetilde{\deg} (\CW)$ as the minimum of the following set:

\begin{equation*}
\begin{split}
\{ \deg(\widetilde{\CW}) \: : \;  \widetilde{\CW} = \{A_1, \ldots, A_s, B_1, \ldots, B_s \},  \; A_i &, B_i \subseteq \A^n(K) \; {\hbox {\rm Zariski closed}}, \; 1 \leq i \leq s, \\
 & {\hbox {\rm and }} W_i = A_i \cap (\A^n(K) \setminus B_i), \; 1 \leq i \leq s \}. 
\end{split}
\end{equation*}
\end{definition}

Note that if $\CW = \{ W_1, \ldots, W_s \}$ consists only of algebraic varieties, taking $A_i = W_i$ and $B_i = \emptyset$ for each $i$, $1 \leq i \leq s$, we conclude $\widetilde{\deg} (\CW) = \deg(\CW)$.

\begin{proposition} \label{CotaCeldasNoVaciasDegW:prop}
Let $C \subseteq \A^n(K)$ be a constructible subset. Let $\CW = \{ W_1, \ldots, W_s \}$ be a finite family of locally closed subsets of $\A^n(K)$. Let $\CH$ be a finite family of constructible subsets of $\A^n(K)$. Assume that:
$$\restr{\CH}{C} := \{ H \cap C \; : \; H \in \CH \} \subseteq \scrB(\CW, C).$$
Then, we have:
$$\sharp(\scrZ(\CH, C)) \leq \deg_{\rm lci} (C) (1+ \widetilde{\deg}(\CW))^{\dim(C)}.$$
In the case where $\CW$ consists only of Zariski closed sets, we have:
$$\sharp(\scrZ(\CH, C)) \leq \deg_{\rm lci} (C) (1+ \deg(\CW))^{\dim(C)}.$$

\end{proposition}

\begin{proof}
First of all, observe that $\scrZ(H,C)$ depends only of its intersections with $C$. Thus, $$\scrZ (\restr{\CH}{C},C) = \scrZ(\CH,C),$$ and we may assume that the constructible subsets of $\CH$ are also constructible subsets of $C$. \\
We start by proving the following claim, which relates subsets of indices to cells:

\begin{claim} \label{BiyeccionIndiceCelda:claim}
Let $C \subseteq \A^n(K)$ be a constructible set and $\CH = \{ H_1, \ldots, H_t \}$ a finite family of constructible sets. Then, the following mapping is a bijection:

\begin{equation} \label{SubconjuntoCelda:eqn}
\begin{matrix}
& \{ T \subseteq [t] \; : \; Y_T \neq \emptyset \} & \longrightarrow & \scrZ(\CH,C) \\
& T & \longmapsto & Y_T
\end{matrix},
\end{equation}
where:
$$Y_T := C \cap \left( \bigcap_{k \in T} H_k \right) \cap \left( \bigcap_{\ell \in [t] \setminus T} \left( \A^n(K) \setminus H_{\ell} \right) \right).$$
\end{claim}

\begin{proof-claim}
Just observe that if $T$ and $T'$ are two distinct subsets of $[t]$, with $Y_T \neq \emptyset$ and $Y_{T'} \neq \emptyset$, then $Y_T \cap Y_{T'} = \emptyset$ because both are $\CH-$cells in $C$. Hence, if $T \neq T'$, then $Y_T \neq Y_{T'}$, and the mapping in Identity (\ref{SubconjuntoCelda:eqn}) is injective. Since this mapping is obviously surjective, it is a bijection.
\end{proof-claim}

Next, we bound the number of $\CH$-cells in $C$ by the number of $\widetilde{\CW}$-cells in $C$ for any class of constructible sets $\widetilde{\CW}$ such that $\CH \subseteq \scrB(\widetilde{W}, C)$:

\begin{claim} \label{AcotacionCeldasBooleana:claim}
With the same assumptions and notations as in the statement of the proposition, let $\widetilde{\CW} = \{V_1, \ldots, V_r \}$ be any class of constructible sets such that $\CH \subseteq \scrB(\widetilde{W}, C)$. Then, we have:
\begin{equation}  \label{AcotacionCeldasBooleana:eqn}
\sharp(\scrZ(\CH,C)) \leq \sharp(\scrZ(\widetilde{\CW},C)).
\end{equation}
\end{claim}

\begin{proof-claim}
Assume that $\CH = \{H_1, \ldots, H_t \}$. For every $T \subseteq [t]$, let us consider the $\CH$-cell in $C$ associated with $T$:
$$Y_T := C \cap \left( \bigcap_{k \in T} H_k \right) \cap \left( \bigcap_{\ell \in [t] \setminus T} \left( \A^n(K) \setminus H_{\ell} \right) \right).$$
As $\CH \subseteq \scrB(\widetilde{\CW}, C)$, for every $T \subseteq [t]$ such that $Y_T \neq \emptyset$, there exists $L(T) \subseteq 2^{[r]}$ such that:
$$Y_T = \bigcup_{J \in L(t)} C \cap \left( \bigcap_{i \in J} V_i \right) \cap \left( \bigcap_{j \in [r] \setminus J} \left( \A^{n} (K) \setminus V_j \right) \right).$$
For every $J \subseteq [r]$, let us denote by $X_J \in \scrZ(\widetilde{W},C)$ the constructible set:
$$X_J := C \cap \left( \bigcap_{i \in J} V_i \right) \cap \left( \bigcap_{j \in [r] \setminus J} \left( \A^n (K) \setminus V_j \right) \right). $$
For every $T \subseteq [t]$ such that $Y_T \neq \emptyset$, there exists at least one $J(T) \in L(T)$ such that $X_{J(T)} \neq \emptyset$. This allows us to build up the following mapping:
\begin{equation} 
\begin{matrix}
& \{ T \subseteq [t] \; : \; Y_T \neq \emptyset \} & \longrightarrow & \{ J \subseteq [r] \; : \; X_{J} \neq \emptyset \}  \\
& T & \longmapsto & J(T)
\end{matrix}.
\end{equation}
Next, we prove that it is an injective mapping. For if $T$ and $T'$ are two distinct subsets of $[t]$ such that $Y_{T} \neq \emptyset$, $Y_{T'} \neq \emptyset$ and $J(T) = J(T')$, we would have:
$$X_{J(T)} = X_{J(T')} \subseteq Y_T \cap Y_{T'} = \emptyset,$$
which contradicts that $X_{J(T)} = X_{J(T')} \neq \emptyset$. Hence, we conclude Inequality (\ref{AcotacionCeldasBooleana:eqn}) by simply applying Claim \ref{BiyeccionIndiceCelda:claim}.

\end{proof-claim}

Now, we will prove the proposition for the case in which $C$ is locally closed and $\CW$ is a finite family of algebraic varieties: 

\begin{claim} \label{CeldasLocalmenteCerradoVariedades:claim}
Let $C \subseteq \A^n (K)$ be a locally closed set, $\CW = \{ W_1, \ldots, W_s \}$ a finite family of algebraic varieties of $\A^n(K)$ and $\CH$ a finite family of constructible subsets of $\A^n(K)$. Assume that $\CH \subseteq \scrB(\CW, C)$. Then, we have:
$$\sharp(\scrZ(\CH, C))\leq \deg(C)(1+ \deg (\CW) )^{\dim(C)},$$
where $\deg(C) = \deg_{\pi}(C) = \deg_{\rm lci} (C)$ and $\deg(\CW) = \widetilde{\deg}(\CW)$.
\end{claim}
\begin{proof-claim}
We consider the following classes of sets:
$$\scrD(\scrZ(\CW, C)) :=\{ V \subseteq \A^n (K) \; :\; \exists X\in \scrZ(\CW, C), \;
{\hbox{\rm $V$ is an irreducible component of $\overline{X}^z$}}\},$$
$$\scrT(\CW, C):= \left\{ X\subseteq C\; :\; \exists S\subseteq \CW, \; X= C \cap\left(\bigcap_{W\in S} W\right) \right\}$$
and
$$\scrD(\scrT(\CW, C)):=\{V\subseteq \A^n(K) \; :\; \exists X\in \scrT(\CW, C),\;{\hbox {\rm $V$ is an irreducible component of $\overline{X}^z$}}\}.$$

The following claim proves that $\scrD(\scrZ(\CW, C))$ is a subset of $\scrD(\scrT(\CW, C))$:

\begin{claimNested}
With the above notations, we have:
\begin{equation}\label{entre-celdas-interseccion:eqn}
\scrD(\scrZ(\CW, C)) \subseteq \scrD(\scrT(\CW, C)).
\end{equation}
\end{claimNested}
\begin{proof-claim}
 Let $V\in \scrD(\scrZ(\CW, C))$ be an irreducible component of  $\overline{X}^z$, where $X\in \scrZ(\CW, C)$. We want to prove that $V$ is an irreducible component of some $\overline{Y}^z$ with $Y \in \scrT(\CW, C)$. Since $X\in \scrZ(\CW, C)$, there exists  some subset $S\subseteq [s]$ such that:
$$X:=C \cap \left( \bigcap_{i\in S} W_i\right)\cap \left( \bigcap_{j\not\in S} (C \setminus W_j) \right)\not=\emptyset.$$
Next, we consider the following locally closed set and its decomposition into locally closed irreducible components:
$$Y:=C \cap \left( \bigcap_{i\in S} W_i\right):=A_1\cup \cdots \cup A_k,$$
where $A_i = U_i \cap V_i$ is locally closed irreducible, with $U_i$ open and $V_i$ closed irreducible with respect to the Zariski topology. Observe that $Y\in \scrT(\CW, C)$. As every non-empty Zariski open subset of an irreducible algebraic variety is dense in the variety, we have that $\overline{A_i}^z = V_i$ for $1 \leq i \leq k$. Therefore, we have the following decomposition of $\overline{Y}^z$ into irreducible components:
$$\overline{Y}^z = V_1 \cup \cdots \cup V_k.$$

Note that we can rewrite the set $X$ using the sets we have just defined as follows:
$$X = Y \cap \left( \bigcap_{j \not\in S} (C \setminus W_j) \right) = (A_1 \cup \cdots \cup A_k) \cap \left( \bigcap_{j \not\in S} (C \setminus W_j) \right) =$$ 
$$= \bigcup_{i=1}^{k} \left( A_i \cap \left( \bigcap_{j \not\in S} (C \setminus W_j) \right) \right) = \bigcup_{i=1}^k B_i,  $$
where :
$$B_i:= A_i \cap \left( \bigcap_{j \not\in S} (C \setminus W_j) \right).$$
We have just two cases for every $i$, $1\leq i \leq k$:
\begin{itemize}
\item Case $1$:
$$B_i:=A_i\cap \left( \bigcap_{j\not\in S} (C \setminus W_j) \right)\neq \emptyset,$$
then, $B_i$ is an open subset in $\overline{A_i}^z$ for the Zariski topology.
\item Case $2$:
$$B_i:=A_i\cap \left( \bigcap_{j\not\in S} (C \setminus W_j) \right) = \emptyset.$$
\end{itemize}
Up to some reordering of the indices, assume that $r\leq k$ is such that $B_i\not=\emptyset$ if and only if $1\leq i \leq r$. Then, we have:
$$X = \bigcup_{i=1}^k B_i = \bigcup_{i=1}^r B_i.$$
Since the closure of a finite union is the union of the closures, we conclude:
$$\overline{X}^z = \bigcup_{i=1}^{r} \overline{B_i}^z.$$
As every non-empty Zariski open subset of an irreducible algebraic variety is dense in the variety, we get: 
$$\overline{B_i}^z = \overline{A_i\cap \left( \bigcap_{j\not\in S} (C \setminus W_j) \right)}^z = \overline{A_i}^z = V_i.$$
Thus, we have the following decomposition of $\overline{X}^z$ into irreducible components:
$$\overline{X}^z = \bigcup_{i=1}^{r} \overline{B_i}^z = \bigcup_{i=1}^{r} V_i.$$
Hence, if $V_{\ell}$ is an irreducible component of $\overline{X}^z$ for some $\ell \in [r]$, then $V_{\ell}$ is an irreducible component of $\overline{Y}^z$. Thus, taking $V = V_{\ell}$, we conclude $V\in \scrD(\scrT(\CW, C))$ and the inclusion at Identity (\ref{entre-celdas-interseccion:eqn}) holds.
\end{proof-claim}

Next, we prove that the cardinality of $\scrZ(\CW, C)$ is less or equal than the cardinality of $\scrD(\scrZ(\CW, C))$:

\begin{claimNested}
With the above notations, we have:
\begin{equation}\label{menos-celdas-que-componentes:eqn}
\sharp\left(\scrZ(\CW, C)\right) \leq \sharp \left(\scrD(\scrZ(\CW, C)) \right).
\end{equation}
\end{claimNested}

\begin{proof-claim}
Recall that $\scrZ(\CW, C)$ is a partition of $C$ into locally closed subsets. We are going to prove that for every non--empty  $X\in \scrZ(\CW, C)$, there exists an irreducible component $V_X$ of $\overline{X}^z$ such that $V_X$ is not an irreducible component of any other $Y\in \scrZ(\CW, C)$ with $Y\not= X$. \\
 We proceed by contradiction. Let $X, Y \in \scrZ(\CW, C) $, with $Y \neq X$. Assume that there exists an irreducible component $T$ of $\overline{X}^z$ that is also an irreducible component of $\overline{Y}^z$. Since $T$ is an irreducible component of $\overline{X}^z$, there must be some open Zariski subset $U\subseteq \A^n (K)$ such that $A = T \cap U$ is a locally closed irreducible component of $X$. On the other hand, there exists another open Zariski subset  $U' \subseteq \A^n (K)$ such that $A' = T \cap U'$ is a locally closed irreducible component of $Y$. As $T$ is irreducible, the intersection of two non-empty Zariski open sets in $T$ is non-empty and we have:
$$\emptyset\not= U\cap U' \cap T \subseteq X \cap Y =\emptyset,$$
which contradicts $X\cap Y=\emptyset$. Thus, there does not exist an irreducible component of $\overline{X}^z$, with $X \in \scrZ(\CW, C)$, that is also an irreducible component of any other $\overline{Y}^z$, with $Y \in \scrZ(\CW, C)$ and $Y \neq X$. This proves Inequality (\ref{menos-celdas-que-componentes:eqn}).
\end{proof-claim}

Combining Inequality (\ref{entre-celdas-interseccion:eqn}) and Inclusion (\ref{menos-celdas-que-componentes:eqn}), we obtain:

\begin{equation}\label{de-celdas-a-erzegrad1:eqn}
\begin{split}
\sharp\left(\scrZ(\CW, C)\right) & \leq \sharp\left(\scrD(\scrZ(\CW, C)) \right) 
  \leq \sharp \left(  \scrD(\scrT(\CW, C)) \right) \leq \\
 & \leq \deg\left( \scrD(\scrT(\CW, C)) \right) = \sum_{A \in \scrD(\scrT(\CW, C))} \deg(A).
  \end{split}
\end{equation}

Applying Lemma \ref{cotasD:lemma}, we conclude:
\begin{equation}\label{de-celdas-a-erzegrad2:eqn}
 \sum_{A \in \scrD(\scrT(\CW, C))} \deg(A) \leq \deg(C) \left( 1 + \deg(\CW)\right)^{\dim(C)}.
\end{equation}
As $\CH \subseteq \scrB(\CW, C)$, Claim \ref{AcotacionCeldasBooleana:claim} implies:
\begin{equation} \label{RefinamientoDivisionCeldas:eqn}
\sharp( \scrZ(\CH, C)) \leq \sharp(\scrZ(\CW, C)).
\end{equation}
Combining Inequalities (\ref{de-celdas-a-erzegrad1:eqn}), (\ref{de-celdas-a-erzegrad2:eqn}) and (\ref{RefinamientoDivisionCeldas:eqn}), we finally conclude:
$$\sharp(\scrZ(\CH, C))\leq \deg(C)(1+ \deg(\CW)))^{\dim(C)}.$$
\end{proof-claim}

Now, we bound the number of $\widetilde{\CW}$-cells in a constructible set $C$ by the sum of the $\widetilde{\CW}$-cells of the locally closed irreducible sets in a decomposition of $C$, where $\widetilde{\CW}$ is a finite family of constructible sets.

\begin{claim} \label{CeldasDescomposicionLCI:claim}
Let $C \subseteq \A^n(K)$ be a constructible set and $\widetilde{\CW} = \{ V_1, \ldots, V_t \}$ a finite family of constructible sets in $\A^n(K)$. Assume that the following is a decomposition of $C$ into locally closed irreducible subsets satisfying the claims of Lemma \ref{descomposicion-union-irreducibles-distintos:lemma}:
$$C = C_1 \cup \cdots \cup C_m.$$
Then, we have:
$$\sharp(\scrZ(\widetilde{\CW}, C)) \leq \sum_{i=1}^m \sharp(\scrZ (\widetilde{W}, C_i)).$$
\end{claim}

\begin{proof-claim}
Let $S \subseteq [t]$ be a subset of indices. The $\widetilde{\CW}$-cell in $C$ associated with the subset $S$ is given by the following equality:
$$X_S := C \cap \left( \bigcap_{i \in S} V_i \right) \cap \left( \bigcap_{j \in [t] \setminus S} \left( \A^n(K) \setminus V_j \right) \right) .$$
For every $i$, $1 \leq i \leq m$, let us define the constructible subset $X_S^{(i)} \subseteq C_i$ by the following equality:
$$X_S^{(i)} := C_i \cap \left( \bigcap_{i \in S} V_i \right) \cap \left( \bigcap_{j \in [t] \setminus S} \left( \A^{n}(K) \setminus V_j \right) \right) \subseteq C_i.$$
Obviously, $X_S^{(i)}$ is an $\widetilde{W}-$cell in $C_i$ and $$X_S = \bigcup_{i=1}^m X_S^{(i)}.$$ Thus, for every $S \subseteq [t]$ such that $X_S \neq \emptyset$, we may choose some $i(S) \in [m]$ such that $X_S^{(i(S))} \neq \emptyset$. This defines the following mapping:

\begin{equation} 
\begin{matrix}
& \{ S \subseteq [t] \; : \; X_S \neq \emptyset \} & \longrightarrow & \bigcup_{i=1}^m \scrZ(\widetilde{\CW}, C_i) \\
& S & \longmapsto & X_S^{(i(S))}
\end{matrix},
\end{equation}

Next, we prove that it is an injective mapping. For if $S$ and $S'$ are two distinct subsets of $[t]$ such that $X_S \neq \emptyset$, $X_{S'} \neq \emptyset$ and $X_S^{(i(S))} = X_{S'}^{(i(S'))}$, we would have:
$$X_{S}^{i(S)} = X_{S'}^{(i(S'))} \subseteq X_S \cap X_{S'} = \emptyset,$$
which contradicts that $X_{S}^{i(S)} = X_{S'}^{(i(S'))}  \neq \emptyset$. From Claim \ref{BiyeccionIndiceCelda:claim}, we conclude:
$$\sharp(\scrZ(\widetilde{\CW}, C) ) \leq \sharp \left( \bigcup_{i=1}^m \scrZ(\widetilde{W},C_i) \right) \leq \sum_{i=1}^m \sharp(\scrZ (\widetilde{W},C_i)),$$
and the claim follows.
\end{proof-claim}

We are finally in conditions to prove the statement of our proposition. By Claim \ref{AcotacionCeldasBooleana:claim}, since $\CH \subseteq \scrB(\CW, C)$, we have:
$$\sharp(\scrZ (\CH,C)) \leq \sharp(\scrZ(\CW, C)).$$
Let $\widetilde{\CW} = \{A_1, \ldots, A_s, B_1, \ldots, B_s \}$ be a finite family of algebraic varieties such that:
\begin{enumerate}
\item $W_i = A_i \cap (\A^n(K) \setminus B_i)$, $1 \leq i \leq s$,
\item $\widetilde{\deg}(\CW) = \deg(\widetilde{\CW})$.
\end{enumerate}
Also, since $\CW \subseteq \scrB(\widetilde{\CW},C)$, Claim \ref{AcotacionCeldasBooleana:claim} yields:
$$\sharp(\scrZ(\CH, C)) \leq  \sharp(\scrZ(\CW, C))  \leq  \sharp(\scrZ(\widetilde{W}, C)).$$
Now, we consider a decomposition of $C$ into locally closed irreducible subsets that minimizes its $LCI$-degree as in Definition \ref{grado-constructibles:def}, i.e.
$$C = C_1 \cup \cdots C_m,$$
such that:
$$\deg_{\rm lci} (C) = \sum_{i=1}^m \deg(C_i).$$
Because of Claim \ref{CeldasDescomposicionLCI:claim}, we obtain:
\begin{equation} \label{Celdas-DescomposicionLCI:eqn}
\sharp(\scrZ(\CH, C)) \leq \sharp(\scrZ (\widetilde{\CW}, C)) \leq \sum_{i=1}^m \sharp(\scrZ( \widetilde{\CW}, C_i)).
\end{equation}

Next, by Claim \ref{CeldasLocalmenteCerradoVariedades:claim}, as $C_i$ is locally closed and the elements of $\widetilde{\CW}$ are Zariski closed subsets of $\A^n(K)$, we conclude:
\begin{equation} \label{Celdas-Localmnete cerrado-LCI:eqn}
\sharp(\scrZ (\widetilde{\CW}, C_i)) \leq \deg(C_i) (1+ \deg(\widetilde{\CW}))^{\dim(C_i)} \leq \deg(C_i) (1+ \deg(\widetilde{\CW}))^{\dim(C)}.
\end{equation}

Finally, combining Inequalities (\ref{Celdas-DescomposicionLCI:eqn}) and (\ref{Celdas-Localmnete cerrado-LCI:eqn}), we obtain the inequality of the proposition:
$$\sharp(\scrZ(\CH,C)) \leq \left( \sum_{i=1}^m \deg(C_i) \right) (1+ \deg(\widetilde{\CW}))^{\dim(C)} = \deg_{\rm lci}(C) (1+ \deg(\widetilde{\CW}))^{\dim(C)} = $$
$$= \deg_{\rm lci}(C) (1+ \widetilde{\deg}(\CW))^{\dim(C)}. $$

\end{proof}

The following statement generalizes Theorem $2$ in \cite{Heintz83}:

\begin{theorem}[{\bf Erzeugungsgrad and combinatorial bounds}]\label{cotas-Erzeugunsgrad:teor} Let $C\subseteq \A^n (K)$ be a constructible set and  $\CH$ a finite family of constructible subsets of $\A^n (K)$. Then, we have:
\begin{enumerate}
\item $\sharp(\scrZ(\CH, C))\leq \deg_{\rm lci}(C)(1+ \grad(\CH)))^{\dim(C)}$.
\item $\deg_z(\scrB(\CH, C))\leq \deg_{\rm lci}(C)(1+ \grad(\CH))^{\dim(C)}$.
\item Any $\CH-$definable finite subset of $C$ contains at most $\deg_{\rm lci}(C)(1+ \grad(\CH))^{\dim(C)}$ points.
\item $\sharp(\scrB(\CH, C))\leq 2^{\deg_{\rm lci}(C)(1+ \grad(\CH)) ^{\dim(C)}}$.
\end{enumerate}
\end{theorem}

\begin{proof}
Let $\CW=\{W_1,\ldots, W_s\}$ be a finite family of algebraic varieties such that $\CH \subseteq \scrB(\CW)$ and $\grad(\CH) = \deg(\CW)$. Then, we have that:
$$\restr{\CH}{C} := \{ H \cap C \; : \: H \in \CH \} \subseteq \scrB(\CW,C).$$
Hence, from Proposition \ref{CotaCeldasNoVaciasDegW:prop}, we conclude:
$$\sharp\left(\scrZ(\CH, C)\right) \leq  \deg_{\rm lci} (C) (1 + \deg(\CW))^{\dim(C)} = \deg_{\rm lci} (C) (1 + \grad (\CH))^{\dim(C)}$$
and Claim $i)$ follows. \\
As for Claim $ii)$, assume that $C$ is locally closed. Recall that we have:
$$\scrD(\scrB(\CH,C)) = \{ V \subseteq \A^n(K) \; : \; \exists X \in \scrB(\CH, C), \; {\hbox {\rm $V$ is an irreducible component of $\overline{X}^z$}} \},$$
and
$$\deg_{z} (\scrB(\CH, C)) = \sum_{V \in \scrD(\scrB(\CH,C))} \deg(V).$$
Since every element $X \in \scrB(\CH,C)$ is a finite union of $\CH-$cells in $C$, the irreducible components of $\overline{X}^z$ are among the irreducible components of $\overline{Y}^z$, where $Y$ is an $\CH-$cell in $C$. Namely, we have $\scrD(\scrB(\CH,C)) \subseteq \scrD(\scrZ(\CH,C))$ and we conclude:
\begin{equation} \label{CotaGradoZBooleano:eqn}
\deg_z (\scrB(\CH,C)) \leq \sum_{V \in \scrD(\scrZ(\CH,C))} \deg(V).
\end{equation}

As in the proof of Claim $i)$, let $\CW = \{W_1, \ldots, W_s \}$ be a finite family of algebraic varieties in $\A^n(K)$ such that $\grad(\CH) = \deg(\CW)$ and $\CH \subseteq \scrB(\CW, C)$. As in Lemma \ref{cotasD:lemma}, let us define the following class of locally closed subsets:
$$\scrT (\CW, C) = \left\{ X \subseteq \A^n(K) \; : \; \exists S \subseteq \{ 1, \ldots, s \}, \; X = C \cap \left( \bigcap_{i \in  S} W_i \right) \right \}$$
and the following class of irreducible algebraic varieties:
$$\scrD(\scrT (\CW, C)) = \{ V \subseteq \A^n(K) \; : \; \exists X \in \scrT (\CW, C), \; {\hbox {\rm $V$ is an irreducible component of $\overline{X}^z$}} \}.$$
We observe that the following inclusion holds:
\begin{equation} \label{InterseccionesHyW:eqn}
\scrD(\scrT(\CH,C)) \subseteq \scrD(\scrT(\CW,C)).
\end{equation}
In order to prove this inclusion, as $C$ is locally closed, assume that $C = W_0 \cap U$, where $W_0 \subseteq \A^n(K)$ is an algebraic variety and $U \subseteq \A^n(K)$ is a Zariski open subset. Then, let $V \in \scrD(\scrT(\CH,C))$ be an irreducible component of $\overline{X}^z$, where there exists $S \subseteq [s]$ such that:
$$X = C \cap \left( \bigcap_{i \in S} W_i \right) \cap \left( \bigcap_{j \not\in S} \left( \A^n(K) \setminus W_j \right) \right).$$
Let us denote by $\widetilde{U} \subseteq \A^n(K)$ the  Zariski open subset given by:
$$\widetilde{U} := U \cap \left( \bigcap_{j \not\in S} \left( \A^n(K) \setminus W_j \right) \right).$$
Next, consider the algebraic variety $\widetilde{W} \subseteq \A^n(K)$ given by:
$$\widetilde{W} := W_0 \cap \left( \bigcap_{i \in S} W_i \right).$$
Let us consider a decomposition of $\widetilde{W}$ into irreducible components as follows:
\begin{equation} \label{WTildeIrreducibles:eqn}
\widetilde{W} = V_1 \cup \cdots \cup V_m.
\end{equation}
Up to some reordering of the indices, there exists $r$, $1 \leq r \leq m$, such that $\widetilde{U} \cap V_k \neq \emptyset$ if and only if $1 \leq k \leq r$. Then, we have:
$$X = \widetilde{W} \cap \widetilde{U} = \bigcup_{k=1}^r \left( V_k \cap \widetilde{U} \right)$$
and, hence,
$$\overline{X}^z = \bigcup_{k=1}^r V_k.$$
Then, $V \in \{ V_1, \ldots, V_r \}$. \\
On the other hand, let us consider the following locally closed subset $Y \in \scrT(\CW,C)$:
$$Y = C \cap \left( \bigcap_{i \in S} W_i \right) = \widetilde{W} \cap U.$$
Now, consider the decomposition of $\widetilde{W}$ into irreducible components described in Identity (\ref{WTildeIrreducibles:eqn}). Observe that for every $k$, $1 \leq k \leq r$, if $\widetilde{U} \cap V_k \neq \emptyset$, then $U \cap V_k \neq \emptyset$ (because $\widetilde{U} \subseteq U$). Then, up to some reordering of the indices, there is some $t$ such that $1 \leq r \leq t \leq m$ and $V_k \cap U \neq \emptyset$ if and only if $1 \leq k \leq t$. Hence, we conclude that:
$$Y = \bigcup_{k=1}^t (V_k \cap U), \qquad \overline{Y}^z = \bigcup_{k=1}^t V_k.$$
In particular, $V \in \{ V_1, \ldots, V_r \} \subseteq \{ V_1, \ldots, V_t \}$. Moreover, $Y \in \scrT(\CW,C)$ and $\{V_1, \ldots, V_t \} \subseteq \scrD(\scrT(\CW,C))$. Thus, $V \in \scrD(\scrT(\CW,C))$ and inclusion in Identity (\ref{InterseccionesHyW:eqn}) follows. \\
Then, combining Inequality (\ref{CotaGradoZBooleano:eqn}) and Inclusion (\ref{InterseccionesHyW:eqn}), we conclude from Lemma \ref{cotasD:lemma} that if $C$ is locally closed:
\begin{equation} \label{DesigualadesFinalesBooleana:eqn}
\deg_z (\scrB(\CH,C)) \leq \sum_{V \in \scrD(\scrT(\CW,C))} \deg(V) \leq \deg(C) (1+ \deg(\CW))^{\dim(C)}.
\end{equation}

As for the general case, let us consider a minimal degree $LCI-$decomposition of $C$:
$$C = C_1 \cup \cdots \cup C_m.$$
Now, observe that every element in $\scrB(\CH,C)$ is a finite union of elements in $\scrB(\CH, C_1) \cup \cdots \cup \scrB(\CH, C_m)$. Then, if $V$ is an irreducible component of $\overline{X}^z$, where $X \in \scrB(\CH,C)$, there is some $i \in [m]$ and some $Y \in \scrB(\CH, C_i)$ such that $V$ is an irreducible component of $\overline{Y}^z$. In conclusion, we have:
$$\scrD(\scrB(\CH,C)) \subseteq \bigcup_{i=1}^m \scrD(\scrB(\CH, C_i)).$$
Next, from Identity (\ref{DesigualadesFinalesBooleana:eqn}), we conclude:

\begin{equation*}
\begin{aligned}
\deg_{z} (\scrB(\CH,C)) & = \sum_{V \in \scrD(\scrB(\CH,C))} \deg(V) \leq & \\
& \leq \sum_{i=1}^m \left( \sum_{V \in \scrD(\scrB(\CH, C_i))} \deg(V) \right) \leq  & \\
& \leq \sum_{i=1}^m \deg(C_i) \left( 1+ \deg(\CW) \right)^{\dim(C_i)} \leq  & \\
& \leq \deg_{\rm lci} (C) \left(1+ \deg(\CW) \right)^{\dim(C)}.  &
\end{aligned}
\end{equation*}

As $\deg(\CW) = \grad(\CH)$, Claim $ii)$ follows. \\
For Claim $iii)$, observe that every point $x$ of any finite $\CH-$definable subset $X$ is irreducible and then it belongs to $\scrD(\scrB(\CH, C))$. Therefore, applying Claim $ii)$, we obtain that the number of such points satisfies:
$$\sharp\left(X\right)=\deg(X) \leq \deg_z\left(\scrB(\CH, C)\right) \leq \deg_{\rm lci}(C) \left( 1 + \grad(\CH)\right)^{\dim(C)}. $$ 
and Claim $iii)$ of the theorem follows.\\
Finally, for Claim $iv)$, just observe that $\scrZ(\CH, C)$ defines a partition of $C$ such that every constructible set in $\scrB(\CH, C)$ is a finite union of some elements of $\scrZ(\CH, C)$. Thus, applying Claim $i)$, we get:
$$\sharp\left( \scrB(\CH, C)\right) \leq 2^{\sharp\left(\scrZ(\CH, C)\right)}\leq 
2^{\deg_{\rm lci}(C) \left( 1 + \grad(\CH)\right)^{\dim(C)}},$$
and the proof of theorem is finished.
\end{proof}

\section{Vapnik-Chervonenkis Theory over algebraically closed fields} \label{VCTheory:sec}

The notion of Erzeugungsgrad appears in \cite{Heintz83} as a tool to bound the number of non-empty cells definable by finite intersections of a finite family of definable sets. The bounds in \cite{Heintz83} and our generalizations in Proposition \ref{CotaCeldasNoVaciasDegW:prop} and Theorem \ref{cotas-Erzeugunsgrad:teor} yield a kind of \emph{growth function} to bound $\sharp(\scrZ(\CH, C))$ from above. This conceptual approach is quite similar to the \emph{growth function} used in \emph{Computational Learning Theory} with a different purpose. The goal of this section is to make explicit what was apparently similar in both approaches. \\
In Computational Learning Theory, the growth function is used to bound the number of subsets of a finite class that can be described by a finite number of restrictions of classifiers. The key ingredient in the upper bounds of a growth function is \emph{the exponent that appears in the known upper bounds}. This exponent is often referred to in Computational Learning Theory as the \emph{dimension}. Several models and cases of Computational Learning Theory yield different notions of dimension (binary, multiclass or continuous cases). The ultimate goal of the growth function used in Computational Learning Theory is to bound the minimal length of the \emph{data set} required by the learning problem to ensure a high probability of success (or a low probability of error) of any learning algorithm with random input data sets. \\
We will not provide a detailed introduction to Computational Learning Theory, nor will we discuss the wide range of known notions of dimension within this field. \emph{Our goal in this section is simply to highlight some of the connections existing between the Erzeugungsgrad (a notion from Intersection Theory) and growth functions in Computational Learning Theory over algebraically closed fields}.  \emph{Finally, we conclude with some connections between dimension in Learning Theory and Krull dimension in Algebraic Geometry.}

\subsection{Basic notions and results of Vapnik-Chervonenkis Theory} \label{basicNotionsVCTheory:subsec}

The aim of this subsection is to introduce the basic terminology and results from Computational Learning Theory that we will use in our discussion.\\
Let $Y$ be any set and let $\{0,1\}^Y$ be the class of all functions defined on $Y$ with values in $\{0,1\}$. We refer to $\{0,1\}^Y$ as the class of characteristic functions defined on $Y$. For every subset $F\subseteq Y$, we denote by $\crctc_F \in \{0,1\}^Y$ the characteristic function (sometimes called indicator function) associated to $F$ as subset of $Y$. If $X\subseteq Y$ is a subset of $Y$, for every $\chi\in \{0,1\}^Y$ we denote by $\restr{\chi}{X}\in \{0,1\}^X$ the restriction of $\chi$ to the subset $X$. \\
In the context of Computational Learning Theory, functions that map \emph{input data} to the range $\{ 0,1 \}$ are typically called \emph{(binary) ``classifiers''}. These classifiers are essentially characteristic functions that determine whether an input belongs to a particular class (labeled as 1) or not (labeled as 0). The \emph{Vapnik-Chervonenkis dimension (VC-dimension for short)}, introduced in \cite{VC:paper}, is a ``measure'' of the capacity of a set of classifiers. The significance of the VC-dimension lies in the fact that a class of classifiers with finite VC-dimension is ``learnable'' (see \cite{Valiant}, \cite{Blumer} and many other classical citations in the field). 

\begin{definition}[{\bf Shattering, VC-dimension}] \label{VCDimension:def}
 Let $Y$ be a set and $\scrH\subseteq \{0,1\}^Y$ a family of classifiers. Given a finite subset $X\subseteq Y$, we say that $\scrH$ \emph{shatters} $X$ if for every $F\subseteq X$, there is some classifier $\chi\in \scrH$ such that the restriction to $X$ of $\chi$ equals the restriction to $X$ of $\crctc_F$. Namely, $\scrH$ shatters $X$ if for every $F\subseteq X$ there is some $\chi\in \scrH$ such that:
$$\restr{\chi}{X}= \restr{\crctc_F}{X}.$$
We define the \emph{Vapnik-Chervonenkis dimension} of the family of classifiers $\scrH$ as:
$$\dim_{VC}(\scrH):=\max\{\sharp(X)\; : \; {\hbox {\rm $\scrH$ shatters $X$}}\}.$$
\end{definition}

Now, we introduce the notion of \emph{growth function}, which is widely used in the context of Computational Learning Theory. This concept will play a central role in our discussion.

\begin{definition}[{\bf Growth function}]
Let $Y$ be a set and $\scrH\subseteq \{0,1\}^Y$ a family of classifiers. For every subset $X\subseteq Y$, let $\restr{\scrH}{X}$ be the class of restrictions to $X$ of functions in $\scrH$. Namely,
$$\restr{\scrH}{X} :=\{ \restr{f}{X}\; : \; f\in \scrH\}.$$
The \emph{growth function} $G(\scrH, \cdot)$ of a family  of classifiers $\scrH$ is the function:
$$G(\scrH, \cdot): \N \longrightarrow \N,$$ 
given by the following identity for every positive integer $m\in \N$:
$$G(\scrH, m):=\sup \{ \sharp\left( \restr{\scrH}{X} \right) \; : \; X \subseteq Y, \; \sharp(X)=m\}.$$
\end{definition}

If $d := \dim_{VC}(\scrH)$, then for all $m \leq d$, we have $G(\scrH, m) = 2^m$. The famous Sauer-Shelah-Perles Lemma (cf. \cite{Sauer}, \cite{Shelah}, \cite{Haussler} or \cite{Pardo-Trace} and references therein) proves that if $\scrH$ is a family classifiers, then the growth function $G(\scrH, m)$ always takes finite values and provides an upper bound. 

\begin{lemma}[\bf Sauer-Shelah-Perles Lemma]
Let $Y$ be a set and $\scrH\subseteq \{0,1\}^Y$ a family of classifiers. Let  $d := \dim_{VC}(\scrH)$. If $m>d$, we have the following inequality:
$$ G(\scrH, m) \leq \sum_{i=0}^{d} {m \choose i} \leq \left( \frac{em}{d} \right)^d.$$
\end{lemma}

The reader will immediately notice the analogies between the upper bounds in the last statement and those in Proposition \ref{CotaCeldasNoVaciasDegW:prop} and Theorem \ref{cotas-Erzeugunsgrad:teor}. In Sauer-Shelah-Perles Lemma, the exponent of the upper bound is the VC-dimension, whereas in Proposition \ref{CotaCeldasNoVaciasDegW:prop} and Theorem \ref{cotas-Erzeugunsgrad:teor}, the exponent of the upper bound is the Krull dimension. We will explore this similarity in forthcoming subsections. \\
Finally, we state the famous Vapnik-Chervonenkis Theorem, which we will use in Section \ref{CTS-Evasivas:sec} to prove that correct test sequences are densely distributed among evasive varieties of positive dimension:

\begin{theorem}[\cite{VC:paper}] \label{teoremaVC:teor}
Let $Y$ be a set and $Z = Y \times \{0,1 \}$. Let $(Z, \scrB, \mu)$ be a probability space, where $\scrB \subseteq 2^Z$ is a $\sigma-$algebra and $\mu: \scrB  \longrightarrow [0,1]$ is a probability distribution defined on $\scrB$. Let $\scrH \subseteq \{0,1 \}^{Y}$ be a family of classifiers. Let $m \in \N$ be a positive integer and $\varepsilon>0$ a positive real number. Assume that:
$$\frac{4 log(2)}{m} \leq \varepsilon^2.$$
Let $\scrB^{\otimes m}$ be the product $\sigma-$algebra in $Z^m$ induced by $(Z, \scrB)$ and let $\mu^{\otimes m}$ be the probability distribution defined on $\scrB^{\otimes m}$ by $\mu: \scrB  \longrightarrow [0,1]$. Let $\ell \; : \; \{0,1 \}^2 \longrightarrow [0,1]$ be any non-negative function (usually called ``error function''). We have:
\begin{equation} \label{teoremaVC:eqn}
Prob_{z\in Z^m} \left[ sup_{f\in \scrH} \left| {{1}\over {m}} \sum_{i=1}^m \ell(f(x_i),y_i) - \mathbb{E}_{(x,y)\in Z}\left[ \ell(f(x),y)\right] \right| >\varepsilon \right] \leq  4 G(\scrH,m)e^{- {{m\varepsilon^2}\over{32}}}.
\end{equation}
\end{theorem}

\begin{remark}
In the current literature on the Vapnik-Chervonenkis Theorem, the constant ocurring in the right hand side of Inequality (\ref{teoremaVC:eqn}) in the previous theorem is $8$. A careful revision of the proof of the Vapnik-Chervonenkis Theorem allows us to assert that the upper bound in Inequality (\ref{teoremaVC:eqn}) holds.
\end{remark}

\subsection{VC-dimension and constructible classifiers} \label{VCTheory-LocalmenteCerradoParametros:sec}

Several studies in the mathematical literature analyze the VC-dimension of classifiers defined by semi-algebraic subsets (cf. \cite{Goldberg}, \cite{Lickteig}, \cite{Karpinsky}, \cite{VCAnalytic}, \cite{Montanna-Pardo} and references therein). Most of these studies are based on the upper bounds for the number of connected components of semi-algebraic sets established by J. Milnor (\cite{Milnor}), R. Thom (\cite{Thom}), O. Oleinik and I. Petrovski (\cite{Oleinik1}, \cite{Oleinik2}) or H. E . Warren (\cite{Warren}). These topological upper bounds can also be applied to study the VC-dimension  of classifiers in $\A^n(K)$, where $K$ is an algebraically closed field of characteristic zero. However, they do not apply in positive characteristic. This gap is filled by using families of constructible and definable subsets of $\A^n(K)$, where $K$ is an algebraically closed field of any characteristic. The geometric quantities that arise are the Krull dimension and the logarithm of the Erzeugungsgrad. This is what we discuss in this subsection. \\
In order to define a parameterized family of constructible classifiers, we proceed as follows. We consider $N,n \in \N$ two positive integers and $V \subseteq \A^N(K) \times \A^n(K)$ a constructible set. For instance, $V$ may be the graph of a polynomial mapping. We also consider a constructible subset $\Omega \subseteq \A^N(K)$, which is called \emph{the constructible set formed by the parameters of our family of classifiers}. Additionally, we have the two canonical projections restricted to $V$:

\begin{center}
\begin{tikzcd}[column sep=2.5em]
& V \arrow[dl, "\pi_1"'] \arrow[dr, "\pi_2"] & \\
\mathbb{A}^N(K) & & \mathbb{A}^n(K)
\end{tikzcd}
\end{center} \medskip

With these notations, we define the following class of constructible sets:
$$\scrC (V, \Omega) := \{ \pi_2 (\pi_1^{-1} (\{a\})) \; : \; a \in \Omega \}.$$
We also define $\grad(V) = \grad(\{ V \})$ as the Erzeugungsgrad of the class defined by $V$. Finally, we define the class of classifiers given as the characteristic functions defined by the subsets in $\scrC (V, \Omega)$:
$$\scrH (V, \Omega) := \left\{\crctc_U \; : \; U \in \scrC (V, \Omega) \right\}.$$
As in the previous subsection, for every finite subset $X \subseteq \A^n(K)$, we define the class of restrictions to $X$ of the classifiers in $\scrH(V, \Omega)$ as follows:
$$\restr{\scrH(V, \Omega)}{X} := \{ \restr{\chi}{X} \; : \; \chi \in \scrH(V, \Omega) \}. $$
Then, the following statement holds:

\begin{theorem} \label{FuncionCrecimientoDimensionConstructibles:thm}
With the same notations and assumptions as above, the following inequality holds:
$$\sharp \left( \restr{\scrH(V, \Omega)}{X} \right) \leq \deg_{\rm lci} (\Omega) (1 + \sharp(X) \grad(V))^{\dim(\Omega)}.$$
Additionally, we have:
$$G(\scrH(V, \Omega), m) \leq  \deg_{\rm lci} (\Omega) (1 +m \grad(V))^{\dim(\Omega)}.$$
\end{theorem}

\begin{proof}
Let us denote by $\Pi_1$ and $\Pi_2$ respectively the canonical projections of $\A^N(K) \times \A^n(K)$ onto $\A^N(K)$ and $\A^n(K)$. Observe that:
$$\pi_1 = \restr{\Pi_1}{V} \text{ and } \pi_2 = \restr{\Pi_2}{V}.$$
Assume, as stated in Remark \ref{ErzeugungsgradIrreducible:rmk}, that $\CW = \{ W_1, \ldots, W_s \}$ is a finite family of irreducible algebraic varieties $W_j \subseteq \A^N(K) \times \A^n(K)$ such that $V \in \scrB(\CW)$ and the following equality holds:
\begin{equation} \label{GradVSumaIrreducibles:eqn}
\grad(V) = \sum_{j=1}^s \deg(W_j).
\end{equation} 

Now, assume also that $X = \{x_1, \ldots, x_m \} \subseteq \A^n(K)$ with $m = \sharp(X)$. For every $i \in [m]$, let $L_i \subseteq \A^N(K) \times \A^n(K)$ be the linear affine variety given by the following identity:
$$L_i := \Pi_{2}^{-1} \left( \{ x_i \} \right) = \A^N(K) \times \{x_i \}.$$ 
Next, we consider the sections with $L_i$ and, then, we project by $\Pi_1$. Namely, we define:
\begin{itemize}
\item The constructible subsets:
$$V_i := \Pi_1 \left( V \cap L_i \right) = \pi_1 (\pi_2^{-1} (\{ x_i \})) \subseteq \A^N(K).$$
\item The projections:
$$W_{j,i} := \Pi_1 \left( W_j \cap L_i \right) \subseteq \A^N(K),$$
where $j \in [s]$ and $i \in [m]$.
\end{itemize}
Observe that the following equalities hold:
\begin{itemize}
\item $V \cap L_i = V_i \times \{ x_i \}$, for every $i \in [m]$.
\item $W_j \cap L_i = W_{j,i} \times \{ x_i \}$, for every $j \in [s]$, $i \in [m]$.
\end{itemize}

In particular, we observe that since $W_j \cap L_i$ is an algebraic variety, $W_{j,i}$ is also an algebraic variety of the same dimension: $W_j \cap L_i$ and $W_{j,i}$ are birregularly isomorphic. Moreover, as the birregular isomorphism between $W_j \cap L_i$ and $W_{j,i}$ is linear (the projection $\Pi_1$ restricted to $W_j \cap L_i$) and its inverse is also linear, we immediately conclude from Bézout's Inequality that:
\begin{equation} \label{AcotaWijPorWj}
\deg(W_{j,i}) = \deg(W_j \cap L_i) \leq \deg(W_j).
\end{equation}
Finally, observe that, since $V \in \scrB(\CW)$,  the finite family of constructible sets $\scrV = \{ V_1, \ldots, V_m \} \subseteq \scrB(\widetilde{\CW})$, where $\widetilde{\CW}$ is the finite family of algebraic varieties given by the following identity:
$$\widetilde{\CW} := \{ W_{j,i} \; : \; j \in [s], \; i \in [m] \}.$$
In particular, we conclude that:
\begin{equation} \label{CotaGradoV:eqn}
\grad(\scrV) \leq \deg(\widetilde{\CW}) \leq \sum_{i=1}^m \sum_{j=1}^s \deg(W_{j,i}).
\end{equation}

Combining Inequalities (\ref{AcotaWijPorWj}) and (\ref{CotaGradoV:eqn}) and Identity (\ref{GradVSumaIrreducibles:eqn}), we obtain:
\begin{equation} \label{CotaGradoscrV:eqn}
\grad(\scrV) \leq \sum_{i=1}^m \sum_{j=1}^s \deg(W_j) = m \grad(V).
\end{equation}
On the other hand, observe that the following equality holds:
\begin{equation} \label{MismasCeldasRestricciones:eqn}
\sharp \left(\restr{\scrH(V, \Omega)}{X} \right) = \sharp \left( \scrZ (\scrV, \Omega) \right),
\end{equation}
where $\scrZ(\scrV, \Omega)$ is the set of non-empty $\scrV$-cells in $\Omega$. In order to prove Identity (\ref{MismasCeldasRestricciones:eqn}), let us consider $S \subseteq [m]$ and the cell:
$$\Omega_S = \Omega \cap \left( \bigcap_{i \in S} V_i \right) \cap \left( \bigcap_{j \not\in S} \left( \A^N(K) \setminus V_j \right) \right) \subseteq \Omega.$$
If $\Omega_S$ is a non-empty $\scrV$-cell, let $a \in \Omega_S$ be any point. Then, consider:
$$V(a) = \pi_2 (\pi_1^{-1} (\{ a \})) \in \scrC(V, \Omega),$$
and $\chi := \crctc_{V(a)}$ its characteristic function. Then, we have that:
$$\restr{\chi}{X} (x_k) =1 \text{ if and only if } k \in S,$$
and $\restr{\chi}{X} \in \restr{\scrH(V, \Omega)}{X}$. Conversely, let $\chi \in \restr{\scrH (V, \Omega)}{X}$ and $U \in \scrC (V, \Omega)$ such that $\chi = \restr{\crctc_U}{X}$. Then, there is $a \in \Omega$  such that $U = \pi_2 (\pi_1^{-1} (\{a\}))$. Consider $S = \{ i \in [m] \; : \; \chi(x_i) =  1 \}$ and:
$$\Omega_S = \Omega \cap \left( \bigcap_{i \in S} V_i \right) \cap \left( \bigcap_{j \not\in S} \left( \A^N(K) \setminus V_j \right) \right) \subseteq \Omega.$$
Clearly, $a \in \Omega_S$ and, hence, $\Omega_S$ is a non-empty $\scrV$-cell in $\Omega$. Thus, we have defined a bijection between the functions in $\restr{\scrH(V, \Omega)}{X}$ and the non-empty cells in $\scrZ(\scrV, \Omega)$, as wanted. \\
Finally, we conclude from  Identity (\ref{MismasCeldasRestricciones:eqn}), Theorem \ref{cotas-Erzeugunsgrad:teor} and Inequality (\ref{CotaGradoscrV:eqn}) the following chain of equalities and inequalities:

$$\sharp \left( \restr{\scrH(V, \Omega)}{X} \right) = \sharp \left( \scrZ(\scrV, \Omega) \right) \leq \deg_{\rm lci} (\Omega) \left( 1+ \grad(\scrV) \right)^{\dim(\Omega)} \leq$$
$$\leq \deg_{\rm lci} (\Omega) \left( 1+m \grad(V) \right)^{\dim(\Omega)}.$$

For the bound on the growth function just note that:
$$G(\scrH(V, \Omega), m) = \sup \{ \sharp\left( \restr{\scrH(V, \Omega)}{X} \right) \; : \; X \subseteq \A^n(K), \; \sharp(X) = m \} \leq $$
$$\leq  \deg_{\rm lci} (\Omega) (1 +m \grad(V))^{\dim(\Omega)}. $$

\end{proof}

Our main conclusion from these technical discussions is that the VC-dimension of a family of classifiers is closely related to the Krull dimension of the parameter space, up to some logarithmic quantities based on Intersection Theory. Specifically, we conclude:

\begin{corollary} \label{VCDimensionConstructibles:corol}
With the same notations and assumptions as above, the following inequality holds:
$$\frac{s}{\log_2(s)+k} - \frac{\log_2 (\deg_{\rm lci} (\Omega))}{\log_2(s)+k} \leq \dim(\Omega),$$
where $s = \dim_{VC} (\scrH(V, \Omega))$ and $k = 1 + \log_2(\grad(V))$.
\end{corollary}

\subsection{VC-dimension of distinguished open sets defined by a class of polynomials} \label{VCAbiertosDistinguidos:subsec}

Recall that, given a polynomial $f \in K[X_1, \ldots, X_n]$, we define the \emph{distinguished open set defined by $f$} as the following Zariski open set: 
$$D(f) := \{ x \in \A^n(K) \; : \; f(x) \neq 0 \} = \A^n(K) \setminus V_{\A} (f).$$
Let $d,n \in \N$ be two positive integers and let $P_{d}^K (X_1, \ldots, X_n)$ be the class of polynomials of degree at most $d$ involving $n$ variables and coefficients in $K$. We may identify $P_{d}^K (X_1, \ldots, X_n) \cong \A^N (K)$, where $N = N(d,n)$ is the binomial number:
$$N = {d+n \choose n}. $$
As in the previous subsection, we may consider the constructible subset $V \subseteq \A^N(K) \times \A^n(K)$ given by the following equality:
$$V:= \{ (f,x) \in \A^N(K) \times \A^n(K) \; : \; f(x) \neq 0 \}.$$

We also have the two canonical projections $\pi_1: V  \longrightarrow P_{d}^K (X_1, \ldots, X_n)$ and $\pi_2 : V \longrightarrow \A^n(K)$. Let $\Omega \subseteq \A^N(K)$ be a constructible set that defines a family of polynomials by its coefficients. Observe that $V$ is the complement of an algebraic hypersurface of degree $d+1$. Hence, $\grad(V) = d+1$. \\
Additionally, $\Omega$ defines a family of distinguished Zariski open subsets of $\A^n(K)$:
$$\scrC_d (V, \Omega) = \{ \pi_2 (\pi_1^{-1} (\{a\})) \; : \; a \in \Omega \}.$$
Thus, we define the family of classifiers given by the characteristic functions of distinguished open subsets of $\A^n(K)$ defined by an equation $f \in \Omega$ as follows:
$$\scrH_d (V, \Omega) := \{ \crctc_{D(f)} \; : \; f \in \Omega \}.$$

\begin{theorem} \label{GrowthFuncionOpenZariski:thm}
With the previous notations and assumptions, for every finite set $X \subseteq \A^n(K)$, the following inequalities hold:
$$\sharp \left( \restr{\scrH_d (V, \Omega)}{X} \right) \leq \deg_{\rm lci} (\Omega) (1+ \sharp(X)  (d+1))^{\dim(\Omega)},$$
In particular, if $\sharp(X) = \dim_{VC} \left( \scrH_d(V, \Omega) \right) =s$, then:
$$\frac{s}{\log_2(s)+2+ \log_2({d})} - \frac{\log_2 (\deg_{\rm lci} (\Omega))}{\log_2(s)+2+ \log_2({d})} \leq \dim(\Omega).$$
Moreover, we have:
$$G(\scrH_d(V, \Omega), m) \leq \deg_{\rm lci} (\Omega) (1+m  (d+1))^{\dim(\Omega)}.$$
\end{theorem}

\begin{proof}
It is a straightforward application of Theorem \ref{FuncionCrecimientoDimensionConstructibles:thm} from the previous subsection.
We just need to observe that the class $\CW := \{ \A^{N+n} (K) \setminus V \}$ is a finite family of algebraic varieties such that $V \in \scrB(\CW)$ and $\deg(\A^{N+n}(K) \setminus V) = d+1$. 
\end{proof}

\subsection{Varieties which are evasive for a constructible set of list of polynomials} \label{evasive:sec}
The term \emph{``variety evasive set''} was used  in \cite{Dvir-Kollar} in a different context with a different meaning which, however, has some relation with our discussion here. \\
Basically,  in \cite{Dvir-Kollar} ``a variety evasive set'' is a set $U\subseteq \A^n(K)$ such that, for every algebraic variety $V\subseteq \A^n(K)$ of given degree and dimension, $U\cap V$ is a zero-dimensional variety (i.e. a finite set) with an upper bound on the cardinality of the intersection $\sharp(U\cap V)$. Their main outcome is a family of equations (basically a Pham system) whose zeros define an algebraic set $U\subseteq \A^{n}(\overline{\F_q})$, where $\F_q$ is a finite field and $\overline{\F_q}$ is the algebraic closure of $\F_q$.\\
Observe that by simply using Claim $xi)$ of Proposition \ref{propiedades-basicas-grado-lc:prop}, one can improve the upper bound of $\sharp(U\cap V)$ in Theorem 2.2 of \cite{Dvir-Kollar} as follows:

\begin{proposition}[{\bf Improving upper bounds of Theorem 2.2 of \cite{Dvir-Kollar}}]
Let $k, D \geq 1$ be integers and let $m > k$ be a positive integer such that $m\mid n$. Let $d_1>d_2> \cdots >d_m> D$ be relatively prime positive integers. Let $A:=\left( a_{i,j}\right)\in \scrM_{k\times m}(\F_q)$ be a $k-$regular matrix.  Let $f_1,\ldots, f_k\in \F_q[X_1,\ldots, X_m]$ be defined as follows:
$$f_i (X_1, \ldots, X_m):=\sum_{j=1}^m a_{i,j} X_j^{d_j}.$$
Let $U:=V(f_1, \ldots, f_k)\subseteq \A^m(\overline{\F_q})$ be the variety of their common zeros.
Finally, let $U^{n/m}\subseteq \A^n(\overline{\F_q})$ be the variety given as the $n/m$ times Cartesian product of $U$ with itself.
Then, for every affine algebraic variety $V\subseteq \A^{n}(\overline{\F_q})$ of degree $D$ and dimension $k$, the intersection $V\cap U^{n/m}$ is zero-dimensional (i.e. a finite set). Moreover, its cardinality satisfies:
\begin{equation}\label{Dvir-Kollar-bound:eqn}
\sharp(V\cap U^{n/m})\leq  \deg(V) \cdot d_1^k = D \cdot d_1^k.
\end{equation}
\end{proposition}
\begin{proof}

First of all, note that $U$ is the intersection of $m$ hypersurfaces  of degree at most $d_1$. Hence, $U^{n/m}$ is then an intersection of $n/m \cdot k$ hypersurfaces, each of them of degree at most $d_1$.  In order to see this, let us define for every $i$, $1\leq i \leq k$, and for every $j$, $1\leq j \leq n/m$, the polynomials $f_{i,j}\in \F_q[X_1,\ldots, X_n]$ as follows:
$$f_{i,j}:=f_i(X_{m (j-1) + 1}, \ldots, X_{mj})\in \F_q[X_1,\ldots, X_n].$$
Then, we have:
$$U^{n/m}= \bigcap_{i=1} ^k \bigcap_{j=1}^{n/m} V_{\A^n(\overline{\F_q})}(f_{i,j}).$$
and, hence:
$$\sharp(V\cap U^{n/m})= \deg(V\cap U^{n/m})= \deg\left(V \cap \left( \bigcap_{i=1} ^k \bigcap_{j=1}^{n/m} V_{\A^n(\overline{\F_q})}(f_{i,j})\right)\right),$$
Thus, applying Claim $xi)$ of Proposition \ref{propiedades-basicas-grado-lc:prop}, we obtain:
 $$\sharp(V\cap U^{n/m})=  \deg(V) \left( \max\{ \deg(V_{\A^n(\overline{\F_q})}(f_{i,j}))\; : \; 1\leq i \leq k, 1\leq j \leq n/m\} \right)^{\dim(V)}.$$
Finally, we conclude:
$$\sharp(V\cap U^{n/m}) \leq \deg(V) \cdot d_1^{\dim(V)} = D \cdot d_1^k, $$
 as stated. \end{proof}

\begin{remark}
Note that our upper bound in Inequality (\ref{Dvir-Kollar-bound:eqn}) strongly improves the main upper bound in \cite{Dvir-Kollar}. Let the reader observe that the main upper bound presented in \cite{Dvir-Kollar} was, with our notations, the following one:
$$D^{k+1} \left( \prod_{i=1}^k d_i\right)^{k}, $$
which is much worse than the one we have exhibited in Inequality (\ref{Dvir-Kollar-bound:eqn}) above.
\end{remark} 

For a list of degrees $(d):=(d_1,\ldots, d_r)$, with $r\leq n$, we introduce the class $\scrP_{(d)}^{K}(X_1,\ldots, X_n)$ of all lists of polynomials $f:=(f_1,\ldots, f_r)$ such that each $f_i \in P_{d_i}^{K}$. Namely, $\scrP_{(d)}^{K}$ is the Cartesian product:
$$\scrP_{(d)}^{K}(X_1,\ldots, X_n):=\prod_{i=1}^r P_{d_i}^{K}(X_1,\ldots, X_n).$$
In this manuscript, we use the term ``evasive'' with a different meaning: it is some kind of algebraic variety that does not intersect any complete intersection variety defined by polynomials in some constructible set $\Omega\subseteq \scrP_{(d)}^{K}$.

\begin{definition} \label{evasivas:def}
Let $(d):=(d_1,\ldots, d_r)$ be a degree list with $r\leq n-1$. Let $\Omega \subseteq \scrP_{(d)}^{K}(X_{1}, \ldots, X_{n})$ be a constructible set of lists of polynomials  and $\Sigma \subseteq \Omega$ a constructible subset of co-dimension at least $1$ in $\Omega$. Assume that for every sequence $f:=(f_1,\ldots, f_r)\in \Omega\setminus\Sigma$, the variety $V_\A(f):=V_\A(f_1,\ldots, f_r)$ has dimension $n-r$. An equi-dimensional subvariety $V \subseteq \A^{n} (K)$ is called evasive for complete intersections in $\Omega$ with respect to $\Sigma$ if the following property holds:
$$\forall f \in \Omega, \: \restr{f}{V} \equiv 0 \Rightarrow f\in \Sigma.$$
\end{definition}

Now, we prove that this kind of evasive varieties do really exist:

\begin{proposition} \label{existencia-evasivas:prop}
Let $(d):=(d_1,\ldots, d_r)$ be a degree list with $r\leq n-1$ and $\Omega \subseteq \scrP_{(d)}^{K}(X_1,\ldots, X_n)$ a constructible set of lists of polynomials. Let $\Sigma \subseteq \Omega$ be a constructible subset of co-dimension at least $1$ such that for all $f:=(f_1, \ldots, f_r)\in \Omega \setminus \Sigma$, the variety $V_\A(f):=V_\A(f_1,\ldots, f_r)$ has dimension $n-r$.  Let $m\in \N$ be any positive integer such that $1\leq m \leq n-r$. Let $\Delta \in \N$ be another positive integer such that $\Delta^r > \scrD_{(d)}:=\prod_{i=1}^r d_i$. We call $\scrD_{(d)}$  the B\'ezout's number associated to the degree list $(d)$. If $V\subseteq \A^n(K)$ is any equi-dimensional variety of dimension $m$ satisfying the following properties:
\begin{enumerate}
\item $V$ is given as the solution variety of a set of polynomials of degree at most $\Delta$ and,
\item $deg(V) = \Delta^{n-m}$,
\end{enumerate}
then $V$ is evasive for complete intersections in $\Omega$ with respect to $\Sigma$.
\end{proposition}
\begin{proof}
Assume there is a list of polynomials $g_1,\ldots, g_s\in K[X_1,\ldots, X_n]$ of degree at most $\Delta$, such that the variety of their common zeros $V:=V_\A(g_1,\ldots, g_s)$ is equi-dimensional of degree $\Delta^{n-m}$. \\
For every $f:=(f_1, \ldots, f_r)\in \Omega\setminus\Sigma$, let us denote by $V_\A(f)$ the algebraic variety $V_\A(f):=V_\A(f_1,\ldots,f_r)$ which is of dimension $n-r$. As $\deg(f_i)\leq d_i$, for each $i$, $1\leq i \leq r$, we have:
$$\deg\left(V_\A(f_1,\ldots, f_r)\right)\leq \prod_{i=1}^r \deg(f_i)\leq \prod_{i=1}^r d_i =\scrD_{(d)}.$$
Now, let us consider the $K-$vector space of matrices $\scrM_{m\times (n+1)}(K)$. For every matrix $A\in \scrM_{m\times (n+1)}(K)$, we have the following affine function:
$$\begin{matrix}
A:& \A^n(K)& \longrightarrow &\A^m(K)\\
& x & \longmapsto & (A_1(x), \ldots, A_m(x))
\end{matrix},$$
where each $A_i(x):= a_{i,1} X_1+\cdots + a_{i,n} X_n + a_{i, n+1}$ is an affine linear function (i.e. a polynomial of degree $1$). For every $f\in \Omega \setminus \Sigma$ and for every $A\in \scrM_{m\times (n+1)}(K)$, we denote by $W_A(f)$ the following intersection:
$$W_\A(f):= V_\A(f) \cap V_\A(A_1, \ldots, A_m).$$
As $1\leq m \leq n-r= \dim (V_\A(f))$, by Noether's Normalization Lemma, there exists a Zariski open subset $U_{f} \subseteq \scrM_{m\times (n+1)}(K)$ of matrices $A$ such that the following dimension equality holds:
\begin{equation}\label{evasive-desigualdad1:eqn}
\dim \left(W_\A(f)\right)=\dim\left(V_\A(f) \cap \left(V_\A(A_1, \ldots, A_m) \right)\right)= \dim (V_\A(f))-m= n-r-m.
\end{equation}
The open set $U_f$ is the set defined by submatrices of $m$ rows of any generic Noether normalization of $V_\A(f)$. \\
On the other hand, as $V$ is unmixed (i.e. all its irreducible components have the same dimension), by Proposition \ref{interpretacion-geometrica-grado:prop}, there is also a Zariski open subset $U_{1} \subseteq \scrM_{m\times (n+1)}(K)$ of matrices such that if $A\in U_1$, then $V\cap V_\A(A_1, \ldots, A_m)$ is a finite set and the following equality holds:
\begin{equation}\label{evasive-desigualdad2:eqn}
\deg\left(V \cap V_\A(A_{1}, \ldots, A_{m})\right) = \sharp(V \cap V_\A(A_{1}, \ldots, A_{m})) = \deg(V) = \Delta^{n-m}.
\end{equation}
Since $\scrM_{m\times (n+1)}(K)$ is irreducible, any two open sets are Zariski dense and, then, they have non-empty intersection. Thus, for every matrix $A\in U:=U_f\cap U_1\not=\emptyset$, we additionally have that the following holds:
\begin{equation}\label{evasive-desigualdad3:eqn}
\deg(W_\A(f))=\deg(V_\A(f) \cap V_\A(A_{1} \cap \ldots \cap A_{m})) \leq \deg(V_\A(f)) \leq \scrD_{(d)}.
\end{equation}
Thus, by Claim $xi)$ of Proposition \ref{propiedades-basicas-grado-lc:prop}, we deduce:
\begin{equation*} \label{evasive-desigualdad4:eqn}
\begin{aligned}
\deg\left(W_\A(f) \cap V\right) & = \deg\left(W_\A(f) \cap V_\A(g_1,\ldots, g_s)\right) \leq  \\ & \leq \deg(W_{\A}(f))\left( \max\{ \deg(g_i)\; : \; 1\leq i \leq s\}\right)^{\dim(W_\A(f))}.
\end{aligned}
\end{equation*}
From Equality (\ref{evasive-desigualdad1:eqn}) and Inequality (\ref{evasive-desigualdad3:eqn}), we obtain:
$$\deg\left(W_\A(f) \cap V\right)\leq \scrD_{(d)} \left( \max\{ \deg(g_i)\; : \; 1\leq i \leq s\}\right)^{n-r-m}\leq \scrD_{(d)} \Delta^{n-m-r}.$$
As $\Delta^r> \scrD_{(d)}=\prod_{i=1}^r d_i$, from Equality (\ref{evasive-desigualdad2:eqn}) we conclude:
$$\Delta^{n-m}= \deg\left(V\cap V_\A(A_1,\ldots, A_m)\right)= \Delta^r\Delta^{n-m-r}> \scrD_{(d)} \Delta^{n-m-r} \geq $$ 
$$\geq \sharp\left( \left(V\cap V_\A(A_1,\ldots, A_m)\cap V_\A(f)\right)\right).$$
In particular, there must be a point $x\in V\cap V_\A(A_1,\ldots, A_m)$ which is not in $V_\A(f)$. In other words, there must be a point $x\in V$ such that $f(x)\not=0$. Thus, we conclude that $V$ is evasive with respect to complete intersections in $\Omega$ with respect to $\Sigma$. \end{proof}

A particular case is that of hypersurfaces, which may be defined as follows:

\begin{definition}
Let $d\in \N$ be a positive integer. Let $\Omega \subseteq P_d^{K}(X_{1}, \ldots, X_{n})$ be a constructible set such that $\Omega\not=\{0\}$.  An equi-dimensional subvariety $V \subseteq \A^{n} (K)$ is called evasive for hypersurfaces in $\Omega$ if the following property holds:
$$\forall f \in \Omega, \: \restr{f}{V} \equiv 0 \Rightarrow f=0.$$
\end{definition}

 By taking $r=1$ in Proposition \ref{existencia-evasivas:prop}, we easily conclude that evasive varieties for hypersurfaces exist:

\begin{corollary} \label{hipersuperficies-evasivas:corol}
Let $d\in \N$  be a positive integer and $\Omega \subseteq P_{d}^{K}(X_1,\ldots, X_n)$ a constructible set. Let $m\in \N$ be any positive integer such that $1\leq m \leq n-1$. Let $\Delta \in \N$ be another positive integer such that $\Delta > d$. If $V\subseteq \A^n(K)$ is any equi-dimensional variety of dimension $m$ satisfying the following properties:
\begin{enumerate}
\item $V$ is given as the solution variety of a set of polynomials of degree at most $\Delta$ and,
\item $deg(V) = \Delta^{n-m}$,
\end{enumerate}
then $V$ is evasive for hypersurfaces in $\Omega$ .
\end{corollary}

\subsection{Correct Test Sequences are densely distributed also within evasive varieties of positive dimension} \label{CTS-Evasivas:sec}

We start the section by recalling the notion of correct test sequence:

\begin{definition} \label{CTS:def}
Let $\Omega \subseteq \scrP_{d}^{K}(X_1,\ldots, X_n)$ be a constructible set of polynomials and $\Sigma \subsetneq \Omega$ a subset. A correct test sequence of lenght $L$ for $\Omega$ with respect to $\Sigma$ is a finite set of $L$ elements ${\bf Q} := \{ x_1, \ldots, x_L \} \subseteq (K^n)^L$ such that the following formula holds:
$$\forall f \in \Omega, \; f(x_1) = \ldots = f(x_L) = 0 \in K \Rightarrow f \in \Sigma.$$
\end{definition}

In this section, we will focus on the case $\Sigma = \{0 \}$.  Next, we prove that correct test sequences are densely distributed among any irreducible evasive variety of positive dimension with respect to any well-behaved probability distribution:

\begin{theorem}\label{longitud-CTS-VC-dimension-positiva:teor}
Let $\Omega\subseteq P_d^{K}(X_1,\ldots, X_n)$ be a constructible set of polynomials of degree at most $d$. Let $V\subseteq \A^n(K)$ be an irreducible variety of positive dimension that is evasive for hypersurfaces in $\Omega$. Let $\scrB \subseteq 2^V$ be any $\sigma-$algebra that contains the Borel subsets of $V$ with respect to the Zariski topology. Let $\mu: \scrB \longrightarrow [0,1]$ be a probability distribution on $\scrB$ that satisfies the following property: for every Zariski closed subset $A \subseteq V$, if $\dim(A) < \dim(V)$, then $\mu(A) =0$. Let $L \in \N$ be a positive integer,  $\scrB^{\otimes L}$  the $\sigma-$algebra in the product $V^{L}$ induced by $(V, \scrB)$ and $\mu^{\otimes L}$ the probability distribution defined on $\scrB^{\otimes L}$ by $\mu : \scrB \longrightarrow [0,1]$. Let $CTS(\Omega, V, L)$ be the class of all sequences ${\bf Q} \in V^L$ that are a correct test sequences for $\Omega$ with respect to $\{0\}$. Assume that the following inequality holds:
\begin{equation}\label{longitud-CTS-VC-dimension-positiva:eqn}
64 \left(1+ \frac{1+ \log(\deg_{\rm lci} (\Omega))}{\dim(\Omega)} + \log(L(d+1)) \right) < \frac{L}{\dim(\Omega)} 
\end{equation}
where $\log$ stands for the natural logarithm. Then, the following inequality holds:
$${\rm Prob}_{x \in V^L}\left[ x \in CTS(\Omega, V, L) \right] \geq 1 - \frac{1}{\deg_{\rm lci} (\Omega) e^{\dim(\Omega)}}.$$
\end{theorem}
\begin{proof}
The classical Vapnik-Chervonenkis Theorem (see Theorem \ref{teoremaVC:teor}) may be interpreted as an estimate for MonteCarlo methods on a variety $V$ endowed with any probability distribution. This interpretation implies the following statement for the MonteCarlo method associated to a family $\scrF$ of characteristic functions defined on $V$. For every point $x\in V^L$, we denote by $x_i\in V$ the $i-$th component of $x:=(x_1, \ldots, x_L)$. Then, we have:
\begin{equation}\label{Vapnik-Chervonenkis-equation-CTS:eqn}
Prob_{x\in V^L}\left[sup_{f\in \scrF}\left| {{1}\over {L}} \sum_{i=1}^L f(x_i) - \int_{z\in V}f(z)\; d\mu(z)\right|>\varepsilon \right] \leq  4 G(\scrF,L)e^{- {{L\varepsilon^2}\over{32}}},
\end{equation}
where $\varepsilon>0$ is a positive real number and $G(\scrF, m)$ is the growth function associated to $\scrF$. With our previous notations, assume that $\scrF:= \scrH_d (V', \Omega)$, where 
$$V' =  := \{ (f,x) \in \A^N(K) \times \A^n(K) \; : \; f(x) \neq 0 \},$$ and that $V$ is an irreducible variety evasive for hypersurfaces in $\Omega$. Thus, given $f\in \Omega\setminus \{0\}$, $f$ does not vanish identically on $V$ (because $V$ is evasive) and, hence, $V_\A(f)\cap V\lestricto V$ is a proper algebraic subvariety of lower dimension. Due to our hypothesis on the probability distribution $\mu$ on $V$, we have that for all $f\in \Omega\setminus \{0\}$, $\mu\left( V_\A(f) \cap V \right))=0$. Thus, we conclude:
$$\forall f\in \Omega\setminus\{0\}, \; \int_{z\in V} \crctc_{D(f)}(z) \; d \mu(z)=1.$$
In Theorem \ref{GrowthFuncionOpenZariski:thm} above, we have stated the following upper bound for the growth function of the class $\scrF$:
$$G(\scrF, L) \leq \deg_{\rm lci}(\Omega) (1+L  (d+1))^{\dim(\Omega)}.$$
Thus, Equation \ref{Vapnik-Chervonenkis-equation-CTS:eqn} becomes:
\begin{equation}\label{Vapnik-Chervonenkis-equation-CTS1:eqn}
Prob_{x\in V^L}\left[sup_{f\in \scrF}\left| {{1}\over {L}} \sum_{i=1}^L \crctc_{D(f)}(x_i) - B(f)\right|>\varepsilon \right] \leq  4 \deg_{\rm lci}(\Omega) (1+L  (d+1))^{\dim (\Omega)}e^{- {{L\varepsilon^2}\over{32}}},
\end{equation}
where:
$$B(f):=\left\{\begin{tabular}{lr} 1, & {\hbox {\rm if $f\not=0$}}\\
0 , & {\hbox{\rm otherwise.}}\end{tabular} \right.$$
In the case $f$ is the zero function, i.e $f \equiv 0$, we obviously obtain:
$$\left| {{1}\over {L}} \sum_{i=1}^L \crctc_{0}(x_i) - B(f)\right|=0.$$
Therefore, we also have:
\begin{equation}\label{Vapnik-Chervonenkis-equation-CTS2:eqn}
Prob_{x\in V^L}\left[sup_{f\in \scrF\setminus\{0\}}\left| {{1}\over {L}} \sum_{i=1}^L \crctc_{D(f)}(x_i) - 1\right|>\varepsilon \right] \leq  4 \deg_{\rm lci}(\Omega) (1+L  (d+1))^{\dim (\Omega)}e^{- {{L\varepsilon^2}\over{32}}},
\end{equation}
In order to have the same level of probability of error than in the zero-dimensional case (cf. Corollary 5.6 in \cite{PardoSebastian} with $m=1$), assume that:
$$\varepsilon^2 = \frac{32}{L} \left( \log(4) + 2 \log(\deg_{\rm lci} (\Omega)) + \dim(\Omega) (1+ \log(1+L (d+1)))  \right).$$
Note that:
$$\frac{4 \log(2)}{L} \leq \varepsilon^2,$$
which means that $\varepsilon$ satisfies the hypothesis of Theorem \ref{teoremaVC:teor}. Additionally, we have:
$$\varepsilon^2 \leq  \frac{64 \dim(\Omega)}{L} \left(1+ \frac{1+ \log(\deg_{\rm lci} (\Omega))}{\dim(\Omega)} + \log(L(d+1)) \right)  < 1.$$
and, then $\varepsilon <1$ also holds.
Now, observe that:
\begin{equation}\label{hipotesis-tecnica-CTS-VC-dimension-positiva1:eqn}
4 \deg_{\rm lci}(\Omega) (1+L (d+1))^{\dim (\Omega)}e^{- {{L\varepsilon^2}\over{32}}} = \frac{1}{\deg_{\rm lci} (\Omega) e^{\dim(\Omega)}}.
\end{equation}
Then, we obtain:
$${\rm Prob}_{x \in V^L}\left[ x\not\in CTS(\Omega, V, L) \right]=
{\rm Prob}_{x \in V^L}\left[sup_{f\in \scrF\setminus\{0\}}\left| {{1}\over {L}} \sum_{i=1}^L \chi_{D(f)}(x_i) - 1\right|=1\right].$$
Therefore, the following also holds:
$${\rm Prob}_{x \in V^L}\left[x\not\in CTS(\Omega, V, L) \right] \leq
{\rm Prob}_{x \in V^L}\left[sup_{f\in \scrF\setminus\{0\}}\left| {{1}\over {L}} \sum_{i=1}^L \chi_{D(f)}(x_i) - 1\right|>\varepsilon\right].$$
From Identity (\ref{hipotesis-tecnica-CTS-VC-dimension-positiva1:eqn}), we obtain:
$${\rm Prob}_{x \in V^L}\left[ x\not\in CTS(\Omega, V, L) \right]\leq \frac{1}{\deg_{\rm lci} (\Omega) e^{\dim(\Omega)}}.$$
Finally, we conclude:
$${\rm Prob}_{x \in V^L}\left[ x \in CTS(\Omega, V, L) \right] \geq 1 - \frac{1}{\deg_{\rm lci} (\Omega) e^{\dim(\Omega)}}.$$

\end{proof}

\section{Neural networks with rational activation function} \label{redes-neuronales:sec}

The notion of \emph{neural network} is ubiquitous in modern Computational Learning. In many contributions within this field, classifiers and functions to be learned are represented as neural networks. However, it is hard to find in the literature a precise mathematical definition of what a neural network is, as authors generally work with ``ad hoc'' notions adapted to their specific case of study. To address this gap, we present a general mathematical definition for neural networks. \\
We begin by describing the syntax of neural networks. Formally, neural networks are graphs endowed with underlying structures where certain functions from a class $\scrA$ are evaluated. On the semantic side, these functions, known as \emph{activation functions}, operate on linear combinations of the fan-in of each node. The choice of activation functions plays a crucial role in determining the set of functions that the network can evaluate. \\
Once the abstract syntax is established in Subsection \ref{syntax-neuralnetwork:subsec}, we proceed to describe the semantics of a neural network in Subsection \ref{semantics-neuralnetwork:subsec}. Given a fixed syntax and a class $\scrA$ of activation functions, the network parameters can be instantiated.  Specifically, if $\scrN$ is a neural network and $\Lambda$ is a class of admissible parameters, we define $W(\scrN, \Lambda)$ as the class of all functions that $\scrN$ can evaluate using parameters from $\Lambda$. A significant part of the literature on  Computational Learning Theory focuses on studying the learnability of $W(\scrN, \Lambda)$ and its capacity to approximate functions we want to learn. \\
At first sight, the reader might think that this topic is unrelated to those discussed in previous sections. However, this is a misconception. In Subsection \ref{exmaples-neuralnetworks:subsec}, we present several examples of ``ad hoc" neural networks commonly found in standard literature. In particular, we have included Example \ref{SLPs-t2:ejem}, which focuses on neural networks with  activation function $\varphi(t) = t^2$. Interestingly, neural networks with $\varphi(t) = t^2$ as activation function correspond precisely to what classical authors in Algebraic Complexity Theory referred to as \emph{straight-line programs} (see \cite{os54}, \cite{Motzkin}, \cite{Strassen:vermeidung}, \cite{Heintz89} etc.).  Therefore, this is is not a new concept. A comprehensive study of neural networks with activation function $\varphi(t) = t^2$ can be found in Section 3 of \cite{Krick-Pardo96}. These authors did not call these objects neural networks but rather \emph{evaluation schemes}. Their work contains a set of results that explain in detail several strategies for working with these evaluation schemes (or, equivalently, neural networks with activation function $\varphi(t) = t^2$) using syntax that considers not only the size of the underlying graph but also its depth, which measures how the network's layers interact with the subsequent layers. \\
All this work may seem useless and purely abstract. However, neural networks are deeply involved in all problems of Computational Algebraic Geometry (also known as classical \emph{Elimination Theory} in the 19th century). In \cite{Cortona} and \cite{Krick-Pardo96}, it is proven that \emph{all elimination polynomials admit a representation as neural networks with activation function $\varphi(t) = t^2$} of ``admissible'' size and small depth.  More precisely, if we define an elimination polynomial as a polynomial $f \in K[U_1, \ldots, U_m]$ that arises as the projection of an intersection $V(F_1, \ldots, F_s, g)$, where $F_i, g \in [X_1, \ldots, X_n, U_1, \ldots, U_m]$ are polynomials of degree at most $d$, then  $f$ can be represented as a well-paralelizable neural network of size $d^{O(n+m)}$, with $\varphi(t)=t^2$ as the unique activation function and using only the parameters needed to represent $\{F_1, \ldots, F_s, g \}$. Improvements on the results in \cite{Cortona} and \cite{Krick-Pardo96} can be found in the series on  the TERA algorithm: \cite{Pardo-survey}, \cite{SolvedFast}, \cite{Kronecker-CRAS}, \cite{lower-diophantine:jpaa}, \cite{SLPsElimination} or \cite{ArithmeticNullstellensatz}. All of these studies heavily rely on the fact that all polynomials occuring in Computational Algebraic Geometry admit reasonable presentations as \emph{evaluation schemes} (also known as neural networks with activation function $\varphi(t) = t^2$). \\
In summary, neural networks are intrinsically involved in Elimination Theory, serving as encodings of elimination polynomials. This establishes a precise connection between Computational Algebraic Geometry and Computational Learning Theory, whose consequences are still unknown.  This is why we have included this section in our treatise, which relates Intersection Theory (Erzeugungsgrad) and Computational Learning Theory (VC-dimension).

\subsection{Syntax of a Neural Network} \label{syntax-neuralnetwork:subsec}

We start by fixing the notation for neural networks:

\begin{definition}[{\bf Neural Network}]\label{neural-network:def}
Let $K$ be any field. A neural network $\scrN$ over the field $K$ with input variables $\{ X_1, \ldots, X_n \}$ is a triplet $\scrN:=(\scrG, \scrA, \Phi)$, where:
\begin{enumerate}
\item $\scrG:=(V,E)$ is a directed acyclic graph, with vertices (representing nodes or neurons) given by:
$$V:=\{ (i,j)\; :\; 0\leq i \leq \ell, 1\leq j \leq L_i\}.$$
We use $\ell$ to denote the depth of the neural network. In the pair $\nu:=(i,j)$, the first coordinate $i$, $0 \leq i \leq \ell$, is called the depth of the node (denoted by $d(\nu)$), and the second coordinate $j$, $1 \leq j \leq L_i$, refers to an enumeration of the set $\scrL_i:=\{ (i,j) \; : \; 1\leq j \leq L_i\}$, which contains all the vertices at depth $i$, where $\sharp(\scrL_i) = L_i$. Additionally, we assume that the set of edges $E$ satisfies:
$$E\subseteq \{ \left(\nu_1, \nu_2\right)\in V^2\; :\; d(\nu_1) \leq d(\nu_2)-1\}.$$
In other words, the fan-in of the nodes at depth $i \geq 1$ consists of nodes of depth at most $i-1$
\begin{footnote}{This is an elementary restriction that can be replaced by the following condition: ``the fan-in of nodes of depth $i$ can consist only of nodes with depth $i-1$''. We choose to restrict the fan-in to nodes of depth at most $i-1$ rather than nodes of depth $i-1$ to simplify the proofs. However, it is clear that these two conditions are equivalent.}\end{footnote}. We denote by ${\rm Fan-In}(\nu)$ the fan-in of node $\nu\in V$ in graph $\scrG$. The graph has $L_\ell$ nodes with empty fan-out which are called output nodes or neurons. The set of output nodes is denoted by $O_\scrN$, where $\sharp(O_{\scrN}) = L_{\ell}$.
\item $\scrA$ is a class of partially defined univariate functions $f:Dom(f)\subseteq K \longrightarrow K$.
\item The last element of the triplet, $\Phi$, is a mapping that assigns labels to every node of the graph according to the following rules. First, there is a set of variables, called \emph{parameters of the network}, associated to every edge of the network. We denote them as follows:
    $$P_\scrN:=\{  A_{\nu}^{\mu}\; :\; \nu \in V, \; \mu \in {\rm Fan-In}(\nu)\}.$$
    We denote by $\scrP(P_\scrN)$ the class of subsets of this space of parameters. Then, the mapping $\Phi$ is defined as:
    $$\Phi :  V  \longrightarrow  \left(\scrA \times \scrP(P_\scrN)\right)  \bigcup \{1, X_1,\ldots, X_n\},$$
    and is given by the following identity:
\begin{equation*}\label{neural-network-syntax-assignment:eqn}
\Phi(\nu):=\left\{\begin{array}{lr}
\left( f_{\nu}, \left\{A_\nu^{\mu} \; :\; \mu \in {\rm Fan-In}(\nu)\right\}\right)\in \scrA\times \scrP(P_\scrN) & {\hbox {\rm if $d(\nu)\geq 1$}}\\
1, & \hbox{\rm if $ \nu=(0,0)$}\\
X_j, & \hbox{\rm if $\nu=(0,j)$, $1\leq j \leq n$}
\end{array}\right.
. \end{equation*}
Namely, this mapping associates to each node $\nu$, with $d(\nu) \geq 1$, a pair consisting of an activation function $f_\nu :Dom(f_{\nu})\subseteq K \longrightarrow K\in \scrA$, referred to as the activation function at node $\nu$, and the list of parameters associated to its fan-in ${\rm Fan-In}(\nu)$. Nodes of depth $0$ are associated to variables in $\{X_1, \ldots, X_n\}$ and the constant $1$.
\end{enumerate}
The number of non-input nodes, $L(\scrN):=\sum_{i=1}^\ell L_i$, is referred to as the \emph{total  size of the network}. Finally, the maximum of the cardinalities of the fan-ins of any node is called the \emph{space requirements of the network}:
    $$S(\scrN):=\max\{ \sharp \left( {\rm Fan-In}(\nu)\right) \; :\; \nu \in V\}.$$
\end{definition}

\subsection{Semantics of a Neural Network} \label{semantics-neuralnetwork:subsec}

Let $ \scrN:=(\scrG,\scrA, \Phi)$ be a neural network as defined in the previous subsection. First, we consider the $K-$vector space $K_{\circ}^{(n)}[\scrA]$ consisting of all partial functions $f:Dom(f)\subseteq \A^n(K)\longrightarrow K$ generated by the constant function, projections  (represented by the variables in $\{X_1,\ldots, X_n\}$) and the activation functions in $\scrA$, and stable under linear combinations and composition. We call $K_{\circ}^{(n)}[\scrA]$ the class of \emph{partial functions with input variables $\{X_1,\ldots, X_n\}$ potentially evaluable by a neural network using functions in $\scrA$ as activation functions}. \\
A neural network $\scrN= (\scrG, \scrA, \Phi)$  is a syntactical representation of a class of partial functions $W(\scrN)\subseteq K_{\circ}^{(n)}[\scrA]$ that may be evaluated by $\scrN$ by \emph{instantiating the parameters}. This semantic process is defined as follows:
\begin{definition}
Let $\scrN= (\scrG, \scrA, \Phi)$ be a neural network. Instantiating $\scrN$ consists on fixing a vector of parameters: $$\underline{a}:=( a_\nu^{\mu}\; : \; \nu \in V, \; \mu\in {\rm Fan-In}(\nu))\in K^N,$$
where $N$ is the total number of edges in the graph $\scrG$. An instantiated neural network $\scrN$ at $\underline{a}$ is  a neural network $\scrN_{\underline{a}}:= (\scrG, \scrA, \Phi_{\underline{a}})$, where $\scrG$ and $\scrA$ are those of $\scrN$ and    $\Phi_{\underline{a}}$  is an assignment to every node determined by $\Phi$ and $\underline{a}$ as follows:
\begin{equation*}\label{neural-network-semantics-assignment:eqn}
\Phi_{\underline{a}}(\nu):=\left\{\begin{array}{lr}
\left( f_{\nu}, \left(a_\nu^{\mu} \; :\;  \mu\in {\rm Fan-In}(\nu)\right\}\right)\in \scrA\times K^{\sharp\left({\rm Fan-In}(\nu)\right)}& {\hbox {\rm if $d(\nu)\geq 1$}}\\
1, & \hbox{\rm if $ \nu=(0,0)$}\\
X_j, & \hbox{\rm if $\nu=(0,j)$, $1\leq j \leq n$}
\end{array}\right.
.\end{equation*}
\end{definition}

\begin{definition}
Given a neural network $\scrN= (\scrG, \scrA, \Phi)$ with $N$ edges and input nodes associated to variables $\{X_1,\ldots, X_n\}$, a parameter list $\underline{a}:=( a_\nu^{\mu}\; : \; \nu \in V, \; \mu\in {\rm Fan-In}(\nu))\in K^N$ and the instantiation $\scrN_{\underline{a}} = (\scrG, \scrA, \Phi_{\underline{a}})$, we may associate to every node $\nu$ in the underlying graph of $\scrN$ an instantiated function $f_\nu^{\scrN}(\underline{a}, \cdot)\in K_{\circ}^{(n)}[\scrA]$ recursively defined in terms of depth as follows:
\begin{itemize}
\item For input nodes, we define the following functions for every $x:=(x_1,\ldots, x_n)\in K^n$: 
$$f_{(0,0)}^{\scrN}(\underline{a}, x):= 1, \; f_{(0,j)}^{\scrN}(\underline{a}, x)= x_j, \; 1\leq j \leq n.$$
\item For nodes $\nu$ of depth $d(\nu) \geq 1$, given $x:=(x_1,\ldots, x_n)\in K^n$ in the domain of definition of all the functions in the class $\{f_\mu(\underline{a}, \cdot) \; : \; \mu\in {\rm Fan-In}(\nu)\}$, we define:
$$f_\nu^\scrN\left(\underline{a}, x\right):=f_\nu\left( \sum_{\mu\in {\rm Fan-In}(\nu)} a_{\nu}^{\mu}f_{\mu}^{\scrN}(\underline{a}, x)\right),$$
provided that:
$$\sum_{\mu\in {\rm Fan-In}(\nu)} a_{\nu}^{\mu}f_{\mu}^{\scrN}(\underline{a}, x)\in Dom(f_\nu).$$
\end{itemize}
Sometimes, we also say that the neural network $\scrN$, instantiated at $\underline{a}$, is a device to represent all partial functions associated with all nodes:
$$\{ f_\nu^\scrN(\underline{a}, \cdot): Dom(f_\nu^\scrN(\underline{a}, \cdot)) \subseteq \A^n(K) \longrightarrow K \; : \; \nu\in V\} \subseteq K_{\circ}^{(n)}[\scrA].$$
\end{definition}

\begin{definition}
Let $\scrN$ be a neural network as in the previous definition and let  $m:=\sharp(O_\scrN)=L_\ell$ be the number of output nodes of the neural network $\scrN$. Let $\underline{a}$ be an instantiation of the parameters. We say that $\scrN_{\underline{a}}$ evaluates the partial mapping:
\begin{equation}\label{aplicacion-evaluada-por red-neuronal:eqn}
\begin{matrix}
f_\scrN(\underline{a}, \cdot): & Dom(f_\scrN(\underline{a}, \cdot))\subseteq \A^n(K) & \longrightarrow & K^{m}\\
& x & \longmapsto & \left( f_\nu^\scrN(\underline{a}, x_1,\ldots, x_n)\; :\; \nu\in O_\scrN\right)
\end{matrix},
\end{equation}
where $Dom(f_\scrN(\underline{a}, \cdot))\subseteq \A^n(K)$ is the (possibly empty) domain where the function $f_\scrN(\underline{a}, \cdot)$ is defined, which depends on $\scrN$ and the parameter instance $\underline{a}\in K^N$.
\end{definition}

Finally, we may define the \emph{class of functions evaluable by a neural network by instantiation of the parameters} as follows:
\begin{definition}
Let $\scrN:=(\scrG, \scrA, \Phi)$ be a neural network with $N$ edges, input nodes associated to variables $\{X_1,\ldots, X_n\}$ and $m$ output nodes in $O_\scrN$. Let $\Lambda \subseteq \A^N(K)$ be a subset. We define the class of (partial) mappings defined by $\scrN$ with parameters instantiated at $\Lambda$ as the following subset of $( K_{\circ}^{(n)}[\scrA] )^m$:
$$\CW(\scrN, \Lambda):=\{ f_\scrN(\underline{a}, \cdot)\in ( K_{\circ}^{(n)}[\scrA] )^m \; :\; \underline{a}\in \Lambda\}.$$

\end{definition}

\subsection{Examples of activation functions: ReLUs, evaluation schemes, polynomials, rationals, Nash functions and Pfaffian functions} \label{exmaples-neuralnetworks:subsec}

\begin{example}[{\bf ReLU neural networks and tropical algebraic rational functions}]
It is the simplest and most widely used class of neural networks, but its study is beyond the scope of this manuscript. The ground field $K=\R$ is the field of real numbers. The class $\scrA$ of activation functions is formed by a single function, sometimes denoted as $(\cdot)_+$ (cf. \cite{Poggio:Deep-but} and references therein):
$$\begin{matrix}(\cdot)_+:& \R & \longrightarrow & \R_+\\
& x & \longmapsto & (x)_+:=\max\{x, 0\}.\end{matrix}$$
The class $\R_{\circ}^{(n)}[\scrA]$ of potentially evaluable functions by a ReLU neural network with $n$ input nodes $\{X_1,\ldots, X_n\}$ and a single output node is an extension of the class of \emph{tropical algebraic rational functions} (cf. \cite{Mikhalkin}, \cite{Sturmfels}, \cite{ReLUTropical}, \cite{Grigorieva} or \cite{Grigorievb} and references therein for further information) admitting real coefficients. Recall that a tropical polynomial over $\R$ is a function $f:\R^n \longrightarrow \R$ such that there exists $a_1,\ldots, a_m \in \Z^n$ and $b_1,\ldots, b_m\in \R$ such that:
$$f(x):= \min\{ \langle a_1 , x \rangle + b_1, \ldots, \langle a_m, x \rangle + b_m\},$$
where $\langle a, x\rangle:=\sum_{i=1}^n a_ix_i$ for each $a=(a_1,\ldots, a_n)\in \Z^n$ and $x=(x_1,\ldots, x_n)\in \R^n$. A \emph{tropical algebraic rational function} is the difference between two tropical polynomials (cf. \cite{Grigorieva} and references therein). It may also be defined as a continuous piecewise linear function with integer slopes. Changing:
$$\min\{ \langle a_1 , x \rangle + b_1, \ldots, \langle a_m, x \rangle + b_m\}=-\max\{-\langle a_1 , x \rangle - b_1, \ldots, -\langle a_m, x \rangle - b_m\},$$
makes the deal between both theories. Given a ReLU neural network $\scrN:=(\scrG,\scrA, \Phi)$ with a single output node, the class $\R_{\circ}^{(n)}[\scrA]$ of potentially evaluable functions by $\scrN$ is defined as follows (see Proposition 5.6 of \cite{ReLUTropical} for further details): a partial function $f: Dom(f)\subseteq \R^n\longrightarrow \R$  satisfies that $f\in \R_{\circ}^{(n)}[\scrA]$ if and only if there exist $a_1,\ldots, a_m, c_1, \ldots, c_r \in \R^n$ and $b_1,\ldots, b_m, d_1, \ldots, d_r\in \R$ such that:
$$f(x):= \min\{ \langle a_1 , x \rangle + b_1, \ldots, \langle a_m, x \rangle + b_m\}-
\min\{ \langle c_1 , x \rangle + d_1, \ldots, \langle c_r, x \rangle + d_r\}.$$
\end{example}

\begin{example}[{\bf Evaluation schemes or neural networks with activation function $\varphi(t)=t^2$}] \label{SLPs-t2:ejem}
We consider a neural network $\scrN:=(\scrG, \scrA, \Phi)$ over a field $K$ and set of input variables $\{X_1,\ldots, X_n\}$. We assume that the class of activation functions $\scrA$ is reduced to a single function $\varphi:K \longrightarrow K$ given by $\varphi(t)=t^2$. This class of neural networks is exactly the same as the class of programs described in Sections 2 and 3 of \cite{Krick-Pardo96} as \emph{evaluation schemes}. In order to be complete, we make explicit the \emph{equivalence between evaluation schemes and neural networks with activation function $\varphi(t)=t^2$}.  Using the previous notations, the function evaluated at a non-input node $\nu$ in a evaluation scheme (as defined in \cite{Krick-Pardo96}) is given by:
$$f_{\nu} (\underline{a}, x) = \left( \sum_{\mu \in {\rm Fan-In}(\nu)} a_{\nu}^{\mu,1} f_{\mu} (\underline{a}, x) \right) \cdot \left( \sum_{\mu \in {\rm Fan-In}(\nu)} a_{\nu}^{\mu,2} f_{\mu} (\underline{a}, x) \right),$$
Next, observe that the next equality relates the product to the squaring operation:
$$XY = \frac{1}{2} \left( (X+Y)^2 - X^2 -Y^2 \right).$$
Therefore, the function evaluated at node $\nu$ in the evaluation scheme is equivalent to the following neural network $\scrN_{\nu}$ of size $4$ and depth $2$: \\
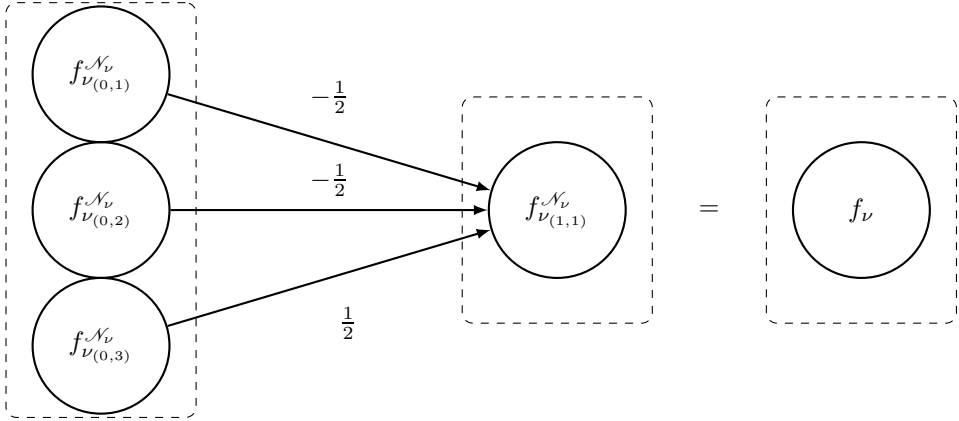
\begin{center}
\begin{tikzpicture}
[	
    neuron/.style={circle, draw, minimum size=1.8cm, thick},
    layer/.style={rectangle, draw, minimum height=5.5cm, minimum width=2.5cm, dashed, rounded corners},
    final/.style={rectangle, draw, minimum height=3cm, minimum width=2.5cm, dashed, rounded corners},
    >=latex
]

\node[layer] (layer1) at (0,0) {};
\node[final] (layer2) at (6,0) {};

\node[neuron] (N11) at (0,1.8) {$f_{\nu_{(0,1)}}^{\scrN_{\nu}}$};
\node[neuron] (N12) at (0,0) {$f_{\nu_{(0,2)}}^{\scrN_{\nu}}$};
\node[neuron] (N13) at (0,-1.8) {$f_{\nu_{(0,3)}}^{\scrN_{\nu}}$};

\node[neuron] (N21) at (6,0) {$f_{\nu_{(1,1)}}^{\scrN_{\nu}}$};

\node at (8,0) {$=$};

\node[final] (layer3) at (10,0) {};
\node[neuron] (N31) at (10,0) {$f_{\nu}$};

\draw[->, thick] (N11) -- (N21);
\draw[->, thick] (N12) -- (N21);
\draw[->, thick] (N13) -- (N21);

\node at (3,1.5) {$-\frac{1}{2}$};
\node at (3,0.4) {$-\frac{1}{2}$}; 
\node at (3.25,-1.5) {$\frac{1}{2}$};

\end{tikzpicture}
\end{center}
where:
$$f_{\nu_{(0,1)}}^{\scrN_{\nu}} (\underline{a}, x) := \left( \sum_{\mu \in {\rm Fan-In}(\nu)} a_{\nu}^{\mu,1} f_{\mu} (\underline{a}, x) \right)^2, \qquad
f_{\nu_{(0,2)}}^{\scrN_{\nu}} (\underline{a}, x) :=  \left( \sum_{\mu \in {\rm Fan-In}(\nu)} a_{\nu}^{\mu,2} f_{\mu} (\underline{a}, x) \right)^2,$$
$$f_{\nu_{(0,3)}}^{\scrN_{\nu}} (\underline{a}, x) :=  \left( \sum_{\mu \in {\rm Fan-In}(\nu)} (a_{\nu}^{\mu,1}+ a_{\nu}^{\mu,2}) f_{\mu} (\underline{a}, x) \right)^2,$$
and
$$f_{\nu_{(1,1)}}^{\scrN_{\nu}}(\underline{a}, x) := \frac{1}{2} f_{\nu_{(0,3)}}^{\scrN_{\nu}} (\underline{a}, x) - \frac{1}{2} f_{\nu_{(0,1)}}^{\scrN_{\nu}} (\underline{a}, x) - \frac{1}{2} f_{\nu_{(0,2)}}^{\scrN_{\nu}} (\underline{a}, x). $$
Thus, by replacing each node $\nu$ of the evaluation scheme with the corresponding neural network $\scrN_{\nu}$ we have just described, we obtain a neural network with  activation function $\varphi(t)=t^2$ that is equivalent to the evaluation scheme. Observe that the converse also holds. Given a neural network with activation function $\varphi(t)=t^2$, we can represent it as an evaluation scheme by using the same parameters for both terms of the product in each node. \\
In this case, the class of potentially evaluable functions $K_{\circ}^{(n)}[\scrA]=K[X_1,\ldots, X_n]$ is the class of all polynomials in $n$ variables with coefficients in $K$. Note that if $\scrN$ is of depth at most $\ell$, the functions evaluable by $\scrN$ are polynomials of degree at most $2^{\ell}$. If we choose a parameter set $\Lambda\subseteq K^N$, where $N$ is the number of edges in $\scrN$, and if $\scrN$ has $m$ output nodes, then the class of all polynomials evaluable by $\scrN$ with parameters in $\Lambda$ is a constructible subset:
 $$\CW(\scrN, \Lambda) \subseteq \scrP_{(2^\ell)}^K(X_1,\ldots, X_n),$$
 where $(2^\ell)=(2^\ell, \ldots, 2^\ell) \in \N^n$ is a list of degrees. 
\end{example}

\begin{example}[{\bf Neural networks with polynomial activation function}] \label{funcionActivacionPolinomio:ejem}
Instead of using $\varphi(t) = t^2$ as the unique activation function, we also consider neural networks whose activation function $\varphi(t) = p(t) \in \K[t]$ is a polynomial. We will discuss some properties of these neural networks in Subsection \ref{properties-rational-neuralnetworks:subsec}.
\end{example}

\begin{example}[{\bf Neural networks with rational activation function}]
As in the previous examples, we may consider neural networks $\scrN:=(\scrG, \scrA, \Phi)$ over a field $K$ and set of input variables $\{X_1,\ldots, X_n\}$ such that the class of activation functions is again a class of a single function $\scrA:=\{\varphi\}$, where
$\varphi:Dom(\varphi)\subseteq K \longrightarrow K,$
is a rational function. Namely, there are two univariate polynomials $p,q\in K[T]$, such that $Dom(\varphi)=\{t\in K \; :\; q(t)\not=0\}$ and:
$$\varphi(t)={{p(t)}\over{q(t)}}, \; \forall t\in Dom(\varphi).$$
We define the \emph{degree of the rational function $\varphi$} as follows:
$$deg\left(\varphi\right):=\max\{ \deg(p), \deg(q)\},$$
where $p$ and $q$ are as above. 
This extends the classes of networks defined in Examples \ref{SLPs-t2:ejem} and \ref{funcionActivacionPolinomio:ejem} since we are admitting any rational function $\varphi\in K(T)$ as activation function. Note that the class of potentially evaluable functions by $\scrN$ satisfies $K_{\circ}^{(n)}[\scrA]\subseteq K(X_1,\ldots, X_n)$. We may distinguish two cases:

\begin{enumerate}
\item Neural networks with rational activation function that admit Strassen's Vermeidung von Divisionen (cf. \cite{Strassen:vermeidung}). This is the case discussed in detail in \cite{Krick-Pardo96}, Section 3, and used in ulterior works of the TERA project (as in \cite{SolvedFast}, \cite{Kronecker-CRAS}, \cite{lower-diophantine:jpaa} or \cite{SLPsElimination}). This will not be the main subject of our discussion in this section.
\item Neural networks with rational activation function that evaluate rational functions which are not polynomials. We will study this case in detail in Subsection \ref{NNrational:subsec}.
\end{enumerate}
\end{example}

\begin{example}[{\bf Neural networks with Nash or Pfaffian activation functions}]
In some cases, more sophisticated functions than polynomials or rational functions are considered. This is the case when the activation function is either a Nash function (analytic algebraic function) or a Pfaffian function (as introduced in \cite{Khovanskii:book}). We refer the reader to the extensive literature on this topic but do not treat it here: \cite{Koiran}, \cite{Goldberg}, \cite{Lickteig},  \cite{Karpinsky}, \cite{VCAnalytic}, \cite{Montanna-Pardo} and references therein.
\end{example}

\subsection{Some properties of the constructible set $W(\scrN, \Lambda)$ in the case of polynomial activation function} \label{properties-rational-neuralnetworks:subsec}

We consider the case of neural networks with fixed polynomial activation function. The following estimates are useful to study both the growth function of the associated class of binary classifiers and the length and density of correct test sequences for $\CW(\scrN, \Lambda)$.

\begin{proposition}\label{constructible-neural-networks-polinomial:prop}
Let $K$ be an algebraically closed field. Let $\scrN:=(\scrG, \scrA, \Phi)$ be a neural network with input nodes assigned to variables in $\{ X_1,\ldots, X_n\}$, of depth $\ell$, size $L$, space requirements bounded by $S$ and $m$ output nodes. Assume that $\scrN$ has $N$ edges and that $\scrA=\{p\}$, where $p$ is a univariate polynomial of degree $d \geq 2$. Let $\Lambda\subseteq \A^N(K)$ be a constructible set in the affine space of the parameters of $\scrN$.
Let $\{ A_\nu^\mu\; :\; \nu \in V(\scrG), \; \mu \in {\rm Fan-In}(\nu)\}$ be the set of parameter variables of the network and let:
$$K[A_\nu^\mu\; :\; \nu \in V(\scrG), \; \mu \in {\rm Fan-In}(\nu)],$$
be the ring of polynomials with coefficients in $K$ in this set of variables. Then, we have:
\begin{enumerate}
\item For every node $\sigma \in V(\scrG)$ of the network at depth $i$ and for every multi-index $\theta=(\theta_1,\ldots, \theta_n)\in \N^n$ of total degree $\vert\theta\vert=\theta_1+\cdots + \theta_n\leq d^i$, there is a polynomial:
$$Q_{\sigma}^{(\theta)}\in K[A_\nu^\mu\; :\; \nu \in V(\scrG), \; \mu \in {\rm Fan-In}(\nu)],$$
of total degree at most $d^{i+1}-2$ such that the following equality holds for all $\underline{a}\in \Lambda$ and for all $x:=(x_1\ldots, x_n)\in \A^n(K)$:
$$f_\sigma^{\scrN}(\underline{a}, x)=\sum_{\theta \in \N^n, \vert \theta\vert \leq d^i} Q_{\nu}^{(\theta)}(\underline{a}) x_1^{\theta_1} \cdots x_n^{\theta_n}\in K.$$
\item The set $\CW(\scrN, \Lambda)\subseteq \scrP_{(d^\ell)}^{K}$, where $(d^\ell)=(d^\ell,\ldots, d^\ell)\in \N^m$,  is a constructible set that satisfies the following properties:
    \begin{enumerate}
    \item The (Krull) dimension of  $\CW(\scrN, \Lambda)$ is bounded by the following inequality:
    $$\dim (\CW(\scrN, \Lambda)) \leq \dim(\Lambda)\leq N\leq LS.$$
    \item The degree of the Zariski closure of  $\CW(\scrN,\Lambda)$  satisfy the following inequality:
    $$\deg_{z} \left(\CW(\scrN, \Lambda) \right) \leq \deg_{\rm lci}(\Lambda) \left( d^{\ell+1}-2\right)^{\dim(\Lambda)} \leq \deg_{\rm lci}(\Lambda) \left( d^{\ell+1}-2\right)^{LS}. $$
\end{enumerate}
\end{enumerate}

\end{proposition}
\begin{proof}
Claim $i)$ easily follows by the same inductive arguments on the depth shown in Lemma 14 of \cite{Krick-Pardo96}. As for the second claim, let $\{\nu_1,\ldots, \nu_m\}$ be the list of output nodes of $\scrN$. Note that the functions $f_{\nu_k}(\underline{a}, \cdot)$ evaluated by $\scrN$ are polynomial functions of degree at most $d^{\ell}$. Then, observe that the polynomials $Q_{\nu_k}^{(\theta)}$ define a mapping from $\Lambda$ that parameterizes $\CW(\scrN, \Lambda)$:
$$\begin{matrix}
Q : & \Lambda & \longrightarrow & \CW(\scrN, \Lambda)\subseteq \scrP_{(d^\ell)}^{K}[X_1,\ldots, X_n]\\
& \underline{a} & \longmapsto & \left(f_{\nu_1}^{\scrN}(\underline{a}, X_1,\ldots, X_n), \ldots,f_{\nu_m}^{\scrN}(\underline{a}, X_1,\ldots, X_n) \right) ,
\end{matrix}$$
where for every $k$, $1\leq k \leq m$, we have:
$$f_{\nu_k}^{\scrN}(\underline{a}, X_1,\ldots, X_n):= \sum_{\theta \in \N^n, \vert \theta\vert \leq d^\ell} Q_{\nu_k}^{(\theta)}(\underline{a}) X_1^{\theta_1} \cdots X_n^{\theta_n}\in K[X_1,\ldots, X_n].$$
 Thus, we have proved that $\CW(\scrN, \Lambda)$ is a constructible subset  given as the image of the constructible subset $\Lambda\subseteq \A^N(K)$ by a list of polynomials of degree at most
$d^{\ell+1}-2$. Obviously, the dimension of $\CW(\scrN, \Lambda)$ is bounded by the dimension of $\Lambda$ and, applying Proposition \ref{grado-imagen-constructible:prop}, we conclude:
$$ \deg_{z} (\CW(\scrN, \Lambda))  \leq \deg_{\rm lci}(\Lambda) \left( d^{\ell+1}-2\right)^{\dim(\Lambda)} \leq  \deg_{\rm lci}(\Lambda) \left( d^{\ell+1}-2\right)^{N}\leq
\deg_{\rm lci}(\Lambda) \left( d^{\ell+1}-2\right)^{LS}, $$
as wanted.
\end{proof}

\begin{remark}
Note that if $\Lambda$ is an open Zariski set, the upper bound of Claim $ii)$ $(b)$ becomes:
 $$\deg_{z} \left( \CW(\scrN, \Lambda) \right)  \leq  \left( d^{\ell+1}-2\right)^{LS}. $$
\end{remark}

\begin{corollary}
Let $K$ be an algebraically closed field. Let $\scrN:=(\scrG, \scrA, \Phi)$ be a neural network with input nodes assigned to variables in $\{ X_1,\ldots, X_n\}$, of depth $\ell$, size $L$, space requirements bounded by $S$ and a single output node. Assume that $\scrA=\{p\}$, where $p$ is a univariate polynomial of degree $d$. Let $\Lambda\subseteq \A^N(K)$ be a constructible set in the affine space of the parameters of $\scrN$ and $\CW(\scrN, \Lambda)$ the class of all polynomials evaluable by $\scrN$ with parameters in $\Lambda$.  
Let $\scrH$ be the class of classifiers given by the characteristic functions of distinguished open subsets of $\A^n(K)$ defined by polynomials in $\CW(\scrN, \Lambda)$:
$$\scrH = \{\crctc_{D(f)} \; : \; f \in \CW(\scrN, \Lambda) \}.$$
Then,
$$G(\scrH,m) \leq \deg_{\rm lci} (\Lambda) ((d^{\ell+1}-2) \cdot (1+ m  (d^{\ell}+1))^{LS}.$$
\end{corollary}

\begin{proof}
The statement follows by noting that the growth function is monotone with respect to the inclusion of families of classifiers and combining Theorem \ref{GrowthFuncionOpenZariski:thm} with the estimates from Proposition \ref{constructible-neural-networks-polinomial:prop}.
\end{proof}

\begin{corollary}
Let $K$ be an algebraically closed field. Let $\scrN:=(\scrG, \scrA, \Phi)$ be a neural network with input nodes assigned to variables in $\{ X_1,\ldots, X_n\}$, of depth $\ell$, size $L$, space requirements bounded by $S$ and $m$ output nodes. Assume that $\scrA=\{p\}$, where $p$ is a univariate polynomial of degree $d$. Let $\Lambda\subseteq \A^N(K)$ be a constructible set in the affine space of the parameters of $\scrN$ and $\Omega = \CW(\scrN, \Lambda)$ the class of all lists of polynomials evaluable by $\scrN$ with parameters in $\Lambda$. Let $s := \dim(\Lambda)$ be the dimension of $\Lambda$ and $t:= s/ \dim(\Omega)$ be the quotient between the dimensions of $\Lambda$ and $\Omega$. Let $\Sigma\subseteq \Omega$ be a constructible subset of co--dimension at least  $1$ in $\Omega$ and assume that the following property holds:
 $$
 \forall f:=(f_1,\ldots, f_m)\in \Omega \setminus \Sigma, \; \dim(V_\A(f_1,\ldots, f_m))= n-m.
$$
Let $V:=V_\A(h_1,\ldots, h_r)\subseteq \A^n$ be a complete intersection algebraic variety of co--dimension $r\geq (n-m)+ m/2 + 1/2$ such that $\deg(V)\geq \delta^r$, where $\delta:=\min\{\deg(h_1),\ldots, \deg(h_r)\}$. Let $M \in \N$ be a positive integer and assume that the following properties hold:
\begin{enumerate}
\item $M \geq 6LS$,
\item $\log(\delta) \geq \max\{2(1+\log(d^{\ell}+1)), 2t\left({{\log(\deg_{\rm lci}(\Lambda))}\over{s}} +  \log(d^{\ell+1}-2)\right)\}$,
\item $\max\{ \deg(h_1), \ldots, \deg(h_r)\}\leq(1+ {{1}\over{n-m}})\delta$.
\end{enumerate}
Let $R:=R(\Omega, \Sigma, V, M)$ be the constructible set of all sequences ${\bf Q}\in V^M$ of length $M$ which are correct test sequences for $\Omega$ with respect to $\Sigma$. Then, there is a non--empty open subset $\G(V)\subseteq \G(n, n-r)$, 
with respect to the final topology induced by the onto mapping $G$ introduced in Equation (\ref{topologiaMatriz:eq}),  such that for every $A\in \G(V)$ the probability that a randomly chosen list ${\bf Q}\in (V\cap A)^M$ is in $R$ satisfies:
$${\rm Prob}_{(V\cap A)^M}[R]\geq 1- {{1}\over{\deg_{\rm lci}(\Lambda)e^{s(\log(d^{\ell+1}-2)+1) + (m-1)M}}},
$$
where $(A\cap V)^M$ is endowed with its uniform probability distribution.
\end{corollary}

\begin{proof}
It is a straight-forward application of Corollary 5.8 of \cite{PardoSebastian} using the bounds obtained in Proposition \ref{constructible-neural-networks-polinomial:prop}.
\end{proof}

\subsection{Neural networks with rational activation function to evaluate rational functions} \label{NNrational:subsec}

In this subsection, we reduce neural networks with rational activation function to neural networks with activation function $\varphi(t) = t^2$. Additionally, we discuss the growth function of the associated class of binary classifiers and the length and density of correct test sequences for this type of neural networks.

 \begin{lemma}\label{constructible-neural-network-rationa-activation:lema}
Let $K$ be an algebraically closed field. Let $\scrN:=(\scrG, \scrA, \Phi)$ be a neural network with input nodes assigned to variables in $\{ X_1,\ldots, X_n\}$, of depth $\ell$, size $L$, space requirements bounded by $S$ and $m$ output nodes. Assume that $\scrN$ has $N$ edges and that $\scrA=\{\varphi\}$, where $\varphi:={{p}\over{q}}$ is a univariate rational function given as the quotient of two co-prime univariate polynomials of degree $d$ (i.e. $\deg(\varphi)=\max\{ \deg(p), \deg(q)\} = d$). Then, there exists a neural network $\scrN':=(\scrG', \scrA', \Phi')$ with input variables in $\{ X_1,\ldots, X_n\}$ over the same filed $K$ such that the following properties hold:
\begin{enumerate}
\item The class $\scrA'$ of new activation functions is reduced to a single function $\scrA':=\{ t^2\}$.
\item The size of $\scrN'$ is in $O(L(d+ S^2))$, the depth of $\scrN'$ is in $O(\ell(\log_2(d)+ \log_2(S)))$, the space requirements of $\scrN'$ is in $O(d+S^2)$ and the number of edges of $\scrN'$ is of order $O(N(d+S^2))$.
\item The number of output nodes of $\scrN'$ doubles the number $m$ of output nodes of $\scrN$. Moreover, if $O_\scrN:=\{ \nu_1,\ldots, \nu_m\}$ are the output nodes of $\scrN$, the output nodes of $\scrN'$ are given by the following equality:
    $$O_{\scrN'}:=\{ (\nu_1, 1), (\nu_1, 2), \ldots, (\nu_m,1), (\nu_m,2)\}.$$
\item 
The  functions evaluated by $\scrN'$ and the functions evaluated by $\scrN$ are related as follows:
    $$f_{\nu_i}^{\scrN}(\underline{A}, X_1,\ldots, X_n)={{f_{(\nu_i,1)}^{\scrN'}(\underline{A}, X_1,\ldots, X_n)}\over{f_{(\nu_i,2)}^{\scrN'}(\underline{A}, X_1,\ldots, X_n)}},$$
    for every $i$, $1\leq i \leq m$.
\end{enumerate}

\end{lemma}
\begin{proof}
The idea is to built a new network $\scrN'=(\scrG', \scrA', \Phi')$ from $\scrN=(\scrG, \scrA, \Phi)$ where each node $\nu$ of the graph $\scrG$ is replaced by a network $\scrN_\varphi^{I_\nu}$ that evaluates the numerator and denominator of $\varphi$ at node $\nu$. In order to do this, we first construct a network evaluating $\varphi$ as follows. Assume that the polynomials $p$ and $q$ defining $\varphi$ are given by the following identities:
$$p(T):= b_dT^d+ b_{d-1}T^{d-1} + \cdots + b_0,\;\; q(T):= c_dT^d+ c_{d-1}T^{d-1} + \cdots + c_0,$$
where $b_i, c_j\in K$. 
Let $I$ be a finite set of cardinality $\sharp(I)=M$. We define the following two parameterized multivariate polynomials based on the homogenizations of $p$ and $q$:
$$P_I(\underline{A}, Z_1,\ldots, Z_M, Y, X_1,\ldots, X_M):= \sum_{k=0}^d b_k Y^{d-k}\left( \sum_{i\in I} A_i Z_iX_i\right)^k,$$
and
$$Q_I(\underline{A},Z_1,\ldots, Z_M, Y, X_1,\ldots, X_M):= \sum_{k=0}^d c_k Y^{d-k}\left( \sum_{i\in I} A_i Z_iX_i\right)^k,$$
where $\underline{A}:=(A_1,\ldots, A_M):=(A_i \; : \; i\in I)$ is a list of parameter variables. Now, assume we are given a set of pairs of elements in $K$:
$$\{ (\alpha_i, \beta_i)\in K^2\; :\; i\in I\},$$
where $\beta_i\not=0$ for all $i\in I$, and a list of parameters:
$$\underline{a}:=(a_1,\ldots, a_M)= ( a_i\; :\; i\in I).$$
Let us define the following elements of $K$:
$$\theta:=\prod_{i\in I}\beta_i,$$
and for every $i\in I$, let us define:
$$\theta_i:=\left(\prod_{j\in I, j\not=i} \beta_j\right)={{\theta}\over{\beta_i}}.$$
Let us observe the following chain of equalities:
$$\varphi\left(\sum_{i\in I} a_i{{\alpha_i}\over{\beta_i}}\right)={{p\left(\sum_{i\in I} a_i{{\alpha_i}\over{\beta_i}}\right)}\over{q\left(\sum_{i\in I} a_i{{\alpha_i}\over{\beta_i}}\right)}}=
{{\left(\prod_{i\in I}\beta_i\right)^dp\left(\sum_{i\in I} a_i{{\alpha_i}\over{\beta_i}}\right)}\over{\left(\prod_{i\in I}\beta_i\right)^dq\left(\sum_{i\in I} a_i{{\alpha_i}\over{\beta_i}}\right)}},
$$
and, hence,
$$\varphi\left(\sum_{i\in I} a_i{{\alpha_i}\over{\beta_i}}\right)={{\sum_{k=0}^d b_k \left(\prod_{i\in I}\beta_i\right)^{d-k}\left(\sum_{i\in I} a_i\left(\prod_{j\in I, j\not=i}\beta_j\right)\alpha_i\right)^k}\over{\sum_{k=0}^d c_k \left(\prod_{i\in I}\beta_i\right)^{d-k}\left(\sum_{i\in I} a_i\left(\prod_{j\in I, j\not=i}\beta_j\right)\alpha_i\right)^k}}.$$
Namely, we have:
$$\varphi\left(\sum_{i\in I} a_i{{\alpha_i}\over{\beta_i}}\right)={{P_I(\underline{a}, \theta_1,\ldots, \theta_M, \theta, \alpha_1,\ldots, \alpha_M)}\over{Q_I(\underline{a}, \theta_1,\ldots, \theta_M, \theta, \alpha_1,\ldots, \alpha_M)}}.$$

Using this last identity, we conclude that for every finite set $I$ of cardinality $\sharp(I)=M$, there is a neural network $\scrN_\varphi^{I}$ with inputs $\{ X_1,\ldots, X_M, Y_1, \ldots, Y_M\}$ of size and space requirements in $O(d+M^2)$ , depth in $O(\log_2(d)+ \log_2(M))$, parameters in $\{ A_1,\ldots, A_M\}$, activation function $\varphi(t):=t^2$ and two output gates $(1)$ and $(2)$ that performs the following tasks:
\begin{itemize}
\item First, the neural network computes the following quantities for every $i \in I$:
$$Z_i:=\prod_{j\not=i, j\in I} Y_j,$$
using $O(M^2)=O(\sharp(I)^2)$ nodes and depth $O(log_2(M))$.
\item Additionally, in $O(1)$ nodes and depth $1$, the network computes $Y:=\prod_{i\in I} Y_i$ (just multiplying $Z_1$ by $Y_1$, for instance).
\item In some additional $O(M)$ nodes and in depth $1$, the network evaluates the list of products:
$$Z_iX_i, \; i\in I.$$
\item Then, the network evaluates for every $k$, $0\leq k \leq d$:
$$\left(\sum_{i\in I} A_iZ_i X_i\right)^k,$$
using $O(d)$ nodes in depth $O(\log_2(d))$.
\item Using an additional number of $O(d)$ nodes, in depth $O(\log_2(d))$, the network evaluates:
$$Y^{d-k}, \; 0\leq k \leq d.$$
\item With just $O(d)$ nodes in $O(1)$ depth, using $(b_0,\ldots, b_d)$ and $(c_0, \ldots, c_d)$ as parameters, the network outputs the two nodes $(1)$ and $(2)$ as follows:
 $$P_I(\underline{A}, Z_1, \ldots, Z_M, Y, X_1,\ldots, X_M),$$
and  $$Q_I(\underline{A}, Z_1, \ldots, Z_M, Y, X_1,\ldots, X_M).$$

\end{itemize}

This elementary idea, applied to the neural network $\scrN:=(\scrG, \scrA, \Phi)$, yields the statement. For every node $\nu$ of $\scrN$, we consider $I_\nu:={\rm Fan-In}(\nu):=\{ \mu \in V \; :\; d(\mu) \leq d(\nu)-1\}$. We then define $\scrN':=(\scrG', \scrA', \Phi')$ as follows.
We first replace the occurrence of the node $\nu$ in the graph $\scrG$ of $\scrN$ by a neural network $\scrN_\varphi^{I_\nu}$ associated to the node $\nu$, determined by $\varphi$ and the list of indices $I_\nu$, as described in the previous identities. This determines the new graph $\scrG'$. The new assignement just uses the activation function $\scrA':=\{ t^2\}$. For every node $\nu$ of the original graph $\scrG$, the subgraph associated to $\scrN_\varphi^{I_\nu}$ has two associated output nodes $\{ (\nu, 1), (\nu,2)\}$. These two nodes associated to the original node $\nu$, satisfy the following property:\\
Assume the original function evaluated at $\nu$ by $\scrN$ is a function of the form:
$$f_\nu^{\scrN}((A_\nu^\mu\; :\; \mu\in I_\nu), X_1,\ldots, X_n).$$
Assume the two functions associated to node $\nu$ in the new network $\scrN'$ are given by:
$$f_{(\nu,1)}^{\scrN'}((A_\nu^\mu\; :\; \mu\in I_\nu), X_1,\ldots, X_n),$$
and
$$f_{(\nu,2)}^{\scrN'}((A_\nu^\mu\; :\; \mu\in I_\nu), X_1,\ldots, X_n),$$
where $M=\sharp(I_\nu)$. These two functions are given recursively in terms of the depth by the following rule:
$$f_{(\nu,1)}^{\scrN'}:=P_{I_\nu}\left( (A_\nu^\mu\; :\; \mu\in I_\nu), Z_1^{\nu}, \ldots, Z_M^\nu, Y, f_{(\mu_1,1)}^{\scrN'}, \ldots, f_{(\mu_M, 1)}^{\scrN'} \right),$$
$$f_{(\nu,2)}^{\scrN'}:=Q_{I_\nu}\left( (A_\nu^\mu\; :\; \mu\in I_\nu), Z_1^{\nu}, \ldots, Z_M^\nu, Y, f_{(\mu_1,1)}^{\scrN'}, \ldots, f_{(\mu_M, 1)}^{\scrN'} \right),$$
where ${\rm Fan-In}(\nu)=\{ \mu_1,\ldots, \mu_M\}$, $Z_i^\nu$ is given by the following identity:
$$Z_i^\nu:=\prod_{\mu\in I_\nu, \mu\not=\mu_i} f_{(\mu,2)}^{\scrN'},$$
and
$$Y:=\prod_{\mu \in I_\nu} f_{(\mu,2)}^{\scrN'}.$$
Thus, recursively, we may easily prove that for every node $\nu$ of the original network $\scrN$, we have:
$$f_\nu^{\scrN}((A_\nu^\mu\; :\; \mu\in I_\nu), X_1,\ldots, X_n)= {{f_{(\nu,1)}^{\scrN'}((A_\nu^\mu\; :\; \mu\in I_\nu), X_1,\ldots, X_n)}\over{f_{(\nu,2)}^{\scrN'}((A_\nu^\mu\; :\; \mu\in I_\nu), X_1,\ldots, X_n)}},$$
and the statement follows.
\end{proof}

\begin{corollary}\label{constructible-neural-network-rationa-activation:corol}
Let $K$ be an algebraically closed field. Let $\scrN:=(\scrG, \scrA, \Phi)$ be a neural network with input nodes assigned to variables in $\{ X_1,\ldots, X_n\}$, of depth $\ell$, size $L$, space requirements bounded by $S$ and $m$ output nodes. Assume that $\scrN$ has $N$ edges and that $\scrA=\{\varphi\}$, where $\varphi:={{p}\over{q}}$ is a univariate rational function given as the quotient of two co-prime univariate polynomials of degree $d$. Let $\Lambda\subseteq \A^N(K)$ be a constructible set in the affine space of the parameters of $\scrN$. Let $\{ A_\nu^\mu\; :\; \nu \in V(\scrG), \; \mu \in {\rm Fan-In}(\nu)\}$ be the set of parameter variables of the network and let:
$$K[A_\nu^\mu\; :\; \nu \in V(\scrG), \mu \in {\rm Fan-In}(\nu)],$$
be the ring of polynomials with coefficients in $K$ in this set of variables. Then, we have:
\begin{enumerate}
\item There are two constructible subsets:
$$\CW_1(\scrN, \Lambda), \CW_2(\scrN, \Lambda) \subseteq \scrP_{((dS)^{\ell})}^{K}(X_1,\ldots, X_n),$$
    where $((dS)^{\ell})  =\left( (dS)^\ell, \ldots, (dS)^\ell\right))\in \N^m$, such that the functions evaluated at the $m$ output nodes $\{\nu_1,\ldots, \nu_m\}$ of  $\scrN$, by instantiating the parameters, can be described as pairs of lists $(f_1,\ldots, f_m)\in \CW_1(\scrN, \Lambda)$, $(g_1,\ldots, g_m)\in \CW_2(\scrN, \Lambda)$ such that for every $i$, $1\leq i \leq m$, we have:
    $$f_{\nu_i}^\scrN(\underline{a}, X_1,\ldots, X_n)= {{f_i(X_1,\ldots, X_n)}\over{g_i(X_1,\ldots, X_n)}}\in K(X_1,\ldots, X_n).$$
\item Dimensions satisfy the following inequalities for every $i \in \{1,2 \}$:
$$\dim(\CW_i(\scrN, \Lambda))\leq \dim(\Lambda)\leq LS.$$
\item There is a universal constant $c>0$, independent of $\scrN$ and $\varphi$, such that the degrees of the Zariski closures of $\CW_1(\scrN, \Lambda)$ and  $\CW_2(\scrN, \Lambda)$ satisfy the following inequalities for every $i \in \{1,2 \}$: 
    $$\deg_{z} \left(\CW_1(\scrN, \Lambda) \right)  \leq \deg_{\rm lci}(\Lambda) \left(2(dS)^{c\ell}-2 \right)^{\dim(\Lambda)}\leq \deg_{\rm lci}(\Lambda) \left(2(dS)^{c\ell}-2 \right)^{LS}.$$
\end{enumerate}
\end{corollary}
\begin{proof}

We use the neural network $\scrN'$ from Lemma \ref{constructible-neural-network-rationa-activation:lema} and the estimates of Proposition \ref{constructible-neural-networks-polinomial:prop}. The network $\scrN'$, built from $\scrN$ and $\varphi$, is a network of depth $O(\ell(\log_2(d)+\log_2(S)))$.  We define the constructible sets $\CW_1(\scrN, \Lambda)$ and $\CW_2(\scrN, \Lambda)$ as follows:
$$ \CW_1(\scrN, \Lambda) = \{ (f_{(\nu_1,1)}^{\scrN'} (\underline{a}, X_1, \ldots, X_n), \ldots, f_{(\nu_m,1)}^{\scrN'} (\underline{a}, X_1, \ldots, X_n)) \; : \; \underline{a} \in \Lambda \},    $$
$$ \CW_2(\scrN, \Lambda) = \{ (f_{(\nu_1,2)}^{\scrN'} (\underline{a}, X_1, \ldots, X_n), \ldots, f_{(\nu_m,2)}^{\scrN'} (\underline{a}, X_1, \ldots, X_n)) \; : \; \underline{a} \in \Lambda \}.    $$
Looking at the operations performed, we easily see that the degrees with respect to the variables $\{X_1,\ldots, X_n\}$ at the outputs of $\scrN'$ are bounded by $(dS)^\ell$.
Now,  we consider for every $i$, $1 \leq i \leq m$: $$f_i (X_1, \ldots, X_n) = f_{(\nu_i,1)}^{\scrN'} (\underline{a}, X_1, \ldots, X_n),$$
and $$g_i (X_1, \ldots, X_n) = f_{(\nu_i,2)}^{\scrN'} (\underline{a}, X_1, \ldots, X_n).$$
Obviously, $(f_1,\ldots, f_m)\in \CW_1(\scrN, \Lambda)$ and $(g_1,\ldots, g_m)\in \CW_2(\scrN, \Lambda)$ and Claim $i)$ follows. \\
Except for the coefficients of the numerator and denominator of $\varphi$, which are used instantiated, the parameters of $\scrN'$ are the same as those of $\scrN$. Hence the dimension is bounded by the dimension of the constructible parameter set $\Lambda$ of $\scrN$ and the statement of Claim $ii)$ holds.\\
As for the last claim, we may view $\scrN'$ as a double network, one for the numerators and the other for the denominators. Although the size of the networks increases, the depth is bounded by $c\ell(\log_2(d)+ \log_2(S)))$ for some universal constant $c>0$. Moreover, the activation function $\phi(t)=t^2$ used by $\scrN'$ has degree $2$. Using the bounds of Claim $i)$ of Proposition \ref{constructible-neural-networks-polinomial:prop}, we conclude that  both $\CW_1(\scrN, \Lambda)$ and $\CW_2(\scrN, \Lambda)$ are parameterized by polynomials of degree at most:
$$2^{c\ell(\log_2(d)+\log_2(S))+ 1}-2= 2(dS)^{c\ell}-2,$$
with parameters in the constructible set $\Lambda$. Hence, using Proposition \ref{grado-imagen-constructible:prop}, we conclude for every $i \in \{i,2 \}$:
$$\deg_{z} \left( \CW_i(\scrN, \Lambda) \right)  \leq \deg_{\rm lci}(\Lambda) \left(2(dS)^{c\ell}-2 \right)^{\dim(\Lambda)} \leq \deg_{\rm lci}(\Lambda) \left(2(dS)^{c\ell}-2 \right)^{LS},$$
 which proves Claim $iii)$ of the corollary.\end{proof}

\begin{lemma} \label{ConstructibleNumeradorDenominador:lemma}
With the same notations and assumptions as above, there exists a constructible subset $\CC (\scrN, \Lambda) \subseteq \CW_1(\scrN, \Lambda) \times \CW_2(\scrN, \Lambda) \subseteq \scrP_{((dS)^{\ell})}^{K}[X_1,\ldots, X_n]^2$ such that for every $\underline{a} \in \Lambda$, we can associate a list of polynomials $(f_1, \ldots, f_m, g_1, \ldots, g_m) \in \CC (\scrN, \Lambda)$  satisfying for every $i$, $1 \leq i \leq m$:
 $$f_{\nu_i}^\scrN(\underline{a}, X_1,\ldots, X_n)= {{f_i(X_1,\ldots, X_n)}\over{g_i(X_1,\ldots, X_n)}}\in K(X_1,\ldots, X_n).$$
Additionally, we have:
$$\dim(\CC (\scrN, \Lambda)) \leq \dim(\Lambda) \leq LS,$$
$$\deg_{z} (\CC (\scrN, \Lambda)) \leq \deg_{\rm lci}(\Lambda) \left(2(dS)^{c\ell}-2 \right)^{\dim(\Lambda)}\leq \deg_{\rm lci}(\Lambda) \left(2(dS)^{c\ell}-2 \right)^{LS},$$
where $c>0$ is a universal constant.
\end{lemma}
\begin{proof}
Let $\underline{X}$ be the set of variables $\{ X_1, \ldots, X_n \}$. As in the previous corollary, we consider for every $i$, $1 \leq i \leq n$:
$$f_i (\underline{X}) = f_{(\nu_i,1)}^{\scrN'} (\underline{a},\underline{X})$$ and $$g_i (\underline{X}) = f_{(\nu_i,2)}^{\scrN'} (\underline{a}, \underline{X}),$$
where $\scrN'$ is the neural network introduced in Lemma \ref{constructible-neural-network-rationa-activation:lema}. Next, we define the following mapping:
$$\begin{matrix}
Q : & \Lambda & \longrightarrow &  \CW_1(\scrN, \Lambda) \times \CW_2(\scrN, \Lambda) \subseteq \scrP_{((dS)^{\ell})}^{K}[X_1,\ldots, X_n]^2\\
& \underline{a} & \longmapsto & \left(f_{(\nu_1,1)}^{\scrN'} (\underline{a}, \underline{X}), \ldots, f_{(\nu_m,1)}^{\scrN'} (\underline{a}, \underline{X}), f_{(\nu_1,2)}^{\scrN'} (\underline{a}, \underline{X}), \ldots
, f_{(\nu_m,2)}^{\scrN'} (\underline{a}, \underline{X}) \right)
\end{matrix}$$
and the constructible set $\CC (\scrN, \Lambda) = Im(Q) = Q(\Lambda)$. Obviously, $(f_1, \ldots, f_m, g_1, \ldots, g_m) \in \CC (\scrN, \Lambda)$. Observe that  $\CC (\scrN, \Lambda)$ is the image of the constructible set $\Lambda$ by a list of polynomials of degree $2(dS)^{c\ell}-2$. Thus, we have:
$$\dim(\CC (\scrN, \Lambda)) \leq \dim(\Lambda) \leq LS.$$
and, applying Proposition \ref{grado-imagen-constructible:prop}, we obtain:
$$\deg_{z} (\CC (\scrN, \Lambda)) \leq \deg_{\rm lci}(\Lambda) \left(2(dS)^{c\ell}-2 \right)^{\dim(\Lambda)}\leq \deg_{\rm lci}(\Lambda) \left(2(dS)^{c\ell}-2 \right)^{LS}. $$
\end{proof}

Let $h \in K(X_1, \ldots, X_n)$ be a rational function, we define the distinguished open set defined by $h$ as the open Zariski subset given by:
$$D(h) := \left\{ x  \in \A^n(K) \; : \; f(x) \cdot g(x) \neq 0, \; h = \frac{f}{g}  \right\}$$

Next, we obtain estimates for the growth function of the class of distinguished open set defined by neural networks with rational activation function.

\begin{corollary} \label{FuncionCrecimientoRacional:corol}
Let $K$ be an algebraically closed field. Let $\scrN:=(\scrG, \scrA, \Phi)$ be a neural network with input nodes assigned to variables in $\{ X_1,\ldots, X_n\}$, of depth $\ell$, size $L$, space requirements bounded by $S$ and a single output node. Assume that $\scrN$ has $N$ edges and that $\scrA=\{\varphi\}$, where $\varphi:={{p}\over{q}}$ is a univariate rational function given as the quotient of two co-prime univariate polynomials of degree $d$ (i.e. $\deg(\varphi)=\max\{ \deg(p), \deg(q)\} = d$). Let $\Lambda \subseteq \A^N(K)$ be a constructible set in the affine space of parameters of $\scrN$ and let $\CW (\scrN, \Lambda)$ denote the class of all rational functions evaluable by $\scrN$ with parameters in $\Lambda$. Let $\scrH$ be the class of classifiers given by the characteristic functions of distinguished open subsets of $\A^n(K)$ defined by rational functions in $\CW (\scrN, \Lambda)$:
$$\scrH = \{\crctc_{D(h)} \; : \; h \in \CW(\scrN, \Lambda) \}.$$
Then,
$$G(\scrH, m) \leq \deg_{\rm lci} (\Lambda) \left( (2(dS)^{c\ell}-2) \cdot ( 1+2m((dS)^{\ell}+1)) \right)^{LS},$$
where $c>0$ is a universal constant.
\end{corollary}
\begin{proof}
From Lemma \ref{ConstructibleNumeradorDenominador:lemma}, we deduce that for all rational functions $h \in \CW(\scrN, \Lambda)$, there exists a pair of polynomials $(f,g) \in \CC(\scrN, \Lambda)$ such that $h=f/g$. Additionally, we have:
$$h(x) = \frac{f(x)}{g(x)} \neq 0 \text{ if and only if } f(x) \cdot g(x) \neq 0.$$
Now, observe that the following equality holds:
$$\scrH = \{\crctc_{D(h)} \; : \; h \in \CW(\scrN, \Lambda) \} = \{ \crctc_{D(f \cdot g)} \; : \; (f,g) \in \CC(\scrN, \Lambda) \}.$$
We define the constructible set:
$$V := \{ ((f,g),x) \in \A^{2D} (K) \times \A^n(K) \; : \; f(x) \cdot g(x) \neq 0 \},$$
where:
$$D = { (dS)^{\ell} +n \choose n }.$$
Note that $V$ is the complement of an algebraic hypersurface of degree $2((dS)^{\ell}+1)$. Since the growth function is monotone with respect to the inclusion of families of classifiers and applying Theorem \ref{FuncionCrecimientoDimensionConstructibles:thm} together with the estimates for $\CC(\scrN, \Lambda)$ obtained in the previous lemma, we conclude:
$$G(\scrH, m) \leq \deg_{\rm lci} (\Lambda) \left( (2(dS)^{c\ell}-2) \cdot ( 1+2m((dS)^{\ell}+1)) \right)^{LS}.$$
\end{proof}

We generalize the notion of correct test sequence for constructible sets of polynomials (see Definition \ref{CTS:def}) to classes of rational functions as follows:

\begin{definition}
Let $\scrR \subseteq K(X_1, \ldots, K_n)$ be a class of rational functions. A correct test sequence of length $M$ for $\scrR$ is a finite set of $M$ elements ${\bf Q} := \{x_1, \ldots, x_M \} \subseteq (K^n)^M$ such that:
$$\forall h = \frac{f}{g} \in \scrR, \; f,g \in K[X_1, \ldots, X_n], \; (f \cdot g) (x_1) = 0, \ldots, (f \cdot g)(x_M) = 0 \Rightarrow h =0.$$ 
\end{definition}

\begin{corollary} \label{cuestoresRacional:corol}
Let $K$ be an algebraically closed field. Let $\scrN:=(\scrG, \scrA, \Phi)$ be a neural network with input nodes assigned to variables in $\{ X_1,\ldots, X_n\}$, of depth $\ell$, size $L$, space requirements bounded by $S$ and a single output node. Assume that $\scrN$ has $N$ edges and that $\scrA=\{\varphi\}$, where $\varphi:={{p}\over{q}}$ is a univariate rational function given as the quotient of two co-prime univariate polynomials of degree $d$ (i.e. $\deg(\varphi)=\max\{ \deg(p), \deg(q)\} = d$). Let $\Lambda \subseteq \A^N(K)$ be a constructible set in the affine space of parameters of $\scrN$ and let $\CW (\scrN, \Lambda)$ denote the class of all rational functions evaluable by $\scrN$ with parameters in $\Lambda$. Let $\Omega = \CW_{p} (\scrN, \Lambda)$ be the constructible set given by:
$$\CW_p (\scrN, \Lambda) = \left\{ f \cdot g \; : \; h = \frac{f}{g} \in \CW(\scrN, \Lambda) \right\}.$$
Let $s := \dim(\Lambda)$ be the dimension of $\Lambda$ and $t := s/\dim(\Omega)$ be the quotient between the dimensions of $\Lambda$ and $\Omega$. Let $V:=V_\A(h_1,\ldots, h_r)\subseteq \A^n (K)$ be a zero-dimensional algebraic variety given by polynomial equations of the same degree $\delta := \deg(h_i)$, $1 \leq i \leq n$. Assume that $\deg(V) = \delta^n$. Let $M \in \N$ be a positive integer and assume that the following properties hold:
\begin{enumerate}
\item $M \geq 6LS$,
\item $\log(\delta) \geq \max \left\{2(1+\log(2(dS)^{\ell}+1)), 2t\left({{\log(\deg_{\rm lci}(\Lambda))}\over{s}} +  \log(4(dS)^{c \ell}- 4)\right) \right\}$,
\end{enumerate}
where $c>0$ is a universal constant.
Let $R:=R(\CW (\scrN, \Lambda), \Omega, V, M)$ be the constructible set of all sequences ${\bf Q}\in V^M$ of length $M$ which are correct test sequences for $\CW (\scrN, \Lambda)$. 
Assume that $V$ is endowed with its uniform probability distribution. Then, we have:
$${\rm Prob}_{V^M}[R]\geq 1- {{1}\over{\deg_{\rm lci}(\Lambda)e^{s(\log(4(dS)^{c \ell}- 4)+1)}}},
$$
\end{corollary}

\begin{proof}
We build a neural network $\scrN''$ by adding $O(1)$ nodes to the neural network $\scrN'$ of Lemma \ref{constructible-neural-network-rationa-activation:lema} that computes the product of the two outputs of $\scrN'$. Clearly, the degree of the output node of $\scrN''$ is bounded by $2(dS)^{\ell}$. Note that:
$$\CW_p (\scrN, \Lambda) = \CW(\scrN'', \Lambda) = \{ f_{(\nu,1)}^{\scrN'} (\underline{a}, X_1, \ldots, X_n) \cdot f_{(\nu,2)}^{\scrN'} (\underline{a}, X_1, \ldots, X_n) \; : \; \underline{a} \in \Lambda \}.$$
Next, observe that the parameters of $\scrN''$ are the same as those of $\scrN'$ and, therefore, we have the following upper bound for the dimension of $\CW_p (\scrN, \Lambda)$:
$$\dim(\CW_p (\scrN, \Lambda)) \leq \dim(\Lambda) \leq LS.$$
From Proposition \ref{constructible-neural-networks-polinomial:prop} and Corollary \ref{constructible-neural-network-rationa-activation:corol}, we conclude that $\CW_p (\scrN, \Lambda)$ is parameterized by a list of polynomials of degree at most $4(dS)^{c \ell}- 4$.
 Now, applying Proposition \ref{grado-imagen-constructible:prop}, we obtain:
$$\deg_{z} (\CW_p (\scrN, \Lambda)) \leq \deg_{\rm lci} (\Lambda) (4(dS)^{c \ell}- 4)^{\dim(\Lambda)} \leq \deg_{\rm lci} (\Lambda) (4(dS)^{c \ell}- 4)^{\dim(\Lambda)}$$
Finally, we apply Corollaries 5.6 and 5.8 of \cite{PardoSebastian} using the previous bounds for the constructible set $\CW_p (\scrN, \Lambda)$ to obtain the result. 
\end{proof}

Now, we use correct test sequences for Identity Test of rational functions given by neural networks:

\begin{corollary} \label{TestNulidadRacionales:corol}
Let $K$ be an algebraically closed field. Consider the neural networks $\scrN_1:=(\scrG_1, \scrA_1, \Phi_1)$ and $\scrN_2:=(\scrG_2, \scrA_2, \Phi_2)$. Both networks have input nodes assigned to variables in $\{ X_1,\ldots, X_n\}$, depth $\ell$, size $L$, space requirements bounded by $S$, a single output node and $N$ edges. Assume that $\scrA_{1}=\{\varphi\}$ and $\scrA_{2}=\{\psi \}$, where $\varphi$ and $\psi$ are univariate rational functions given as the quotient of two co-prime univariate polynomials of degree $d$. 
Let $\Lambda \subseteq \A^N(K)$ be a constructible set in the affine space of parameters of $\scrN_1$ and let $\CW (\scrN_1, \Lambda)$ denote the class of all rational functions evaluable by $\scrN_1$ with parameters in $\Lambda$. Similarly, let $\Gamma \subseteq \A^N(K)$ be a constructible set in the affine space of parameters of $\scrN_2$ and let $\CW (\scrN_2, \Gamma)$ denote the class of all rational functions evaluable by $\scrN_2$ with parameters in $\Gamma$. Let 
$$\CW (\scrN_1, \Lambda) - \CW (\scrN_2, \Gamma) = \{ h_1 - h_2 \; : \; h_1 \in \CW (\scrN_1, \Lambda), \; h_2 \in \CW (\scrN_2, \Gamma) \}.$$
denote the class of the differences between elements of $\CW (\scrN_1, \Lambda)$ and $\CW (\scrN_2, \Gamma)$. Let $\Omega = \CW_p ((\scrN_1, \Lambda), (\scrN_2, \Gamma) )$ be the constructible set defined by:
$$\CW_p ((\scrN_1, \Lambda), (\scrN_2, \Gamma)) = \left\{ f \cdot g \; : \; h = \frac{f}{g} \in \CW (\scrN_1, \Lambda) - \CW (\scrN_2, \Gamma) \right\}. $$
Let $s := \dim(\Lambda \times \Gamma)$ be the dimension of $\Lambda \times \Gamma$ and $t := s/\dim(\Omega)$ be the quotient between the dimensions of $\Lambda \times \Gamma$ and $\Omega$.  Let $V:=V_\A(h_1,\ldots, h_r)\subseteq \A^n (K)$ be a zero-dimensional algebraic variety given by polynomial equations of the same degree $\delta := \deg(h_i)$, $1 \leq i \leq n$. Assume that $\deg(V) = \delta^n$. Let $M \in \N$ be a positive integer and assume that the following properties hold:
\begin{enumerate}
\item $M \geq 12LS$,
\item $\log(\delta) \geq \max \left\{2(1+\log(4(dS)^{\ell}+1)), 2t\left({{\log(\deg_{\rm lci}(\Lambda) \deg_{\rm lci}(\Gamma))}\over{s}} +  \log(8(dS)^{c \ell}- 8)\right) \right\}$,
\end{enumerate}
where $c>0$ is a universal constant.
Let $R:=R(\CW (\scrN_1, \Lambda) - \CW (\scrN_2, \Gamma),  \Omega, V, M)$ be the constructible set of all sequences ${\bf Q}\in V^M$ of length $M$ which are correct test sequences for $\CW (\scrN_1, \Lambda) - \CW (\scrN_2, \Gamma)$. Assume that $V$ is endowed with its uniform probability distribution. Then, we have:
$${\rm Prob}_{V^M}[R]\geq 1- {{1}\over{\deg_{\rm lci}(\Lambda)e^{s(\log(8(dS)^{c \ell}- 8)+1)}}},
$$
\end{corollary}
\begin{proof}
First, we apply Lemma \ref{constructible-neural-network-rationa-activation:lema} to $\scrN_1$ and $\scrN_2$ to obtain the neural networks $\scrN_1'$ and $\scrN_2'$, respectively. We then construct a new neural network $\scrN''$ by combining $\scrN_1'$ and $\scrN_2'$ and adding $O(1)$ nodes to compute the product of the numerator and denominator of the difference between the rational functions associated to $\scrN_1'$ and $\scrN_2'$. To do that, assume that the outputs of $\scrN_1'$ are $f_1^{\scrN_1'}$ (numerator) and $f_2^{\scrN_1'}$ (denominator) and that the outputs of $\scrN_2'$ are $f_1^{\scrN_2'}$ (numerator) and $f_2^{\scrN_2'}$ (denominator). We can compute the output of $\scrN''$ as follows:
$$(f_1^{\scrN_1'} \cdot f_2^{\scrN_2'} - f_2^{\scrN_1'} \cdot f_1^{\scrN_2'}) \cdot (f_2^{\scrN_1'} \cdot f_2^{\scrN_2'})$$
The degree of the output of the neural network $\scrN''$ is bounded by $4(dS)^{\ell}$.  Let $\underline{X}$ be the set of variables $\{ X_1, \ldots, X_n \}$. Note that:
\begin{equation*}
\begin{split}
& \CW_p ((\scrN_1, \Lambda), (\scrN_2, \Gamma)) = \CW(\scrN'', \Lambda \times \Gamma) = \\ &   = \left\{ (f_1^{\scrN_1'} (\underline{a}, \underline{X}) \cdot f_2^{\scrN_2'} (\underline{b}, \underline{X}) - f_2^{\scrN_1'} (\underline{a}, \underline{X}) \cdot f_1^{\scrN_2'} (\underline{b}, \underline{X})) \cdot (f_2^{\scrN_1'} (\underline{a}, \underline{X}) \cdot f_2^{\scrN_2'}  (\underline{b}), \underline{X}) \; : \; \underline{a} \in \Lambda, \; \underline{b} \in \Gamma \right\}.
\end{split}
\end{equation*}

Next, observe that $\CW_p ((\scrN_1, \Lambda), (\scrN_2, \Gamma))$ is given as the image of the constructible set $\Lambda \times \Gamma$ by a list of polynomials of degree at most $8(dS)^{cl}-8$. Obviously, we have: 
$$\dim(\CW_p ((\scrN_1, \Lambda), (\scrN_2, \Gamma))) \leq \dim(\Lambda) + \dim(\Gamma) \leq 2 LS.$$
Applying Proposition \ref{grado-imagen-constructible:prop}, we obtain:
$$\deg_{z} (\CW_p ((\scrN_1, \Lambda), (\scrN_2, \Gamma))) \leq \deg_{\rm lci} (\Lambda \times \Gamma) (8(dS)^{c \ell}- 8)^{\dim(\Lambda \times \Gamma)} \leq$$
$$ \leq \deg_{\rm lci} (\Lambda) \deg_{\rm lci} (\Gamma) (8(dS)^{c \ell}- 8)^{2LS}.$$
Finally, we apply Corollaries 5.6 and 5.8 of \cite{PardoSebastian} using the previous bounds for the constructible set $\CW_p ((\scrN_1, \Lambda), (\scrN_2, \Gamma))$ to obtain the result. 

\end{proof}

\end{document}